\documentclass[final,12pt]{colt2022}
\usepackage{xcolor}
\usepackage{pifont}
\usepackage{makecell}
\usepackage{hhline}
\usepackage{colortbl}
\usepackage{wrapfig,lipsum}
\usepackage{float}    
\usepackage{mathtools}  
\usepackage{hyperref}
\usepackage{cleveref}
\usepackage{bookmark}
\usepackage{enumerate}
\usepackage{paralist}
\usepackage{xspace}
\usepackage{caption}
\usepackage{bbm}
\usepackage{bm}
\usepackage{lipsum}
\usepackage{algorithm}
\usepackage{algorithmic}
\usepackage{color}
\usepackage{microtype}
\usepackage{lscape}
\usepackage{multirow}
\usepackage{graphicx}
\usepackage{caption}
\usepackage{subcaption}

\newcommand{\stoptocwriting}{%
  \addtocontents{toc}{\protect\setcounter{tocdepth}{-5}}}
\newcommand{\resumetocwriting}{%
  \addtocontents{toc}{\protect\setcounter{tocdepth}{\arabic{tocdepth}}}}

\newcommand{\vertiii}[1]{{\left\vert\kern-0.25ex\left\vert\kern-0.25ex\left\vert #1
		\right\vert\kern-0.25ex\right\vert\kern-0.25ex\right\vert}}

\makeatletter

\newcommand{\fro}{\mathrm{F}}

\newcommand{\mX}{\vect{X}}
\newcommand{\mY}{\vect{Y}}

\newcommand{\id}{\cI}

\makeatletter
\def\Hy@raisedlink@left#1{%
    \ifvmode
        #1%
    \else
        \Hy@SaveSpaceFactor
        \llap{\smash{%
        \begingroup
            \let\HyperRaiseLinkLength\@tempdima
            \setlength\HyperRaiseLinkLength\HyperRaiseLinkDefault
            \HyperRaiseLinkHook
        \expandafter\endgroup
        \expandafter\raise\the\HyperRaiseLinkLength\hbox{%
            \Hy@RestoreSpaceFactor
            #1%
            \Hy@SaveSpaceFactor
        }%
        }}%
        \Hy@RestoreSpaceFactor
        \penalty\@M\hskip\z@ 
    \fi
}
\makeatother

\makeatletter
\newcommand\newtarget[2]{\Hy@raisedlink@left{\hypertarget{#1}{}}#2}
\makeatother

\newcommand\linkofproof[1]{\textbf{of {#1}. }\newtarget{proof:#1}}

\DeclareMathOperator{\dist}{dist}

\DeclareMathOperator{\Var}{Var}
\DeclareMathOperator{\poly}{\mathrm{poly}}

\newcommand{\rank}{\mathrm{rank}}

\newcommand{\proj}{\cP}

\newcommand{\inner}[2]{\left\langle #1, #2 \right\rangle}

\newcommand{\norm}[1]{\left\lVert#1\right\rVert}
\newcommand{\const}{\Gamma}

\newcommand{\tr}{{\operatorname{tr}}}
\newcommand{\diag}{{\operatorname{Diag}}}

{\medskip}

\newcommand{\cB}{\mathcal{B}}

\newcommand{\cI}{\mathcal{I}}

\newcommand{\cM}{\mathcal{M}}
\newcommand{\cN}{\mathcal{N}}
\newcommand{\cO}{\mathcal{O}}
\newcommand{\cP}{\mathcal{P}}

\newcommand{\cR}{\mathcal{R}}
\newcommand{\cS}{\mathcal{S}}

\newcommand{\cV}{\mathcal{V}}

\newcommand{\bE}{\mathbb{E}}

\newcommand{\bP}{\mathbb{P}}

\newcommand{\bR}{\mathbb{R}}

\newcommand{\projection}{\cR_{\Omega}}
\newcommand{\event}{E_{\mathrm{good}}}
\newtheorem{lemma}[theorem]{Lemma}

\newtheorem{condition}{Condition}

\usepackage{amsmath,amsfonts,bm}

\DeclareMathOperator{\col}{col}
\newcommand*{\rom}[1]{%
\textup{\uppercase\expandafter{\romannumeral#1}}%
}

\def\eqref#1{equation~\ref{#1}}

\def\1{\bm{1}}

\def\vzero{{\bm{0}}}

\def\va{{\bm{a}}}

\def\vp{{\bm{p}}}

\def\vu{{\bm{u}}}
\def\vv{{\bm{v}}}

\def\vx{{\bm{x}}}
\def\vy{{\bm{y}}}

\def\mA{{\bm{A}}}
\def\mB{{\bm{B}}}
\def\mC{{\bm{C}}}
\def\mD{{\bm{D}}}
\def\mE{{\bm{E}}}

\def\mG{{\bm{G}}}
\def\mH{{\bm{H}}}
\def\mI{{\bm{I}}}
\def\mJ{{\bm{J}}}

\def\mL{{\bm{L}}}
\def\mM{{\bm{M}}}

\def\mO{{\bm{O}}}
\def\mP{{\bm{P}}}

\def\mR{{\bm{R}}}
\def\mS{{\bm{S}}}

\def\mU{{\bm{U}}}
\def\mV{{\bm{V}}}

\def\mX{{\bm{X}}}
\def\mY{{\bm{Y}}}
\def\mZ{{\bm{Z}}}

\def\mLambda{{\bm{\Lambda}}}
\def\mSigma{{\bm{\Sigma}}}
\def\mDelta{{\bm{\Delta}}}

\def\mXi{{\bm{\Xi}}}
\def\mOmega{{\bm{\Omega}}}

\DeclareMathAlphabet{\mathsfit}{\encodingdefault}{\sfdefault}{m}{sl}
\SetMathAlphabet{\mathsfit}{bold}{\encodingdefault}{\sfdefault}{bx}{n}

\def\emA{{A}}
\def\emB{{B}}

\def\emU{{U}}
\def\emV{{V}}

\def\emX{{X}}

\DeclareMathOperator*{\argmin}{arg\,min}

\title[GD Converges for Matrix Completion]{Convergence of Gradient Descent with Small Initialization for Unregularized Matrix Completion}
\usepackage{times}

\coltauthor{%
 \Name{Jianhao Ma} \Email{jianhao@umich.edu}\\
 \Name{Salar Fattahi} \Email{fattahi@umich.edu}\\
 \addr University of Michigan, Ann Arbor
}
\allowdisplaybreaks
\begin{document}

\maketitle

\begin{abstract}
  We study the problem of symmetric matrix completion, where the goal is to reconstruct a positive semidefinite matrix $\mX^\star \in \bR^{d\times d}$ of rank-$r$, parameterized by $\mU\mU^{\top}$, from only a subset of its observed entries. We show that the vanilla gradient descent (GD) with small initialization provably converges to the ground truth $\mX^\star$ without requiring any explicit regularization. This convergence result holds true even in the over-parameterized scenario, where the true rank $r$ is unknown and conservatively over-estimated by a search rank $r'\gg r$. The existing results for this problem either require explicit regularization, a sufficiently accurate initial point, or exact knowledge of the true rank $r$. 

In the over-parameterized regime where $r'\geq r$, we show that, with $\widetilde\Omega(dr^9)$ observations, GD with an initial point $\norm{\mU_0} \leq \epsilon$ converges near-linearly to an $\epsilon$-neighborhood of $\mX^\star$. Consequently, smaller initial points result in increasingly accurate solutions. Surprisingly, neither the convergence rate nor the final accuracy depends on the over-parameterized search rank $r'$, and they are only governed by the true rank $r$. In the exactly-parameterized regime where $r'=r$, we further enhance this result by proving that GD converges at a faster rate to achieve an arbitrarily small accuracy $\epsilon>0$, provided the initial point satisfies $\norm{\mU_0} = O(1/d)$. At the crux of our method lies a novel \textit{weakly-coupled leave-one-out analysis}, which allows us to establish the global convergence of GD, extending beyond what was previously possible using the classical leave-one-out analysis.
\end{abstract}

\stoptocwriting
\section{Introduction}
Matrix completion is a fundamental problem in the field of machine learning, where the objective is to reconstruct a positive semidefinite (PSD) matrix of rank-$r$, denoted as $\mX^{\star}\in \bR^{d\times d}$, from only a subset of its observed entries. The most natural approach to solve this problem involves minimizing the following mean squared error:
\begin{equation}\tag{MC}\label{MC}
    \min_{\mU\in \bR^{d\times r'}}f(\mU)=\frac{1}{4p}\norm{\proj_{\Omega}(\mU\mU^{\top}-\mX^{\star})}_{\fro}^2.
\end{equation}
Here, $p$ represents the probability of observing each entry in $\mX^{\star}$, $\Omega$ denotes the set of observed entries, and $\proj_{\Omega}$ shows the projection operation onto the set of matrices supported by $\Omega$. When the true rank $r$ is unknown, it is often over-estimated by the search rank $r'\geq r$, leading to what is referred to as \textit{over-parameterized} matrix completion.

A prominent application of matrix completion is in collaborative filtering \citep{gleich2011rank}. Additionally, it has applications in other areas, including image reconstruction \citep{hu2018generalized}, fast kernel matrix approximation \citep{graepel2002kernel,paisley2010nonparametric}, and more recently, in teaching arithmetic to transformers~\citep{lee2023teaching}.

Perhaps the most natural approach for solving the above optimization is the (vanilla) gradient descent (GD): given an initial point $\mU_0$ and a fixed step-size $\eta>0$, generate a sequence of iterates $\{\mU_t\}_{t=1}^T$ according to $\mU_{t+1} = \mU_t-\eta\nabla f(\mU_t)$. Despite its simplicity and desirable practical performance (see Figure~\ref{fig::performance} and the experiments in~\citep{zheng2016convergence}), the conditions under which the GD converges globally to the ground truth $\mX^\star$ have remained a long-standing mystery.

A line of research has been devoted to studying gradient-based algorithms with \textit{explicit regularization}~\citep{sun2016guaranteed,zheng2016convergence,jain2013low, ge2016matrix}. These methods typically incorporate either an $\ell_{2, \infty}$-norm regularizer or a projection step to constrain the iterates within a set with $\ell_{2, \infty}$-norm bounds to promote \textit{incoherence} (see Definition~\ref{condition_incoherence}). However, the use of $\ell_{2, \infty}$-norm regularization or projection techniques often introduce more tuning
parameters, and has been found to be unnecessary in practice \citep{zheng2016convergence,ma2018implicit}.
 
On the other hand, the convergence of GD without explicit regularization was initially tackled by \citet{ma2018implicit} in the context of symmetric matrix completion, and subsequently extended by \citet{chen2020nonconvex} to asymmetric settings. However, these studies consider a very special case of matrix completion where GD is initialized sufficiently close to the ground truth, and the rank of the ground truth $r$ is known. In practice, however, GD converges even if it is initialized far from the ground truth and the rank is over-parameterized $r'\gg r$ (see Figure~\ref{fig::performance}). 

Therefore, the following question still remains open:
\begin{quote}
    \textit{Why does GD with a small initialization efficiently converge to the ground truth of~\ref{MC} in the absence of explicit regularization, even in the general rank-$r$ case where $r' \geq r \geq 1$?}
\end{quote}

\begin{figure}\centering
    \begin{centering}
    \subfigure[]{
    {\includegraphics[width=0.43\linewidth]{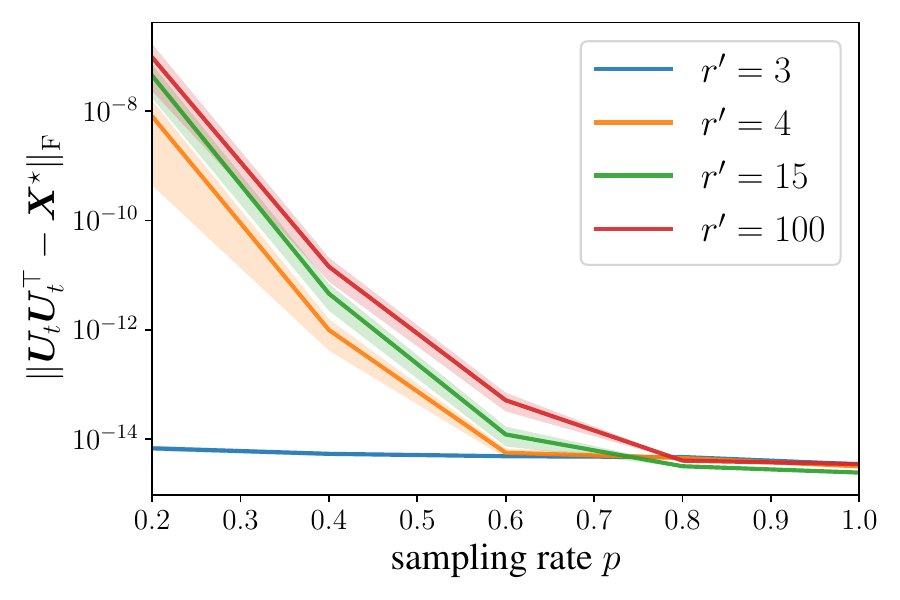}}\label{fig::sampling_rate}}
    \subfigure[]{
    {\includegraphics[width=0.43\linewidth]{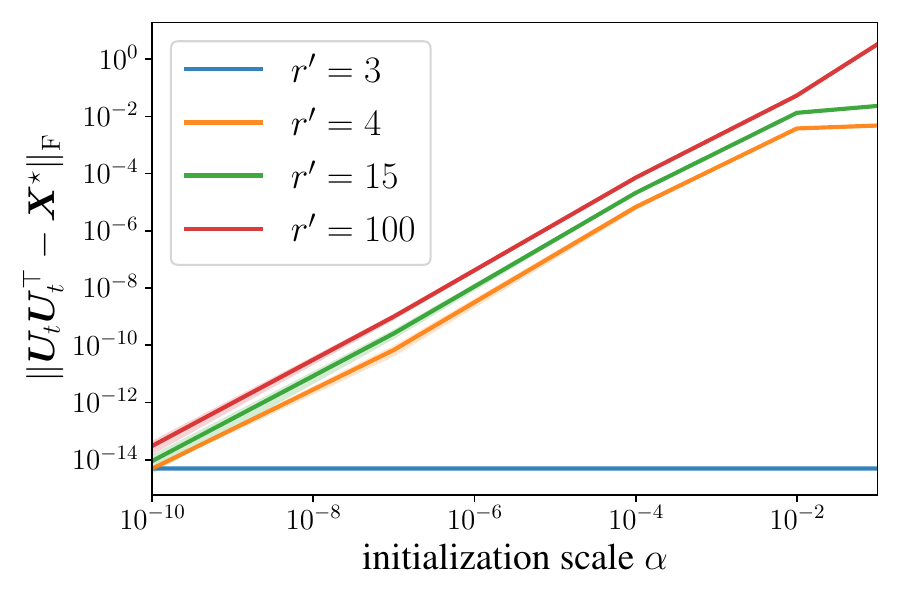}}\label{fig::initialization}}
    \end{centering}
    \caption{\footnotesize The performance of vanilla GD with small initialization and constant step-size on \ref{MC} without any explicit regularization. In all experiments, we set the dimension $d=100$ and the true rank $r=3$. $(a)$: Higher sampling rates improve the final accuracy of GD for the over-parameterized \ref{MC}, but have no impact on the exactly-parameterized \ref{MC}. $(b)$: Increasing the initialization scale hampers the final accuracy of GD for the over-parameterized \ref{MC}, but has no impact on the exactly-parameterized \ref{MC}.}
    \label{fig::performance}
\end{figure}

Recently, \citet{kim2022rank} answered the above question for the special case of rank-$1$ symmetric matrix completion with $r'=r=1$, but their proposed approach does not extend to the general rank-$r$ case. It is also worth noting that the above question has been addressed for another class of matrix factorization problems satisfying a norm-preserving property called \textit{restricted isometry property} (RIP). Problems that satisfy this property include matrix sensing~\citep{li2018algorithmic, stoger2021small, ma2023global} and phase retrieval~\citep{chen2019gradient}. However, a significant challenge arises with matrix completion, as it {\it does not} satisfy the restricted isometry property. Consequently, the existing methodologies built upon RIP are not applicable to matrix completion.

\subsection{Summary of Contributions}
In this work, we provide a complete answer to the aforementioned question.
A comparison of our results with other studies on matrix completion can be found in Table~\ref{table::main}. The key contributions of our work are as follows:
\begin{itemize}
\item [-] {\bf Convergence of GD with small initialization in over-parameterized regime:} When the rank of the ground truth $r$ is unknown and over-parameterized by $r'\geq r\geq 1$, we prove that GD with small initialization converges to the ground truth at a near-linear rate. Surprisingly, neither the convergence rate nor the final accuracy depends on the over-parameterized search rank $r'$, and they are only governed by the true rank $r$. In particular, given an initial point that satisfies $\norm{\mU_0}\leq \epsilon$ for some $\epsilon>0$ and a sampling rate of $p=\widetilde{\Omega}(r^9\log^6(1/\epsilon)/d)$, GD converges to $\epsilon$-neighborhood of $\mX^\star$ in $O(\log^4(1/\epsilon))$ iterations. Therefore, a smaller initial point or a larger sampling rate can improve the final error of GD. The empirical observation presented in Figure~\ref{fig::performance} provides further support for this result.

\item [-] {\bf Improved results in the exactly-parameterized regime:} We show that GD enjoys an improved convergence when the rank of the ground truth $r$ is known and $r'=r\geq 1$. In particular, given an initial point that satisfies $\norm{\mU_0}\leq O(1/d)$ and a sampling rate of $p=\widetilde{\Omega}(r^9/d)$, GD converges to $\epsilon$-neighborhood of $\mX^\star$ in $O(\log(1/\epsilon))$ iterations for any arbitrarily small $\epsilon>0$. A key distinction from the over-parameterized setting is that the final error of GD remains unaffected by the initialization scale or the sample size, provided that they meet certain thresholds. This is also evident in Figure~\ref{fig::performance}. When $r=O(1)$, the resulting sample complexity is information-theoretically optimal (modulo logarithmic factors).

    \item [-] {\bf Weakly-coupled leave-one-out analysis:} In order to establish the implicit regularization of GD for~\ref{MC}, an effective approach relies on a decoupling mechanism known as leave-one-out analysis, a trick rooted in probability and random matrix theory. However, the current theory based on this technique is only effective when the iterates are sufficiently close to the ground truth. 
    At the crux of our technical analysis lies an extension of the classical leave-one-out analysis to the global setting, which we call \textit{weakly-coupled leave-one-out analysis}. In essence, our proposed method relaxes the requirement for the initial iterates to be sufficiently close to the ground truth, making it particularly suitable for the global convergence analysis of GD. 
\end{itemize}

\colorlet{shadecolor}{gray!20}
\begin{table*}[t] 
    \centering
    \small
    \resizebox{0.95\linewidth}{!}{%
        \renewcommand{\arraystretch}{1.5}
        \begin{tabular}{|c|c|c|c|c|c|}
            \hline 
            \textbf{Algorithm} &   \textbf{Sample complexity}& \textbf{Computational complexity} & \textbf{Global} & \textbf{Exact} & \textbf{Over-param.} \\
            
            \hhline{|=|=|=|=|=|=|}
            \citep{ma2018implicit} &  $dr^3\log^3(d)$ & $\kappa^2\log\left(\frac{1}{\epsilon}\right)$ & \ding{56} & \ding{52} & \ding{56} \\
            \citep{ma2018implicit} &  $dr^3\log^3(d)$ & $\kappa^2\log\left(\frac{1}{\epsilon}\right)$ & \ding{56} & \ding{52} & \ding{56} \\
            \hline
            \citep{chen2020nonconvex} &  $dr^2\log(d)$ & $\kappa^2\log\left(\frac{1}{\epsilon}\right)$ & \ding{56} & \ding{52} & \ding{56} \\
            \hline
             \citep{kim2022rank}&   $d\log^{22}(d)$ & $\log\left(\frac{1}{\epsilon}\right)$& \ding{52} & \ding{81} & \ding{56} \\
            \hline
            \rowcolor{shadecolor} \Gape[0pt][2pt]
            Ours (Theorem~\ref{thm::main-exact-param}) & $dr^9\log^8(d)$ & $\kappa^4\log\left(\frac{1}{\epsilon}\right)$& \ding{52} & \ding{52} & \ding{56} \\
            \hline
            \rowcolor{shadecolor} \Gape[0pt][2pt]
            Ours (Theorem~\ref{thm::main}) & $dr^9\log^2(d)\log^6\left(\frac{1}{\epsilon}\right)$ & $\kappa^4\log^4\left(\frac{1}{\epsilon}\right)$ & \ding{52} & \ding{52} & \ding{52} \\
            \hline

        \end{tabular}
    }
    \caption{\footnotesize Comparisons between different algorithms for matrix completion without explicit regularization. \ding{81} The result only holds for $r'=r=1$.
    }
    \label{table::main}
\end{table*}

\paragraph*{Notations.} We use bold uppercase letters $\mX, \mY$ to denote matrices and bold lowercase letters $\vx, \vy$ to denote vectors. For vectors, we use $\norm{\cdot}$ to denote $\ell_2$-norm, and for matrices we use $\norm{\cdot}$ and $\norm{\cdot}_{\fro}$ to denote operator norm and Frobenius norm, respectively. For matrix $\mX\in \bR^{d_1\times d_2}$, we denote by $\emX_{i, j}$ the $(i, j)$-th element of $\mX$, $\mX_{i,\cdot}$ the $i$-th row, and $\mX_{\cdot, j}$ the $j$-th column. The $\ell_{2, \infty}$-norm of $\mX$, denoted as $\norm{\mX}_{2,\infty}$, is defined as $\max_{i}\norm{\mX_{i,\cdot}}$. Additionally, we define the singular values of $\mX\in \bR^{d_1\times d_2}$ as $\sigma_{1}(\mX)\geq \sigma_2(\mX)\geq \cdots\geq \sigma_{\min\{d_1, d_2\}}(\mX)\geq 0$. The set of all orthogonal matrices is denoted by $\cO_{d_1\times d_2}:=\{\mO\in \bR^{d_1\times d_2}:\mO^{\top}\mO=\mI\}$. For two matrices $\mX, \mY\in \bR^{d_1\times d_2}$, we define their Procrustes distance as $\dist(\mX, \mY)=\min_{\mO\in \cO_{d_2\times d_2}}\norm{\mX-\mY\mO}_{\fro}$. The projection matrix onto the column space of an orthogonal matrix $\mV\in \cO_{d_1\times d_2}$ is defined as $\proj_{\mV} := \mV\mV^\top$.

We use the notation $f(n)\lesssim g(n)$ or $f(n)=O(g(n))$ when a constant $C>0$ exists such that $f(n) \leq Cg(n)$ for sufficiently large $n$. Conversely, $f(n)\gtrsim g(n)$ or $f(n)=\Omega(g(n))$ implies the existence of a constant $C>0$ such that $f(n) \geq Cg(n)$ for sufficiently large $n$. Moreover, we use the notations $\widetilde O(\cdot)$ and $\widetilde \Omega(\cdot)$ to hide logarithmic dependencies on the dimension or other parameters of the problem. Additionally, we use $f(n)\asymp g(n)$ or $f(n) = \Theta(g(n))$ when $f(n)\lesssim g(n)$ and $f(n)\gtrsim g(n)$. 

\section{Problem Setup and Main Results}
Suppose that the singular value decomposition (SVD) of $\mX^\star$ is given by $\mX^{\star}=\mV^{\star}\mSigma^{\star}\mV^{\star \top}$, where $\mV^{\star}\in \cO_{d\times r}$ and $\mSigma^{\star}$ is an $r\times r$ diagonal matrix with diagonal elements in descending order $\sigma_1^{\star} \geq \cdots \geq \sigma_r^{\star}> 0$. We denote the condition number of $\mX^\star$ as $\kappa=\sigma_1^{\star}/\sigma_r^{\star}$. 
Upon defining the symmetrized operator $\projection=\frac{1}{2p}({\proj_{\Omega}+\proj_{\Omega}^{\top}})$, the update rule for GD can be written as
\begin{equation}\label{eq_GD}
    \mU_{t+1}=\mU_t-\eta \nabla f(\mU_t)=\mU_t-\eta \projection(\mU_t\mU_t^{\top}-\mX^{\star})\mU_t, \qquad \mU_0=\alpha \mZ\tag{GD}
\end{equation}
where $\alpha>0$ is the initialization scale and $\mZ$ is the initialization matrix satisfying $\norm{\mZ}=1$. We assume that $\mZ$ satisfies the following \textit{alignment condition}.
\begin{condition}[Alignment]
    \label{condition::small-initialization}
    We say the matrix $\mZ\in \bR^{d\times r'}$ with $\norm{\mZ}=1$ satisfies the alignment condition if there exists a universal constant $c_0>0$ such that
    \begin{equation}
        \sigma_{r}(\proj_{\mV^{\star}}\mZ)\geq c_0.\tag{alignment condition}
        \label{eq::alignment-condition}
    \end{equation}
\end{condition}
Intuitively, this condition necessitates that the initialization matrix should have a non-negligible alignment with the column space of the ground truth. The following lemma reveals that this alignment condition is satisfied for common initialization strategies with overwhelming probability.
\begin{lemma}[Sufficient condition for alignment]\label{lem_init}
    The following statements are satisfied:
    \begin{itemize}
        \item \textbf{Gaussian initialization.} Given $0.5d\leq r'\leq d$ and $\mZ=\mG/\norm{\mG}$, where $\mG$ is a standard Gaussian matrix, \ref{eq::alignment-condition} holds with $c_0=0.1$ with probability at least $1-\exp\{-\Omega(d)\}$.
        \item \textbf{Orthogonal initialization.} Given $r'=d$ and $\mZ=\mO$ for any $\mO\in \cO_{d\times d}$, \ref{eq::alignment-condition} is satisfied with $c_0=1$. 
        \item \textbf{Spectral initialization.} Let $\mV\mSigma \mV^\top$ be the eigendecomposition of the best rank-$r'$ approximation of $\projection(\mX^{\star})$ measured in Frobenius norm. Given $r\leq r'\leq d$ and $\mZ=\mU/\norm{\mU}$, where $\mU=\mV\mSigma^{1/2}$, \ref{eq::alignment-condition} holds with $c_0=\frac{1}{2\kappa}$ with probability at least $1-\frac{1}{d^3}$.
    \end{itemize}
\end{lemma}
Before proceeding to the main theorem, we introduce two crucial conditions on the ground truth $\mX^{\star}$ and the random observation set $\Omega$.
\begin{condition}[Incoherence]\label{condition_incoherence}
    The rank-$r$ PSD matrix $\mX^{\star}\in \bR^{d\times d}$ with eigendecomposition $\mX^{\star}=\mV^{\star}\mSigma^{\star}\mV^{\star\top}$ is $\mu$-incoherent for some $\mu\geq 1$ if $\norm{\mV^{\star}}_{2, \infty}= \sqrt{\frac{\mu}{d}}\norm{\mV^{\star}}_{\fro}=\sqrt{\frac{\mu r}{d}}$.
\end{condition}

\begin{condition}[Random sampling model]\label{cond_random}
    Each entry of $\mX^{\star}$ is observed independently with probability $p$. In other words, $\bP((i, j)\in \Omega)=p$ independently for all $1\leq i, j\leq d$.
\end{condition}

\citet{chen2015incoherence} shows that the incoherence condition is necessary for the recovery of the ground truth. 
Intuitively, the incoherence condition with $\mu=O(1)$ entails that none of the columns of $\mV^\star$ have significant alignment with the standard basis vectors. Such a condition ensures that the ground truth is far from being sparse. We note that successful recovery of a sparse $\mX^\star$ is only achievable in a near-ideal scenario where $p\approx 1$ and nearly the entirety of $\mX^\star$ is observed. On the other hand, when $\mu = O(1)$, the recovery is possible even when $p$ scales with $\widetilde\Omega(\poly(r)/d)$ \citep{chen2015incoherence}. 

With the aforementioned conditions in place, we can now present our main result, which establishes the global convergence of GD on~\ref{MC} for the over-parameterized setting.

\begin{sloppypar}
    \begin{theorem}[Convergence of GD for over-parameterized~\ref{MC}]
    \label{thm::main} Consider~\ref{MC} with search rank $r\leq r'\leq d$.
    Suppose that Conditions~\ref{condition_incoherence} and~\ref{cond_random} are satisfied with a sampling rate of $p\gtrsim \frac{\kappa^6\mu^4r^9\log^6(\frac{1}{\alpha})\log^2(d)}{d}$. Consider the iterates of~\ref{eq_GD} with the step-size $\eta\asymp \frac{\mu r}{\sqrt{pd}\sigma^{\star}_{1}}$ and the initial point $\mU_0=\alpha \mZ$, where $0<\alpha\leq \sqrt{\frac{\sigma_1^{\star}}{d}}$ and $\mZ$ satisfies Condition~\ref{condition::small-initialization}.
    With probability at least $1-\frac{2}{d^3}$, after $T\lesssim\frac{1}{\eta\sigma^{\star}_{r}}\log\left(\frac{1}{\alpha}\right)$ iterations, we have 
    \begin{equation}\nonumber
        \norm{\mU_{T}\mU_{T}^{\top}-\mX^{\star}}_{\fro}\lesssim \sqrt{\frac{\sigma^{\star}_{1}\kappa^2\mu r^2}{p}}\alpha.
    \end{equation}
\end{theorem}
\end{sloppypar}

 A few observations are in order based on the above theorem.

\paragraph{Computational complexity.} The initialization scale $\alpha$ governs the final accuracy of GD. Therefore, to ensure that $\norm{\mU_{T}\mU_{T}^{\top}-\mX^{\star}}_{\fro}\leq \epsilon$, it suffices to set the initialization scale to $\alpha\lesssim \sqrt{\frac{p}{\sigma_1^{\star}\kappa^2\mu r^2}}\epsilon$. Moreover, assuming that $\max\{\kappa, \mu, r\} = O(1)$ and $\epsilon\leq {1}/{d}$, this accuracy is achieved within $\widetilde O\left(\log^4\left(\frac{1}{\epsilon}\right)\right)$ iterations, which scales only poly-logarithmically with $1/\epsilon$. 

\paragraph{Effect of over-parameterization.} The level of over-parameterization in the search rank $r'$ does not have any impact on either the sample complexity or the convergence of GD.  As a result, our results hold even if $d^2p\ll dr'$. In such cases,~\ref{MC} has many global minima, some of which may not satisfy $\mU\mU^\top\approx \mX^\star$. This sheds light on the implicit regularization of the vanilla GD with small initialization toward low-rank solutions when applied to~\ref{MC}.

\paragraph{Sample complexity.} The required sample complexity is given by $d^2p\gtrsim dr^9\kappa^6\mu^4\log^6\left(\frac{1}{\alpha}\right)\log^2\left(d\right)$, which is optimal with respect to the dimension $d$ up to a logarithmic factor. This contrasts with the direct extension of the approach by \citet{kim2022rank}, which necessitates a sample size on the order of $d^{1+\Theta(\kappa-1)}$. Moreover, Theorem~\ref{thm::main} highlights that the sampling rate $p$ must scale with $\log^6\left(\frac{1}{\epsilon}\right)$ to attain an accuracy level of $\epsilon$, which in turn leads to a mild dependency of the final error on the sampling rate. In other words, given a fixed sampling rate $p$, GD achieves an accuracy in the order of $\exp(-\Omega(pd))$. Our empirical findings in Figure~\ref{fig::performance} corroborate this result. 
Our next theorem relaxes this restriction in the exact-parameterization regime, showing that it is possible to attain a sample complexity that does not depend on the desired accuracy level or the initialization scale. 
\begin{sloppypar}
    \begin{theorem}[Convergence of GD for exactly-parameterized~\ref{MC}]
    \label{thm::main-exact-param}
    Consider~\ref{MC} with search rank $r'=r$.
    Suppose that Conditions~\ref{condition_incoherence} and~\ref{cond_random} are satisfied with a sampling rate of $p\gtrsim \frac{\kappa^6\mu^4r^9\log^8(d)}{d}$. Consider the iterates of~\ref{eq_GD} with the step-size $\eta\asymp \frac{\mu r}{\sqrt{pd}\sigma^{\star}_{1}}$ and the initial point $\mU_0=\alpha \mZ$, where $\alpha\asymp \frac{\sigma_r^{\star}}{\kappa^{1.5} d}$ and $\mZ$ satisfies Condition~\ref{condition::small-initialization}. Given any accuracy $\epsilon>0$, with probability at least $1-O\big(\frac{1}{d^3}\big)$ and after $T\lesssim\frac{1}{\eta\sigma^{\star}_{r}}\log\left(\frac{1}{\epsilon}\right)$ iterations, we have 
    \begin{equation}
        \norm{\mU_{T}\mU_{T}^{\top}-\mX^{\star}}_{\fro} \leq \epsilon.
    \end{equation}
\end{theorem}
\end{sloppypar}

We next outline the key distinctions between the two aforementioned theorems. In the exactly-parameterized regime, neither the sampling rate $p$ nor the initialization scale $\alpha$ affect the final error $\epsilon$, as long as they meet certain thresholds. In contrast, a smaller initialization scale or a larger sampling rate improves the final error in the over-parameterized regime. Moreover, the convergence rate of GD improves from $O\left(\log^4\left(\frac{1}{\epsilon}\right)\right)$ to $O\left(\log\left(\frac{1}{\epsilon}\right)\right)$. This is because the required sampling rate is smaller in the exactly-parameterized regime, allowing the algorithm to adopt a more aggressive step-size.

\section{Proof Outline}
In this section, we present the key ideas underpinning our proof techniques. We begin in Section~\ref{sec::dynamical-signal-residual-decomposition} with a dynamic signal-residual decomposition. Next, in Section~\ref{sec::incoherence-dynamic}, we introduce the weakly-coupled leave-one-out analysis which, together with our dynamic signal-residual decomposition, completes the proof for Theorem~\ref{thm::main}. Section~\ref{sec::exact-param} explains how these techniques can be further refined to yield improved results for the exactly-parameterized regime. Throughout this section, we occasionally omit the consideration of higher-order terms involving the step-size $\eta$. We note that while this omission serves to streamline the presentation, our rigorous proofs in the appendix carefully account for these higher-order terms.

\subsection{Dynamic Signal-residual Decomposition}
\label{sec::dynamical-signal-residual-decomposition}
We employ a dynamic projection scheme akin to that described by \citet{li2018algorithmic}, which decomposes the iterates $\mU_t$ into two distinct components: a low-rank signal part, $\mS_t$, and a residual part, $\mE_t$. This decomposition is represented as follows:
\begin{equation}\nonumber
    \mU_t=\mS_t+\mE_t, \quad \text{where} \quad \mS_t=\proj_{\mV_t}\mU_t, \text{ and } \mE_t=\proj^{\perp}_{\mV_t}\mU_t.
\end{equation}
Upon defining $\mM_t=\projection(\mX^{\star}-\mU_t\mU_t^{\top})$, the dynamic orthonormal matrix $\mV_t$ is recursively defined as:
\begin{equation}\nonumber
    \mV_{t+1}=\mZ_{t+1}\left(\mZ_{t+1}^{\top}\mZ_{t+1}\right)^{-1/2} \quad \text{where} \quad \mZ_{t+1}=\left(\mI+\eta\mM_t\right)\mV_t \text{ and }\mV_0=\mV^{\star}.
\end{equation}
Define the error matrix as $\mDelta_t:= \mX^{\star}-\mU_t\mU_t^{\top}$. Our goal is to show that $\norm{\mDelta_t}_{\fro}$ decreases efficiently to $O(\alpha)$. To show the benefit of the proposed dynamic signal-residual decomposition in achieving this goal, we start by stating the one-step dynamic of the error matrix:
\begin{equation}\label{loss_update}
    \norm{\mDelta_{t+1}}_{\fro}^2 = \norm{\mDelta_{t+1}}_{\fro}^2 - 4\eta\inner{\mDelta_t}{\mM_t\mU_t\mU_t^{\top}}+O(\eta^2).
\end{equation}
Therefore, to establish the convergence of $\norm{\mDelta_t}_{\fro}$, it suffices to provide a reasonable lower-bound for $\inner{\mDelta_t}{\mM_t\mU_t\mU_t^{\top}}$. This can be achieved via the following descent lemma:
\begin{sloppypar}
    \begin{lemma}[Descent lemma, informal]\label{lem_descent}
    Suppose that $\frac{\sqrt{\sigma^\star_r}}{2}\leq \sigma_r(\mS_t)$, $\norm{\mS_t}\leq 2\sqrt{\sigma_1^\star}$, and $\norm{\mV_t-\mV^\star}\leq 0.1$. Then, we have 
    \begin{equation}\nonumber
        \inner{\mDelta_t}{\mM_t\mU_t\mU_t^{\top}}\geq \frac{\sigma_r^\star}{15}\norm{\mDelta_t}_{\fro}^2-O\!\left(\sqrt{\frac{\sigma_r^{\star 3}\mu r^2}{p}}\norm{\mE_t}+\sqrt{r}\sigma_1^{\star}\norm{(\id-\projection)(\mDelta_t)}\right)\!\norm{\mDelta_t}_{\fro}\!.
    \end{equation}
\end{lemma}
\end{sloppypar}
By combining the descent lemma with \Cref{loss_update}, we arrive at the following expression:
\begin{equation}\label{eq_loss_decrement}
    \norm{\mDelta_{t+1}}_{\fro}\leq \left(1-\frac{\eta\sigma_r^\star}{10}\right)\norm{\mDelta_t}_{\fro}
    +O(\eta)\left(\sqrt{\frac{\sigma_r^{\star 3}\mu r^2}{p}}\norm{\mE_t}+\sqrt{r}\sigma_1^{\star}\norm{(\id-\projection)(\mDelta_t)}\right)+O(\eta^2).
\end{equation}
The above inequality holds once the conditions $\frac{\sqrt{\sigma^\star_r}}{2}\leq \sigma_r(\mS_t)\leq \norm{\mS_t}\leq 2\sqrt{\sigma_1^\star}$ and $\norm{\mV_t-\mV^\star}\leq 0.1$ are met.
These conditions entail that during the initial phase of the algorithm, $\sigma_r(\mS_t)$ must undergo a fast growth, whereas $\mV_t$ must remain close to $\mV^\star$. Under these conditions, GD enters a fast linear convergence phase, provided that $\norm{(\id-\projection)(\mDelta_t)}\ll \norm{\mDelta_t}_F$. In fact, we can readily establish that $\norm{(\id-\projection)(\mDelta_t)}\leq c\norm{\mDelta_t}$ for some $c>0$. However, the challenge lies in ensuring that $c$ remains sufficiently small so as not to negate the effect of a constant factor improvement $1-\frac{\eta\sigma_r^\star}{10}$ in \Cref{eq_loss_decrement}. This phase continues until GD reaches an error level controlled by $\norm{\mE_t}$. Therefore, to prove the convergence of GD, we need to establish the following properties:
\begin{itemize}
    \item {\bf Fast growth of $\mS_t$:} Recall that $\max\{\norm{\mS_0},\norm{\mE_0}\}=O(\alpha)$. We need to ensure efficient growth of $\sigma_r(\mS_t)$ from $O(\alpha)$ to $\frac{\sqrt{\sigma^\star_r}}{2}$, while keeping $\norm{\mS_t}$ below $2\sqrt{\sigma_1^\star}$.
    \item {\bf Slow growth of $\mE_t$:} We need to show that while the signal term grows rapidly, the residual term $\mE_t$ grows at a much slower rate. Specifically, we will demonstrate that $T=O\left(\frac{1}{\eta\sigma_r^{\star}}\log(\frac{1}{\alpha})\right)$ suffices to ensure $\frac{\sqrt{\sigma^\star_r}}{2}\leq \sigma_r(\mS_t)\leq \norm{\mS_t}\leq 2\sqrt{\sigma_1^\star}$ while keeping $\norm{\mE_t}=O(\alpha)$.
    \item {\bf Small values of $\norm{(\id-\projection)(\mDelta_t)}$ and $\norm{\mV_t-\mV^\star}$:} Equally important is maintaining control over $\norm{(\id-\projection)(\mDelta_t)}$ and $\norm{\mV_t-\mV^\star}$. While $\norm{\mV_t-\mV^\star}\leq 0.1$ is needed as a crucial condition for~\Cref{eq_loss_decrement}, the value of $\norm{(\id-\projection)(\mDelta_t)}$ directly controls the convergence rate of GD.
\end{itemize}

To establish the above properties, we provide the one-step dynamics of $\mS_t$, $\mE_t$, and $(\id-\projection)(\mDelta_t)$.
\begin{lemma}[One-step dynamics, informal]\label{lem::onestep}
    Under the conditions of Theorem~\ref{thm::main}, the following inequalities hold with an overwhelming probability:
    \begin{align}
        \norm{\mS_{t+1}}\leq& \left(1+\eta\left(\sigma_1^\star-\norm{\mS_t}^2+O\left(\sigma_1^\star\norm{\mV_t-\mV^\star}+\norm{(\id-\projection)(\mDelta_t)}\right)\right)\right)\norm{\mS_t},\label{eq_St}\\
        \sigma_r(\mS_{t+1})\geq& \left(1+\eta\left(\sigma_r^\star-\sigma_r^2(\mS_t)-O\left(\sigma_1^\star\norm{\mV_t-\mV^\star}+\norm{(\id-\projection)(\mDelta_t)}\right)\right)\right)\sigma_r(\mS_t)\nonumber\\
        &+O(\eta)\cdot\left(\sigma_1^\star\norm{\mV_t-\mV^\star}+\norm{(\id-\projection)(\mDelta_t)}\right)\norm{\mE_t},\label{eq_St_lb}\\
        \norm{\mE_{t+1}}\leq& \left(1+O(\eta)\cdot\left(\sigma_1^\star\norm{\mV_t-\mV^\star}+\norm{(\id-\projection)(\mDelta_t)}\right)\right)\norm{\mE_t},\label{eq_Et}\\
        \norm{(\id-\projection)(\mDelta_t)}\leq& O\Bigg(\sqrt{\frac{d}{p}}\norm{\mDelta_t}\left(\norm{\mV^\star}_{2,\infty}^2+\norm{\mV_t}_{2,\infty}^2\right)+\sqrt{\frac{\sigma_1^\star\mu r}{p}}\norm{\mE_t}\Bigg).\label{eq_IR}
    \end{align}
\end{lemma}
 
Next, we provide a high-level overview of how Lemma~\ref{lem::onestep} can be used to establish the aforementioned properties. To this goal, we only focus on the initial phase of the algorithm, where both $\mS_t$ and $\mE_t$ are small. 
A more formal analysis for the entire trajectory is provided in Appendix~\ref{sec::proof-signal-residual-decomposition}.

To use Lemma~\ref{lem::onestep}, it suffices to control two key quantities: $\norm{\mV_t}_{2,\infty}$ and $\norm{\mV_t-\mV^\star}$. To illustrate this, let us assume that ${\mV_t}$ inherits the incoherence of ${\mV^\star}$, that is, $\norm{\mV_t}_{2,\infty}\leq O(\sqrt{{\mu r}/{d}})\ll 1$ for all $1\leq t\leq T$. Then, \Cref{eq_IR} suggests that $\norm{(\id-\projection)(\mDelta_t)}\leq c\norm{\mDelta_t}+O(\alpha)$, where $c=O(\mu r/\sqrt{pd})\ll 1$, thereby ensuring the necessary control over $\norm{(\id-\projection)(\mDelta_t)}$. 

On the other hand, the small values of $\norm{(\id-\projection)(\mDelta_t)}$ and $\norm{\mV_t-\mV^\star}$ play crucial roles in controlling the behavior of the signal and residual terms. To illustrate this, let us assume that at a certain point, $\norm{\mS_t}\geq 1.5\sqrt{\sigma_1^\star}$. Given that we have considered $\norm{(\id-\projection)(\mDelta_t)}$ and $\norm{\mV_t-\mV^\star}$ to be small, \Cref{eq_St} simplifies to $\norm{\mS_{t+1}}\leq (1-\eta\Omega(\sigma_1^\star))\norm{\mS_{t}}+O(\eta\sigma_1^\star\norm{\mE_t})$, effectively preventing $\norm{\mS_{t+1}}$ from further growth. With a similar reasoning, \Cref{eq_St_lb} simplifies to $\sigma_r(\mS_{t+1})\geq (1+\eta\Omega(\sigma_r^\star))\sigma_r(\mS_t)$. Here, we have leveraged the assumption that $\sigma_r^2(\mS_t)\ll \sigma_r^\star$ during the initial phase. This implies that $\sigma_r(\mS_{t})$ grows at a rate of $1+\eta\Omega(\sigma_r^\star)$. In contrast, \Cref{eq_Et} implies that $\norm{\mE_t}$ grows at a rate of $1+\eta O(\norm{(\id-\projection)(\mDelta_t)}+\sigma_1^{\star}\norm{\mV_t-\mV^\star})$, which is significantly slower than the growth rate of $\sigma_r(\mS_{t})$ because $\max\{\norm{(\id-\projection)(\mDelta_t)}, \sigma_1^{\star}\norm{\mV_t-\mV^\star}\}\ll \sigma_r^\star$. It is due to this discrepancy in the growth rates of $\norm{\mS_t}$ and $\norm{\mE_t}$ that GD enters the local linear convergence rate and achieves a final error of $O(\alpha)$. 

\subsection{Refined Leave-one-out Analysis with Weak Coupling}
\label{sec::incoherence-dynamic}
Indeed, it is not immediately evident why both $\norm{\mV_t}_{2,\infty}$ and $\norm{\mV_t-\mV^\star}$ would remain small. In fact, our initial intuition might suggest the opposite: recall that GD takes an aggressively large step size. Consequently, even a single GD update has the potential to disrupt the incoherence of $\mV_t$. Our key contribution is to establish that such disruption does not occur, even when the iterates are arbitrarily far from the ground truth. In essence, we show that despite the gradient update $\norm{\eta\projection(\mDelta_t)\mU_t}$ potentially having a magnitude of $\widetilde\Omega(1)$, its impact on $\mV_t$ is distributed fairly evenly across its elements. As a result, it has minimal influence on $\norm{\mV_t}_{2,\infty}$ and $\norm{\mV_t-\mV^\star}$.

We start by showing that a small $\norm{\mV_t}_{2,\infty}$ implies a small $\norm{\mV_t-\mV^\star}$.
\begin{sloppypar}
    \begin{lemma}[Small $\norm{\mV_t}_{2,\infty}$ implies small $\norm{\mV_t-\mV^\star}_{\fro}$, informal]
    \label{lem::fro-norm-control}
    Suppose that $\norm{\mS_t}\leq 2\sqrt{\sigma_1^\star}$ and $\norm{\mV_t-\mV^\star}\leq \frac{1}{2\kappa}$. With an overwhelming probability, we have
    \begin{equation}\nonumber
        \norm{\mV_{t+1}-\mV^\star}_{\fro}\leq \norm{\mV_{t}-\mV^\star}_{\fro}+O(\eta)\cdot\sigma_1^{\star}\sqrt{\frac{dr}{p}}\left(\norm{\mV^\star}^2_{2,\infty}+\norm{\mV_t}^2_{2,\infty}\right)+O(\eta^2).
    \end{equation}
\end{lemma}
\end{sloppypar}
Recall that, due to the incoherence of $\mV^\star$, we have $\norm{\mV^\star}_{2,\infty}=\sqrt{{\mu r}/{d}}$. Now, suppose we can further establish that $\mV_t$ enjoys a similar incoherence property. In such a case, the aforementioned lemma leads to $\norm{\mV_{t}-\mV^\star}\leq \norm{\mV_t-\mV^\star}_{\fro}\leq O\left(T\cdot\eta\sigma_1^{\star}\sqrt{\frac{\mu^2 r^3}{pd}}\right)+O(T\cdot\eta^2)$ for every $1\leq t\leq T$. Given the provided bounds on $T$ and $p$, this automatically establishes that $\norm{\mV_{t}-\mV^\star}$ remains small throughout the iterations. Therefore, it suffices to control the incoherence of $\norm{\mV_t}_{2,\infty}$.

Controlling $\norm{\mV_t}_{2, \infty}$, which necessitates estimating the $\ell_2$-norm of each row, requires a more fine-grained analysis than what is needed for the Frobenius norm. The primary challenge lies in the intricate correlations between the orthogonal matrix $\mV_t$ and the random observation set $\Omega$, which preclude the straightforward application of classical concentration inequalities. To effectively decouple these correlations, we introduce a technique called \textit{weakly-coupled leave-one-out analysis}. Before introducing our proposed methodology, it is essential to grasp the core principles of the classical leave-one-out analysis.

\paragraph{Local leave-one-out analysis.}  
When the search rank is exactly parameterized ($r=r'$) and the initial point is sufficiently close to the ground truth $\mU_0\mU_0^\top\approx \mX^\star$, \citet{ma2018implicit} established the incoherence of the iterates via the following leave-one-out sequences $\bigl\{\mU_t^{(l)}\bigr\}_{t=0}^{T}$ for each $1\leq l\leq d$:
\begin{equation}
    \dist\left(\mU_0,\mU_0^{(l)}\right) \approx 0, \quad\text{and}\quad \mU_{t+1}^{(l)}=\left(\mI-\eta \cR_{\Omega^{(l)}}\left(\mU_{t}^{(l)}\mU_{t}^{(l)\top}-\mX^{\star}\right)\right)\mU_{t}^{(l)},
    \label{eq::classic-leave-one-out}
\end{equation}
where $\cR_{\Omega^{(l)}}$ is the leave-one-out projection operator defined by
\begin{equation}\nonumber
    \cR_{\Omega^{(l)}}=\frac{1}{2p}\left({\proj_{\Omega^{(l)}}+\proj_{\Omega^{(l)}}^{\top}}\right),\quad\text{and}\quad[\proj_{\Omega^{(l)}}(\mX)]_{i, j}=\begin{cases}
        p\emX_{i, j} & \text{if $i= l$ or $j= l$},\\
        \emX_{i, j} & \text{if $(i, j)\in \Omega$, $i\neq l$, and $j\neq l$},\\
        0 & \text{otherwise}.
    \end{cases}
\end{equation}
The sole distinction between the projection operators $\cR_{\Omega^{(l)}}$ and $\cR_{\Omega}$ is in their $l$-th row and $l$-th column: in contrast to $\cR_{\Omega}$, the $l$-th row and $l$-th column of $\cR_{\Omega^{(l)}}$ are \textit{deterministically} set to match the corresponding values of $\mX$. This seemingly minor adjustment yields two important consequences: first, it ensures that $\mU_t\approx \mU_t^{(l)}$, and second, it guarantees that the behavior of $\mU_t^{(l)}$ remains independent of the random measurements in the $l$-th row and $l$-th column. This decoupling technique is the key to controlling the deviation of $\norm{\mU_t}_{2, \infty}$.
To formalize this intuition, let us define $\mU^\star = \mV^\star\mSigma^{\star 1/2}$. One can write
\begin{equation}
    \begin{aligned}
        \norm{\mU_t}_{2, \infty}&=\max_{1\leq l\leq d}\left\{\norm{\left[\mU_t\mH_t^{(l)}\right]_{l,\cdot}}\right\}\\
        &\leq \max_{1\leq l\leq d}\left\{\norm{\left[\mU^{\star}-\mU_t^{(l)}\mR_t^{(l)}\right]_{l,\cdot}}+\norm{\left[\mU_t\mH_t^{(l)}-\mU_t^{(l)}\mR_t^{(l)}\right]_{l,\cdot}}+\norm{\left[\mU^{\star}\right]_{l,\cdot}}\right\}\\
        &\leq \underbrace{\max_{1\leq l\leq d}\left\{\norm{\left(\mU^{\star}-\mU_t^{(l)}\mR_t^{(l)}\right)_{l,\cdot}}\right\}}_{\text{leave-one-out error}}+\underbrace{\max_{1\leq l\leq d}\left\{\norm{\mU_t\mH_t^{(l)}-\mU_t^{(l)}\mR_t^{(l)}}_{\fro}\right\}}_{\text{proximal error}}+\sqrt{\frac{\sigma_{1}^{\star}\mu r}{d}}.
    \end{aligned}
\end{equation}
Here $\mR_t^{(l)}$ and $\mH_t^{(l)}$ are orthogonal matrices defined as $\mR_t^{(l)}=\argmin_{\mO\in\cO_{r\times r}}\big\|\mU_t^{(l)}\mO-\mU^{\star}\big\|_{\fro}$ and $\mH_t^{(l)}=\argmin_{\mO\in\cO_{r\times r}}\big\|\mU_t^{(l)}\mO-\mU_t^{(l)}\mR_t^{(l)}\big\|_{\fro}$. Although it may not be immediately obvious, it can be shown that the $l$-th row of the matrix $\mU^{\star}-\mU_t^{(l)}\mR_t^{(l)}$ is purely \textit{deterministic}. Consequently, it becomes possible to efficiently control the leave-one-out error.
To tackle the proximal error, recall that the initial point $\mU_0\mU_0^\top$ is assumed to be close to $\mX^\star$. Within this region, the local landscape exhibits restricted strong convexity. This ensures that the true iterates $\mU_t$ and the leave-one-out versions $\mU_t^{(l)}$ become increasingly close, leading to a small proximal error. By combining these two arguments, we can guarantee the incoherence of the true iterates. Furthermore, the incoherence of $\mU_t$ automatically implies the incoherence of $\mV_t$, given that $\sqrt{\sigma_r^\star}\norm{\mV_t}_{2,\infty}\lesssim \norm{\mU_t}_{2,\infty}$ when $\mU_t\mU_t^\top\approx \mX^\star$. For more details, we refer interested readers to the discussions in \citep{ma2018implicit}.

\begin{sloppypar}
    While the classical leave-one-out analysis provides precise \textit{local} guarantees within the \textit{exactly-parameterized} regimes, we shed light on its limitations when applied {\it globally} in the \textit{over-parameterized} settings. A significant challenge arises from the discrepancy of the singular values of $\mU_t$ and $\mU_t^{(l)}$: although they may remain close to the singular values of $\mU^\star$ in the local regime, they can undergo substantial changes when positioned far from the true solution. Consequently, the original measure of proximal error based on $\dist(\mU_t, \mU_t^{(l)})$ loses its effectiveness as a reliable metric.
\end{sloppypar} 

Instead, recall that we only require controlling $\mV_t$, which unlike $\mU_t$, has \textit{unit} singular values. This motivates us to switch to a more stable metric---the divergence between the left column spaces of $\mU_t$ and $\mU_t^{(l)}$. 
However, an additional complication is that these left column spaces may also not align perfectly due to over-parameterization. Fortunately, by resorting to our proposed dynamic signal-residual decomposition, we can show that the iterates $\mU_t$ are well-approximated by the low-rank signal $\mU_t\approx \mS_t$. Therefore, it suffices to focus on controlling the discrepancy in the column spaces of $\mS_t$ and $\mS_t^{(l)}$, i.e., $\dist(\mV_t, \mV_t^{(l)})$. However, the new proximal error $\dist(\mV_t, \mV_t^{(l)})$ can still grow exponentially.
To explain the root cause of this exponential growth, we employ matrix Taylor expansion to derive the first-order approximations for $\mV_{t+1}$ and $\mV_{t+1}^{(l)}$:
\begin{equation}
    \mV_{t+1}\approx\mV_t+\eta\proj_{\mV_t}^{\perp} \mM_t\mV_t\quad \text{and}\quad \mV_{t+1}^{(l)}\approx\mV_t^{(l)}+\eta\proj_{\mV_t^{(l)}}^{\perp} \mM_t^{(l)}\mV_t^{(l)},
\end{equation}
where we define $\mM_t^{(l)}=\cR_{\Omega^{(l)}}\big(\mX^{\star}-\mU_t^{(l)}\mU_t^{(l)\top}\big)$.
To effectively control the proximal error, it is crucial to establish an upper bound for $\big\|\mM_t-\mM^{(l)}_t\big\|$. This distance tends to concentrate around $\big\|\mU_t\mU_t^{\top}-\mU_t^{(l)}\mU_t^{(l)\top}\big\|\approx \big\|\mS_t\mS_t^{\top}-\mS_t^{(l)}\mS_t^{(l)\top}\big\|$ when the sampling rate $p$ is sufficiently large. However, as previously noted, the singular values of $\mS_t$ and $\mS_t^{(l)}$ may diverge. This misalignment can lead to $\big\|\mM_t-\mM^{(l)}_t\big\|=\Omega(\sigma_1^{\star})$ in the worst case. Hence, the proximal error can grow exponentially.

\paragraph{Weakly-coupled leave-one-out analysis.} To remedy the alignment challenges identified earlier, we propose the following refined leave-one-out sequences $\big\{\widetilde \mV_t^{(l)}\big\}_{t=0}^{T}$:
\begin{equation}
    \widetilde\mV^{(l)}_{t+1}= \widetilde\mZ^{(l)}_{t+1}\left(\widetilde\mZ^{(l)\top}_{t+1}\widetilde\mZ^{(l)}_{t+1}\right)^{-1/2} \quad \text{where} \quad \widetilde\mZ^{(l)}_{t+1}=\left(\mI+\eta\widetilde\mM^{(l)}_t\right)\widetilde\mV^{(l)}_t \text{ and }\widetilde\mV^{(l)}_{0}=\mV^{\star}.
\end{equation}
In this context, $\widetilde\mM^{(l)}_t$ is defined as:
\begin{equation}
    \widetilde\mM^{(l)}_t=\cR_{\Omega^{(l)}}\left(\mX^{\star}-\widetilde\mV_{t}^{(l)}\mSigma_t\widetilde\mV_{t}^{(l)\top}\right)\quad \text{where }\mSigma_t=\mV_t^{\top}\mU_t\mU_t^{\top}\mV_t\in \bR^{r\times r}.
    \label{eq::M_t}
\end{equation}
Compared to the original $\mM_t^{(l)}$, we replace $\mSigma_t^{(l)}=\mV_t^{(l)\top}\mU_t^{(l)}\mU_t^{(l)\top}\mV_t^{(l)}$ by $\mSigma_t=\mV_t^{\top}\mU_t\mU_t^{\top}\mV_t$ in the definition of $\widetilde\mM^{(l)}_t$. Our analysis indicates that this adjustment significantly improves the control over the distance $\big\|\mM_t-\widetilde \mM^{(l)}_t\big\|=o(1)$ when the sampling rate $p$ is sufficiently large. Hence, the proximal error grows at a much slower rate. 

Despite their promise, these refined leave-one-out sequences do introduce a trade-off: the statistical independence inherent in the original leave-one-out sequences is compromised due to the inclusion of $\mSigma_t$. In other words, the $l$-th leave-one-out sequence $\big\{\widetilde \mV_t^{(l)}\big\}_{t=0}^{T}$ is no longer independent of the random measurements in the $l$-th row and $l$-th column. Nonetheless, we demonstrate that the resulting correlation is relatively weak, primarily because $\mSigma_t$ is a comparatively small $r\times r$ matrix. To control this statistical coupling, we employ a novel \textit{adaptive covering argument}, which can be of independent interest. This approach effectively mitigates the statistical coupling while incurring a mild increase in the required sample complexity, which remains only polynomial in $r$.

To formalize our arguments, we can decompose the refined leave-one-out sequences $\norm{\mV_t}_{2, \infty}$ as:
    \begin{equation}
        \label{eq::decomposition-of-V_t}
        \begin{aligned}
            \norm{\mV_t}_{2, \infty}&\leq \norm{\mV^{\star}-\mV_t}_{2, \infty}+\norm{\mV^{\star}}_{2, \infty}\\
            &=\max_{1\leq l\leq d}\left\{\norm{\left(\mV^{\star}-\mV_t\right)_{l,\cdot}}\right\}+\norm{\mV^{\star}}_{2, \infty}\\
            &\leq \underbrace{\max_{1\leq l\leq d}\left\{\norm{\left(\mV^{\star}-\widetilde \mV_t^{(l)}\right)_{l,\cdot}}\right\}}_{\text{refined leave-one-out error (Proposition~\ref{prop::dynamic-of-v})}}+\underbrace{\max_{1\leq l\leq d}\left\{\norm{\mV_t-\widetilde \mV_t^{(l)}}_{\fro}\right\}}_{\text{refined proximal error (Proposition~\ref{prop::incoherence-dynamic})}}+\sqrt{\frac{\mu r}{d}}.
        \end{aligned}
    \end{equation}
    Next, we characterize the dynamic of the refined leave-one-out error.
    \begin{proposition}[Refined leave-one-out error]
        \label{prop::dynamic-of-v}
        Suppose that $p\gtrsim \frac{\log(d)}{d}$ and $\norm{\mV^{\star}-\mV_t}\leq \frac{1}{2\kappa}$. With probability at least $1-\frac{1}{d^3}$, for any $1\leq t\leq T\lesssim\frac{1}{\eta\sigma_{r}^{\star}}\log\left(\frac{1}{\alpha}\right)$ and $1\leq l\leq d$, we have
        \begin{equation}\nonumber
            \norm{\left(\mV^{\star}- \widetilde\mV_{t+1}^{(l)}\right)_{l,\cdot}}\leq \left(1-0.5\eta\sigma_{r}^{\star}\right)\norm{\left(\mV^{\star}-\widetilde\mV_t^{(l)}\right)_{l,\cdot}}+O(\eta)\cdot\sigma_{1}^{\star}\frac{\kappa\mu^{1.5} r^{2}\log\left(\frac{1}{\alpha}\right)}{\sqrt{pd^2}}.
        \end{equation}
    \end{proposition}
    A simple inductive argument based on Proposition~\ref{prop::dynamic-of-v} reveals that the following inequality holds with an overwhelming probability for all $1\leq t\leq T$: 
\begin{align}\nonumber
    \Bigl\|\Bigl(\mV^{\star}- \widetilde\mV_{t}^{(l)}\Bigr)_{l,\cdot}\Bigr\|\lesssim\frac{\kappa^2\mu^{1.5} r^{2}\log\left(\frac{1}{\alpha}\right)}{\sqrt{pd^2}}\leq \sqrt{\frac{\mu r}{4d}}, \ \ \text{assuming}\ \ \ p\gtrsim \frac{\kappa^4\mu^2 r^3\log^2\left(\frac{1}{\alpha}\right)\log(d)}{d}.
\end{align}

    Next, we characterize the dynamic of the refined proximal error.

    \begin{sloppypar}
        \begin{proposition}[Refined proximal error]
        \label{prop::incoherence-dynamic}
        Suppose that $p\gtrsim \frac{\kappa^6\mu^4r^9\log^6(\frac{1}{\alpha})\log^2(d)}{d}$, $\norm{\mV^{\star}-\mV_t}\leq \frac{1}{2\kappa}$, and $\big\|\mV_t-\widetilde\mV_t^{(l)}\big\|_{\fro}\leq \sqrt{\frac{\mu r}{4d}}$. With probability at least $1-\frac{1}{d^3}$, for any $1\leq t\leq T\lesssim\frac{1}{\eta\sigma_{r}^{\star}}\log\left(\frac{1}{\alpha}\right)$ and $1\leq l\leq d$, we have
        \begin{equation}\nonumber
            \norm{\mV_{t+1}-\widetilde\mV_{t+1}^{(l)}}_{\fro}\leq \norm{\mV_t-\widetilde\mV_t^{(l)}}_{\fro}+O(\eta)\cdot \sigma_1^{\star}\sqrt{\frac{\kappa\mu^3r^{5.5}\log\left(\frac{1}{\alpha}\right)\log\left(d\right)}{\sqrt{pd}\cdot d}}.
        \end{equation}
    \end{proposition}
    \end{sloppypar}
    The above proposition implies that
    \begin{equation}\nonumber
        \begin{aligned}
            \norm{\mV_{t}-\widetilde\mV_{t}^{(l)}}_{\fro}\lesssim\eta \sigma_1^{\star}\sqrt{\frac{\kappa\mu^3r^{5.5}\log\left(\frac{1}{\alpha}\right)\log\left(d\right)}{\sqrt{pd}\cdot d}}\cdot T\lesssim \sqrt{\frac{\kappa^3\mu^3r^{5.5}\log^3\left(\frac{1}{\alpha}\right)\log\left(d\right)}{\sqrt{pd}\cdot d}}\leq \sqrt{\frac{\mu r}{4d}}.
        \end{aligned}
    \end{equation}
    Combining the above inequalities with the proposed decomposition in~\Cref{eq::decomposition-of-V_t} leads to:
    \begin{equation*}
        \norm{\mV_t}_{2, \infty}\leq \sqrt{\frac{\mu r}{4d}}+\sqrt{\frac{\mu r}{4d}}+\sqrt{\frac{\mu r}{d}}\leq \sqrt{\frac{4\mu r}{d}}, \quad \text{with probably at least $1-\frac{2}{d^3}$}.
    \end{equation*}
    This establishes the incoherence of $\mV_t$ for all $1\leq t\leq T$.
    
\subsection{Improved Results for Exact Parameterization}
\label{sec::exact-param}
Finally, we show that our analysis in the over-parameterized regime, combined with the following local convergence result for the exact parameterization regime by \citet{ma2018implicit}, readily establishes the proof of Theorem~\ref{thm::main-exact-param}.

\begin{theorem}[Local convergence of GD \protect{\citep[Theorem~2]{ma2018implicit}}]
    \label{thm::local-linear-convergence-exact-param}
    Consider~\ref{MC} with search rank $r'=r$. Suppose that the sampling rate satisfies $p\gtrsim \frac{\mu^3r^3\log^3(d)}{d}$. Consider the iterates of \ref{eq_GD} with step-size $\eta\leq \frac{2}{25\kappa\sigma^{\star}_{1}}$. Suppose that there exists $t_0\geq 0$ such that $\mU_{t_0}$ and the leave-one-out sequences $\bigl\{\mU_{t_0}^{(l)}\bigr\}_{l=0}^{d}$ defined in~\Cref{eq::classic-leave-one-out} satisfy:
    \begin{align}
        \dist\left(\mU_{t_0}, \mU^{\star}\right)&\leq O\Bigg(\sqrt{\frac{\sigma_r^{\star}\mu^3 r^3\log(d)}{pd^2}}\Bigg),\label{eq_U0}\\
        \max\left\{\dist\left(\mU_{t_0}, \mU_{t_0}^{(l)}\right), \dist\left( \mU_{t_0}^{(l)}, \mU^{\star}\right)\right\}&\leq O\Bigg(\sqrt{\frac{\sigma_r^{\star}\mu^3 r^3\log(d)}{pd^2}}\Bigg), 
        \quad\text{for all }1\leq l\leq d.\label{eq_U0l}
    \end{align}
    With probability at least $1-O\big(\frac{1}{d^3}\big)$, for all $t_0\leq t\leq t_0+O(d^5)$, we have 
    \begin{equation}\nonumber
        \begin{aligned}
            \norm{\mU_t\mU_t^{\top}-\mX^{\star}}_{\fro}\leq \norm{\mU_{t_0}\mU_{t_0}^{\top}-\mX^{\star}}_{\fro}\left(1-0.2\eta\sigma_{r}^{\star}\right)^{t-t_0}.
        \end{aligned}
    \end{equation}
    \label{thm::local-linear-convergence-exact-rank-r}
\end{theorem}

To prove Theorem~\ref{thm::main-exact-param}, it suffices to show that the conditions of the above theorem are met at a certain iteration $0 \leq t_0 \leq T$. This can be achieved by leveraging our result for the over-parameterized regime. In particular, upon choosing $\alpha = c\cdot\frac{\sigma_r^{\star}}{\kappa^{1.5} d}$ for sufficiently small $c>0$ in Theorem~\ref{thm::main}, one can show that both Conditions~(\ref{eq_U0}) and~(\ref{eq_U0l}) are satisfied with an overwhelming probability after $t_0=\widetilde O\big(\frac{1}{\eta\sigma_r^{\star}}\big)$ iterations. From this iteration onward, Theorem~\ref{thm::local-linear-convergence-exact-param} shows that the iterations of GD enter a local linear convergence regime, which readily establishes the final result of Theorem~\ref{thm::main-exact-param}.

\section{Conclusion and Future Directions}
In this paper, we prove the convergence of vanilla gradient descent (GD) with small initialization for symmetric matrix completion. Existing convergence results for this problem typically require explicit regularization or precise initializations. However, our work proves that neither condition is necessary for GD to converge. Moreover, our results also apply to the over-parameterized regime, where the rank of the true solution is unknown and over-estimated instead.

Although our required sample complexity $\widetilde O(dr^9)$ is optimal with respect to the dimension $d$, it remains sub-optimal with respect to the rank $r$. Specifically, it exceeds the sample complexity of regularized GD, which stands at $\widetilde O(dr^2)$~\citep{chen2015fast}. We expect our analysis can be sharpened to achieve a similar sample complexity.

We anticipate that our findings will pave the way for broader results extending beyond symmetric matrix completion. In particular, our proposed weakly-coupled leave-one-out analysis relaxes several stringent conditions of classical leave-one-out analysis, making it highly applicable for the global analysis of GD. We believe that this approach, along with potential variations, holds promise for explaining the favorable performance of GD or its variants in various statistical learning problems. 

\section*{Acknowledgment}
We thank Richard Y. Zhang and Cédric Josz for insightful discussions. This work is supported, in part, by NSF CAREER Award CCF-2337776, NSF Award DMS-2152776, and ONR Award N00014-22-1-2127.

\bibliography{ref.bib}
\newpage
\appendix
\resumetocwriting
\tableofcontents
\newpage
\section{Related Work}
\paragraph{Nonconvex matrix completion.} To solve the matrix completion problem, several algorithms based on convex optimization have been developed \citep{candes2012exact,candes2010power,gross2011recovering}, offering excellent theoretical guarantees. However, in high-dimensional scenarios, convex optimization techniques require significant memory and computational resources due to the iterative singular
value decompositions. To overcome these limitations, researchers have shifted towards nonconvex optimization techniques using first-order methods such as GD \citep{sun2016guaranteed}, projected GD \citep{zheng2016convergence}, and alternating minimization \citep{jain2013low}. Specifically, \citet{sun2016guaranteed} demonstrate that GD can achieve local linear convergence provided that the initialization is close to the ground truth. Subsequent studies provide the global convergence guarantees for the first-order methods by showing the benign landscape of these nonconvex optimization formulations. Specifically, they reveal that the loss landscape has no spurious local minima and all the saddle points are strict \citep{ge2016matrix, ge2017no, chen2017memory, fattahi2020exact}. Nonetheless, these advancements necessitate either an explicit $\ell_{2, \infty}$-norm regularization or a projection step to maintain the incoherence of the iterates. Moreover, these works are only applicable in the exactly-parameterized setting~\cite{ma2023optimization}.
For a more detailed exploration of matrix completion and its variants, we refer the readers to the comprehensive survey by \citet{chi2019nonconvex}.

\paragraph{Leave-one-out analysis.}
Leave-one-out analysis is a powerful statistical technique designed to decouple correlations among individual entries of a stochastic process. Initially employed by \citet{el2013robust} to establish asymptotic sampling distributions for robust estimators in high or moderate dimensional regression, this technique has been proven invaluable across a broad spectrum of applications. For instance, \citet{abbe2020entrywise} utilized it to control $\ell_{\infty}$ estimation errors for eigenvectors in stochastic spectral problems, enabling precise spectral clustering in community detection without the need for data cleaning or regularization. More relevantly, \citet{ma2017implicit} applied leave-one-out analysis to demonstrate the local linear convergence of GD for the unregularized and symmetric matrix completion. Their approach not only elucidated the convergence properties of GD in matrix completion but also paved the way for similar analyses in other low-rank recovery challenges, such as phase retrieval and blind deconvolution. Extending these insights, \citet{chen2020nonconvex} and \citet{kim2022rank} broadened the scope of this analysis to include asymmetric matrix completion and global convergence in rank-1 scenarios, respectively. Furthermore, leave-one-out analysis has facilitated advancements in Singular Value Projection (SVP) for matrix completion, as demonstrated by \citet{ding2020leave} and has been instrumental in analyzing gradient descent with random initialization for phase retrieval, as shown by \citet{chen2019gradient}.

\paragraph{Implicit regularization of GD in other applications.}
Indeed, the conventional wisdom in statistics suggests that increasing the number of parameters beyond the true dimension without proper regularization would lead to inferior solutions due to \textit{overfitting}. However, a growing body of works show that, for a large class of learning problems, GD leads to surprisingly good solutions, due to its \textit{implicit regularization} property. For instance, it is known that GD recovers the true low-dimensional solutions in matrix factorization and sensing~\citep{gunasekar2018implicit, li2018algorithmic, stoger2021small}, tensor decomposition~\citep{wang2020beyond, ge2021understanding}, deep linear neural networks~\citep{arora2018convergence, ma2022blessing}, and beyond~\citep{ma2022behind}. However, {the current theory behind the success of GD in these classes of problems hinges heavily upon a norm-preserving property of the measurements, known as the restricted isometry property (RIP), limiting its applicability in settings where RIP is not satisfied.

\section{Preliminaries}

\subsection{Otuline of the Appendix}
The structure of the appendix is as follows. In the remainder of this section, we introduce additional notation. Following this, we present key intermediate lemmas (Lemmas~\ref{lem::helper-lemma-U-t} to \ref{lem::helper-lemma-A-t}) crucial for our main proofs. Section~\ref{sec::proof-signal-residual-decomposition} delves into a detailed proof of the signal and residual dynamics, starting with a refined version of Lemma~\ref{lem::onestep} (Proposition~\ref{prop_onestep}) that takes into account the incoherence of $\mV_t$. Additionally, the proof of Lemma~\ref{lem_descent} is provided in this section. Moving on to Section~\ref{sec_main_proofs}, we present the proofs of our main theorems. Section~\ref{sec_incoherence} presents the key novelty of our paper, focusing on establishing the incoherence of $\mV_t$ via weakly-coupled leave-one-out analysis. The validation of different initialization schemes, as presented in Lemma~\ref{lem_init}, is addressed in Section~\ref{sec_init_proof}. In Section~\ref{sec::concentration-inequalities-matrix-completion}, we compile several known results on matrix completion crucial to our arguments. Lastly, Section~\ref{sec_aux_lemmas} collects several basic lemmas, which we include for completeness. 

\subsection{Additional Notations} We introduce some additional notations that will be used throughout the appendix. The max-norm of $\mX$, denoted as $\norm{\mX}_{\max}$, is defined as $\max_{i,j}|\emX_{i,j}|$. We define the operator and Frobenius norm ball as $\cB_{\mathrm{op}}^{d_1\times d_2}(r):=\{\mX\in \bR^{d_1\times d_2}:\norm{\mX}\leq r\}$ and $\cB_{\fro}^{d_1\times d_2}(r):=\{\mX\in \bR^{d_1\times d_2}:\norm{\mX}_{\fro}\leq r\}$, respectively. For any matrix $\mX$, we denote its SVD as $\mX=\mL_{\mX}\mSigma_{\mX}\mR_{\mX}^{\top}$. We denote $\cS_{d\times d}$ as the set of all the symmetric matrices $\mX\in \bR^{d\times d}$. In the appendix, $\const, \const_1, \const_2, \ldots$ denote fixed universal constants, while $C, C_1, C_2, c_1, c_2, \ldots$ represent universal constants whose specific values may vary depending on the context.

Throughout the appendix, our arguments are conditioned on the following good event without further explanation. We define the random observation matrix $\mOmega$ as
\begin{equation}
    \mOmega_{i, j}=\begin{cases}
        1 & \text{if $(i, j)\in \Omega$},\\
        0 & \text{otherwise}.
    \end{cases}
\end{equation}
Then, the good event can be defined as
\begin{equation}
    \event=\left\{\norm{\frac{\mOmega+\mOmega^{\top}}{2p}-\mJ}\leq \const\sqrt{\frac{d}{p}}\right\}.
\end{equation}
Here $\mJ$ is the all-one matrix. According to Lemma~\ref{lem::Omega}, we have $\bP(\event)\geq 1-\frac{1}{d^3}$. 

\subsection{Important Intermediate Lemmas}
Next, we collect some useful intermediate results that will be directly used throughout our proofs. We also note that some of these intermediate results rely on concentration inequalities for matrix completion, which are thoroughly discussed in Appendix~\ref{sec::concentration-inequalities-matrix-completion}. Before proceeding, we define the following notations
\begin{equation}
    \begin{aligned}
        \mDelta_t=\mX^{\star}-\mU_t\mU_t^{\top}, \quad \mM_t=\projection\left(\mX^{\star}-\mU_t\mU_t^{\top}\right), \quad \text{and} \quad \mLambda_t=\mS_t\mE_t^{\top}+\mE_t\mS_t^{\top}+\mE_t\mE_t^{\top}.
    \end{aligned}
\end{equation}
Moreover, we introduce the following term
\begin{equation}
    \mA_t=-\eta^2\mM_t\mV_t\mV_t^{\top}\mM_t\mV_{t}-0.5\eta^2\mV_t\mV_t^{\top}\mM_t^2\mV_t-0.5\eta^3\mM_t\mV_t\mV^{\star\top}\mM_t^2\mV_t+\left(\mI +\eta \mM_t\right)\mV_t\mR(\mY_t)
\end{equation}
where $\mY_t=\mV_t^{\top}\left(2\eta \mM_t+\eta^2\mM_t^2\right)\mV_t$ and $\mR(\mX)=\sum_{k=2}^{\infty}\frac{(-1)^k(2k)!}{4^k(k!)^2}\mX^k$. This notion of $\mA_t$ will be used when controlling the higher-order terms with respect to the step-size $\eta$. We are now ready to statement our helper lemmas.
\begin{lemma}[Helper lemma for $\mU_t$]
    \label{lem::helper-lemma-U-t}
    Suppose that $\norm{\mS_t}\leq 2\sqrt{\sigma^{\star}_{1}}$ and $\norm{\mE_t}\leq \sqrt{\frac{\sigma^{\star}_{1}\mu r}{d}}$. Then, we have 
    \begin{equation}
        \label{eq::helper-U}
        \begin{aligned}
            \norm{\mU_t\mU_t^{\top}}_{\fro}\leq 8\sqrt{r}\sigma_1^{\star}.
        \end{aligned}
    \end{equation}
\end{lemma}
\begin{proof}
    Applying triangle inequality, we have 
    \begin{equation}
        \begin{aligned}
            \norm{\mU_t\mU_t^{\top}}_{\fro}&\leq \norm{\mS_t\mS_t^{\top}}_{\fro}+2\norm{\mS_t\mE_t^{\top}}_{\fro}+\norm{\mE_t\mE_t^{\top}}_{\fro}\\
            &\leq \sqrt{r}\norm{\mS_t}^2+2\sqrt{r}\norm{\mS_t}\norm{\mE_t}+\sqrt{d}\norm{\mE_t}^2\\
            &\leq 8\sqrt{r}\sigma_1^{\star}.
        \end{aligned}
    \end{equation}
    Here in the last inequality, we use the assumptions $\norm{\mS_t}\leq 2\sqrt{\sigma^{\star}_{1}}$ and $\norm{\mE_t}\leq \sqrt{\frac{\sigma^{\star}_{1}\mu r}{d}}$ and the fact that $d\gg \mu r$.
\end{proof}
\begin{lemma}[Helper lemma for $\mLambda_t$]
    \label{lem::helper-lemma-lambda-t}
    Suppose that $\norm{\mS_t}\leq 2\sqrt{\sigma^{\star}_{1}}$, $\norm{\mV_t}_{2,\infty}\leq 2\sqrt{\frac{\mu r}{d}}$, and $\norm{\mE_t}\leq \sqrt{\frac{\sigma^{\star}_{1}\mu r}{81d}}$. Then, conditioned on $\event$, we have 
    \begin{equation}
        \label{eq::helper-residual-term}
        \begin{aligned}
            \norm{\mLambda_t}&\leq 5\sqrt{\sigma^{\star}_{1}}\norm{\mE_t},\\
            \norm{\projection(\mLambda_t)}&\leq 10 \const\sqrt{\frac{\sigma^{\star}_{1}\mu r}{p}}\norm{\mE_t},\\ \norm{\left(\id-\projection\right)(\mLambda_t)}&\leq 9 \const\sqrt{\frac{\sigma^{\star}_{1}\mu r}{p}}\norm{\mE_t}.
        \end{aligned}
    \end{equation}
\end{lemma}
\begin{proof}
    First, we can bound $\norm{\mLambda_t}$ as follows 
    \begin{equation}
        \begin{aligned}
            \norm{\mLambda_t}&\leq \norm{\mS_t\mE_t^{\top}}+\norm{\mE_t\mS_t^{\top}}+\norm{\mE_t\mE_t^{\top}}\leq 5\sqrt{\sigma^{\star}_{1}}\norm{\mE_t}
        \end{aligned}
    \end{equation}  
    where we use the assumptions $\norm{\mS_t}\leq 2\sqrt{\sigma^{\star}_{1}}$ and $\norm{\mE_t}\leq \sqrt{\frac{\sigma^{\star}_{1}\mu r}{d}}\leq \sqrt{\sigma^{\star}_{1}}$.
    Next, we control $\norm{\left(\id-\projection\right)(\mLambda_t)}$. To this end, we first apply triangle inequality to obtain
    \begin{equation}
        \begin{aligned}
            \norm{\left(\id-\projection\right)(\mLambda_t)}&\leq 2\norm{\left(\id-\projection\right)\left(\mS_t\mE_t^{\top}\right)}+\norm{\left(\id-\projection\right)\left(\mE_t\mE_t^{\top}\right)}\\
            &\stackrel{(a)}{\leq}  \const\sqrt{\frac{d}{p}}\norm{\mE_t}_{2, \infty}\left(\norm{\mE_t}_{2, \infty}+2\norm{\mS_t}_{2, \infty}\right)\\
            &\stackrel{(b)}{\leq} \const\sqrt{\frac{d}{p}}\norm{\mE_t}\left(\norm{\mE_t}+8\sqrt{\frac{\sigma^{\star}_{1}\mu r}{d}}\right)\\
            &\leq 9  \const\sqrt{\frac{\sigma^{\star}_{1}\mu r}{p}}\norm{\mE_t}.
        \end{aligned}
    \end{equation}
    Here in $(a)$, we apply Lemma~\ref{lem::uniform-concentration-operator-norm}. In $(b)$, we use the facts that $\norm{\mE_t}_{2, \infty}\leq \norm{\mE_t}\leq \sqrt{\frac{\sigma^{\star}_{1}\mu r}{d}}$ and
    \begin{equation}
        \norm{\mS_t}_{2,\infty}=\norm{\mV_t\mV_t^{\top}\mU_t}_{2, \infty}\stackrel{\text{Lemma~\ref{lem::uniform-concentration-operator-norm}}}{\leq} \norm{\mV_t}_{2, \infty}\norm{\mS_t}\leq 2\sqrt{\frac{\mu r}{d}}\cdot 2\sqrt{\sigma^{\star}_{1}}=4\sqrt{\frac{\sigma^{\star}_{1}\mu r}{d}}.
    \end{equation}
    In the last inequality, we use the fact that $\norm{\mS_t}\leq 2\sqrt{\sigma^{\star}_{1}}$ and $\norm{\mV_t}_{2,\infty}\leq 2\sqrt{\frac{\mu r}{d}}$.
    Lastly, applying triangle inequality leads to 
    \begin{equation}
        \begin{aligned}
            \norm{\projection(\mLambda_t)}&\leq \norm{\mLambda_t}+\norm{\left(\id-\projection\right)(\mLambda_t)}\\
            &\leq 5\sqrt{\sigma^{\star}_{1}}\norm{\mE_t}+9  \const\sqrt{\frac{\sigma^{\star}_{1}\mu r}{p}}\norm{\mE_t}\\
            &\leq 10 \const\sqrt{\frac{\sigma^{\star}_{1}\mu r}{p}}\norm{\mE_t},
        \end{aligned}
    \end{equation}
    where the last inequality is due to $p\leq \frac{1}{25}\const^2\mu r$.
\end{proof}
    \begin{lemma}[Helper lemma for $\mDelta_t$]
        \label{lem::helper-lemma-delta-t}
        Under the same conditions as Lemma~\ref{lem::helper-lemma-lambda-t} with the additional assumption that $\norm{\mV_t-\mV^{\star}}\leq \const_1\frac{\kappa\mu r^{1.5}\log\left(\frac{1}{\alpha}\right)}{\sqrt{pd}}$, we have 
        \begin{equation}
            \begin{aligned}
                \norm{\mDelta_t}&\leq 5\sigma_{1}^{\star},\\
                \norm{\left(\id-\projection\right)(\mDelta_t)}&\leq 5  \const\sqrt{\frac{\mu^2 r^2}{pd}}\norm{\mDelta_t}+10  \const\sqrt{\frac{\sigma^{\star}_{1}\mu r}{p}}\norm{\mE_t},\\
                \norm{\left(\id-\projection\right)(\mDelta_t)}&\leq 21\const\sigma_{1}^{\star}\sqrt{\frac{\mu^2 r^2}{pd}},\\
                \norm{\mM_t}&\leq \left(1+5  \const\sqrt{\frac{\mu^2 r^2}{pd}}\right)\norm{\mDelta_t}+10  \const\sqrt{\frac{\sigma^{\star}_{1}\mu r}{p}}\norm{\mE_t},\\
                \norm{\mM_t}&\leq 6\sigma^{\star}_{1},\\
                \norm{\mM_t\proj^{\perp}_{\mV_t}}&\leq 2\const_1\sigma_1^{\star}\frac{\kappa\mu r^{1.5}\log\left(\frac{1}{\alpha}\right)}{\sqrt{pd}}.
            \end{aligned}
        \end{equation}
    \end{lemma}
    \begin{proof}
    We first control $\norm{\mDelta_t}$ as follows 
    \begin{equation}
        \begin{aligned}
            \norm{\mDelta_t}&\leq \norm{\mX^{\star}-\mS_t\mS_t^{\top}}+\norm{\mLambda_t}\\
            &\leq \max\{\norm{\mX^{\star}}, \norm{\mS_t}^2\}+5\sqrt{\sigma^{\star}_{1}}\norm{\mE_t}\\
            &\leq 4\sigma^{\star}_{1}+5\sqrt{\sigma^{\star}_{1}}\norm{\mE_t}\\
            &\leq 5\sigma^{\star}_{1}.
        \end{aligned}
    \end{equation}
    Next, we control $\norm{\left(\id-\projection\right)(\mDelta_t)}$. To this end, applying triangle inequality leads to
    \begin{equation}
        \begin{aligned}
            \norm{\left(\id-\projection\right)(\mDelta_t)}&\leq \underbrace{\norm{(\id-\projection)\left(\mX^{\star}-\mS_t\mS_t^{\top}\right)}}_{:=(\rom{1})}+\underbrace{\norm{(\id-\projection)(\mLambda_t)}}_{:=(\rom{2})}.
        \end{aligned}
    \end{equation}
    For $(\rom{1})$, applying Lemma~\ref{lem::2-norm-diff}, we have
    \begin{equation}
        \begin{aligned}
            (\rom{1})&\leq  \const\sqrt{\frac{d}{p}}\norm{\mX^{\star}-\mS_t\mS_t^{\top}}\left(\norm{\mV^{\star}}_{2, \infty}^2+\norm{\mV_t}_{2, \infty}^2\right)\\
            &\stackrel{(a)}{\leq} 5 \const\sqrt{\frac{\mu^2 r^2}{pd}}\norm{\mX^{\star}-\mS_t\mS_t^{\top}}\\
            &\leq 5  \const\sqrt{\frac{\mu^2 r^2}{pd}}\left(\norm{\mDelta_t}+\norm{\mLambda_t}\right)\\
            &\stackrel{(b)}{\leq} 5  \const\sqrt{\frac{\mu^2 r^2}{pd}}\norm{\mDelta_t}+25  \const\sqrt{\frac{\sigma^{\star}_{1}\mu^2r^2}{pd}}\norm{\mE_t}.
        \end{aligned}
    \end{equation}
    Here in $(a)$, we use the assumption that $\norm{\mV_t}_{2, \infty}\leq 2\sqrt{\frac{\mu r}{d}}$. In $(b)$, we use the result from Lemma~\ref{lem::helper-lemma-lambda-t} that $\norm{\mLambda_t}\leq 5\sqrt{\sigma^{\star}_{1}}\norm{\mE_t}$.
    On the other hand, we know that $(\rom{2})\leq 9  \const\sqrt{\frac{\sigma^{\star}_{1}\mu r}{p}}\norm{\mE_t}$ due to Lemma~\ref{lem::helper-lemma-lambda-t}.
    Overall, we conclude that 
    \begin{equation}
        \begin{aligned}
            \norm{\left(\id-\projection\right)(\mDelta_t)}&\leq 5  \const\sqrt{\frac{\mu^2 r^2}{pd}}\norm{\mDelta_t}+25  \const\sqrt{\frac{\sigma^{\star}_{1}\mu^2r^2}{pd}}\norm{\mE_t}+9  \const\sqrt{\frac{\sigma^{\star}_{1}\mu r}{p}}\norm{\mE_t}\\
            &\leq 5  \const\sqrt{\frac{\mu^2 r^2}{pd}}\norm{\mDelta_t}+10  \const\sqrt{\frac{\sigma^{\star}_{1}\mu r}{p}}\norm{\mE_t}.
        \end{aligned}
    \end{equation}
    In the final inequality, we make the assumption $d\geq 9\mu r$ without loss of generality. This assumption simplifies the presentation of the proof but does not impact the final result of the paper.

    Furthermore, upon noticing that $\norm{\mX^{\star}-\mS_t\mS_t^{\top}}\leq \max\{\norm{\mX^{\star}}, \norm{\mS_t}^2\}\leq 4\sigma^{\star}_{1}$, we have
    \begin{equation}
        \begin{aligned}
            \norm{\left(\id-\projection\right)(\mDelta_t)}&\leq 5 \const\sqrt{\frac{\mu^2 r^2}{pd}}\norm{\mX^{\star}-\mS_t\mS_t^{\top}}+\norm{(\id-\projection)(\mLambda_t)}\\
            &\leq 20\const\sigma_{1}^{\star}\sqrt{\frac{\mu^2 r^2}{pd}}+9  \const\sqrt{\frac{\sigma^{\star}_{1}\mu r}{p}}\norm{\mE_t}\\
            &\leq 21\const\sigma_{1}^{\star}\sqrt{\frac{\mu^2 r^2}{pd}}.
        \end{aligned}
    \end{equation}
    Next, combining the above two inequalities, we can control $\norm{\mM_t}$ as follows:
    \begin{equation}
        \begin{aligned}
            \norm{\mM_t}&\leq\norm{\mDelta_t}+\norm{\left(\id-\projection\right)(\mDelta_t)}\\
            &\leq \norm{\mDelta_t}+5  \const\sqrt{\frac{\mu^2 r^2}{pd}}\norm{\mDelta_t}+10  \const\sqrt{\frac{\sigma^{\star}_{1}\mu r}{p}}\norm{\mE_t}\\
            &\leq \left(1+5  \const\sqrt{\frac{\mu^2 r^2}{pd}}\right)\norm{\mDelta_t}+10  \const\sqrt{\frac{\sigma^{\star}_{1}\mu r}{p}}\norm{\mE_t}.
        \end{aligned}
    \end{equation}
    Furthermore, we can also bound $\norm{\mM_t}$ as
    \begin{equation}
        \begin{aligned}
            \norm{\mM_t}& \leq\norm{\mDelta_t}+\norm{\left(\id-\projection\right)(\mDelta_t)}\\
            &\leq 5\sigma^{\star}_{1}+21\const\sigma_{1}^{\star}\sqrt{\frac{\mu^2 r^2}{pd}}\\
            &\leq 6\sigma^{\star}_{1}.
        \end{aligned}
    \end{equation}
    Lastly, for $\norm{\mM_t\proj^{\perp}_{\mV_t}}$, we have the following decomposition 
    \begin{equation}
        \begin{aligned}
            \norm{\mM_t\proj_{\mV_t}^{\perp}}&\leq \norm{\left(\id-\projection\right)(\mDelta_t)\proj_{\mV_t}^{\perp}}+\norm{\left(\mX^{\star}-\mU_t\mU_t^{\top}\right)\proj_{\mV_t}^{\perp}}\\
            &\leq \norm{\left(\id-\projection\right)(\mDelta_t)}+\norm{\mV^{\star}\mSigma \left(\mV^{\star}-\mV_t\right)^{\top}\proj_{\mV_t}^{\perp}}+\norm{\mU_t\mE_t^{\top}\proj_{\mV_t}^{\perp}}\\
            &\leq 21\const\sigma_{1}^{\star}\sqrt{\frac{\mu^2 r^2}{pd}}+\sigma^{\star}_{1}\norm{\mV^{\star}-\mV_t}+\norm{\mE_t}\left(\norm{\mS_t}+\norm{\mE_t}\right)\\
            &\leq 22  \const\frac{\sigma^{\star}_{1}\mu r}{\sqrt{pd}}+\sigma^{\star}_{1}\norm{\mV^{\star}-\mV_t}\\
            &\stackrel{(a)}{\leq} 22  \const\frac{\sigma^{\star}_{1}\mu r}{\sqrt{pd}}+\const_1\sigma_1^{\star}\frac{\kappa\mu r^{1.5}\log\left(\frac{1}{\alpha}\right)}{\sqrt{pd}}\\
            &\leq 2\const_1\sigma_1^{\star}\frac{\kappa\mu r^{1.5}\log\left(\frac{1}{\alpha}\right)}{\sqrt{pd}}.
        \end{aligned}
    \end{equation}
    Here in $(a)$, we apply Lemma~\ref{lem::fro-norm-control}.
\end{proof}

\begin{lemma}[Helper lemma for $\mA_t$]
    \label{lem::helper-lemma-A-t}
    Under the same conditions as Lemma~\ref{lem::helper-lemma-lambda-t}, we have 
    \begin{equation}
        \norm{\mA_t}\leq 300\eta^2\sigma_1^{\star 2}.
    \end{equation}
\end{lemma}
\begin{proof}
    We first use triangle inequality to bound $\norm{\mA_t}$ as follows
    \begin{equation}
        \begin{aligned}
            \norm{\mA_t}\leq 1.5\eta^2\norm{\mM_t}^2+0.5\eta^3\norm{\mM_t}^3+\left(1+\eta\norm{\mM_t}\right)\norm{\mR(\mY_t)}.
        \end{aligned}
    \end{equation}
    Next, Lemma~\ref{lem::helper-lemma-delta-t} tells us that $\norm{\mM_t}\leq 6\sigma^{\star}_{1}$ conditioned on $\event$. For $\norm{\mR(\mY_t)}$, we first have 
    \begin{equation}
        \norm{\mR(\mY_t)}\leq \sum_{k=2}^{\infty}\frac{(2k)!}{4^k(k!)^2}\norm{\mY_t}^k=\frac{2+\sqrt{1-\norm{\mY_t}}}{2\sqrt{1-\norm{\mY_t}}\left(1+\sqrt{1-\norm{\mY_t}}\right)^2}\norm{\mY_t}^2.
    \end{equation}
    Note that $\norm{\mY_t}\leq \norm{2\eta\mM_t+\eta^2\mM_t^2}\leq 20\eta\sigma^{\star}_{1}$ and the right-hand side is an increasing function of $\norm{\mY_t}$.
    Therefore, we derive that 
    \begin{equation}
        \norm{\mR(\mY_t)}\leq \frac{1}{2}\norm{\mY_t}^2\leq 200\eta^2\sigma_1^{\star 2}.
    \end{equation}
    This implies that 
    \begin{equation}
        \norm{\mA_t}\leq 300\eta^2\sigma_1^{\star 2}.
    \end{equation}
\end{proof}

\section{Proofs for Dynamic Signal-residual Decomposition}
\label{sec::proof-signal-residual-decomposition}We first present a more precise version of the one-step dynamics of the signal and residual terms.

\begin{proposition}\label{prop_onestep}
    Suppose that $\norm{\mS_t}\leq 2\sqrt{\sigma^{\star}_{1}}$, $\norm{\mE_t}\leq \sqrt{\frac{\sigma_1^{\star}}{d}}$, $\norm{\mV_t}_{2, \infty}\leq 2\sqrt{\frac{\mu r}{d}}$ and $\norm{\mV^{\star}-\mV_{t}}_{\fro}\leq  \const_1\frac{\kappa\mu r^{1.5}\log\left(\frac{1}{\alpha}\right)}{\sqrt{pd}}$. Then, the following dynamics hold conditioned on $\event$:
    \begin{align}
        \sigma_{r}(\mS_{t+1})&\geq \left(1+0.8\eta\sigma^{\star}_{r}-\eta\sigma_{r}^2(\mS_t)\right)\sigma_{r}(\mS_t)-6\const_1\eta \frac{\sigma^{\star}_{1}\kappa\mu r^{1.5}\log\left(\frac{1}{\alpha}\right)}{\sqrt{pd}}\norm{\mE_t},\tag{minimal signal dynamic}\\
        \norm{\mS_{t+1}}&\leq 2\sqrt{\sigma^{\star}_{1}},
                \tag{maximal signal dynamic}\\
                \norm{\mE_{t+1}}&\leq \left(1+ 2\const_1\eta\frac{\sigma^{\star}_{1}\kappa\mu r^{1.5}\log\left(\frac{1}{\alpha}\right)}{\sqrt{pd}}\right)\norm{\mE_t}.\tag{residual dynamic}
    \end{align} 
    Additionally, if $\sigma_{r}(\mS_t)\geq \frac{\sqrt{\sigma^{\star}_{r}}}{2}$, then 
    \begin{align}
        \norm{\mX^{\star}-\mU_{t+1}\mU_{t+1}^{\top}}_{\fro}\leq \left(1-\frac{1}{10}\eta\sigma_{r}^{\star}\right)\norm{\mX^{\star}-\mU_{t}\mU_{t}^{\top}}_{\fro}+\const_6\eta\sqrt{\frac{\sigma^{\star 3}_{1}\mu r^2}{p}}\norm{\mE_t}.\tag{error dynamic}
    \end{align}
\end{proposition}
The key distinction between the above proposition and Lemma~\ref{lem::onestep} lies in a finer control over the one-step dynamics, accompanied by additional assumptions on $\norm{\mE_t}$, $\norm{\mV_t}_{2,\infty}$, and $\norm{\mV_t-\mV^\star}_{\fro}$. We emphasize that the one-step dynamics in Lemma~\ref{lem::onestep} are derived from the proof of this proposition. Moreover, the proof of the error dynamics (last inequality in Proposition~\ref{prop_onestep}) will incorporate the proof of the descent lemma (Lemma~\ref{lem_descent}).
\subsection{Proof of Signal Dynamic} 
We notice that
\begin{equation}
    \mS_{t+1}=\proj_{\mV_{t+1}}\left(\mI+\eta\mM_t\right)\left(\mS_t+\mE_t\right)=\left(\mI+\eta\mM_t\right)\mS_t+\proj_{\mV_{t+1}}\left(\mI+\eta\mM_t\right)\mE_t.
\end{equation}
Here the second equality follows from the definition of $\proj_{\mV_{t+1}}$. 

\paragraph*{Maximal signal dynamic.} We first provide an upper-bound for $\norm{\mS_{t+1}}$ by 
\begin{equation}
    \begin{aligned}
        \norm{\mS_{t+1}}&\leq \underbrace{\norm{\left(\mI+\eta\mM_t\right)\mS_t}}_{:=(\rom{1})}+\underbrace{\norm{\proj_{\mV_{t+1}}\left(\mI+\eta\mM_t\right)\mE_t}}_{:=(\rom{2})}.
    \end{aligned}
\end{equation}
Next, we further control $(\rom{1})$ by 
\begin{equation}
    \begin{aligned}
        (\rom{1})&\leq \norm{\left(\mI+\eta \mDelta_t\right)\mS_t}+\eta \norm{\left(\id-\projection\right)(\mDelta_t)}\norm{\mS_t}\\
        &\stackrel{(a)}{\leq} \norm{\left(\mI+\eta \mDelta_t\right)\mS_t} + \eta\cdot 21\const\sigma_{1}^{\star}\sqrt{\frac{\mu^2 r^2}{pd}}\cdot 2\sqrt{\sigma_{1}^{\star}}\\
        &\leq \norm{\left(\mI+\eta\left(\mX^{\star}-\mS_t\mS_t^{\top}\right)\right)\mS_t}+ \eta\norm{\mLambda_t}\norm{\mS_t} + 42  \const\eta\frac{\sigma^{\star 1.5}_{1}\mu r}{\sqrt{pd}}\\
        &\stackrel{(b)}{\leq} \norm{\left(\mI+\eta\left(\mX^{\star}-\mS_t\mS_t^{\top}\right)\right)\mS_t} + 10\eta\sigma_{1}^{\star}\norm{\mE_t}+42  \const\eta\frac{\sigma^{\star 1.5}_{1}\mu r}{\sqrt{pd}}\\
        &\leq \underbrace{\norm{\left(\mI+\eta\left(\mX^{\star}-\mS_t\mS_t^{\top}\right)\right)\mS_t}}_{:=(\rom{1}_1)} +43  \const\eta\frac{\sigma^{\star 1.5}_{1}\mu r}{\sqrt{pd}}.
    \end{aligned}
\end{equation}
Here we apply Lemma~\ref{lem::helper-lemma-delta-t} in $(a)$ and $(b)$. In the last inequality, we use the assumption that $\norm{\mE_t}\leq \frac{\const\sqrt{\sigma_1^{\star}}\mu r}{\sqrt{pd}}$. For $(\rom{1}_1)$, we further decompose it via triangle inequality as follows
\begin{equation}
    \begin{aligned}
        (\rom{1}_1)&\leq \norm{\proj_{\mV_t}\left(\mI+\eta\left(\mX^{\star}-\mS_t\mS_t^{\top}\right)\right)\mS_t}+\norm{\proj_{\mV_t^{\perp}}\left(\mI+\eta\left(\mX^{\star}-\mS_t\mS_t^{\top}\right)\right)\mS_t}\\
        &=\underbrace{\norm{\left(\mI+\eta\left(\mV_t^{\top}\mV^{\star}\mSigma^{\star}\mV^{\star\top}\mV_t-\bar{\mS}_t\bar{\mS}_t^{\top}\right)\right)\bar{\mS}_t}}_{:=(\rom{1}_{1, 1})}+\eta \underbrace{\norm{\proj_{\mV_t^{\perp}}\left(\mV^{\star}-\mV_t\right)\mSigma^{\star}\mV^{\star\top}\mS_t}}_{:=(\rom{1}_{1, 2})}.
    \end{aligned}
\end{equation}
Here we define $\bar{\mS}_t=\mV_t^{\top}\mU_t\in \bR^{r\times d}$. Note that $\norm{\bar{\mS}_t}=\norm{\mS_t}$. For $(\rom{1}_{1, 1})$, we have
\begin{equation}
    \begin{aligned}
        (\rom{1}_{1, 1})&\leq \norm{\left(\mI-\eta\bar{\mS}_t\bar{\mS}_t^{\top}\right)\bar{\mS}_t}+\eta \norm{\mV_t^{\top}\mV^{\star}\mSigma^{\star}\mV^{\star\top}\mV_t\bar{\mS}_t}\\
        &\stackrel{(a)}{=} \norm{\bar{\mS}_t}\left(1-\eta\norm{\bar{\mS}_t}^2\right)+\eta \sigma_{1}^{\star}\norm{\bar{\mS}_t}\\
        &=\norm{\mS_t}\left(1+\eta\sigma_{1}^{\star}-\eta\norm{\mS_t}^2\right).
    \end{aligned}
\end{equation}
where $(a)$ follows from the fact that $\bar \mS_t$ and $\bar \mS_t\bar \mS_t^{\top}$ share the same eigenvectors, and the assumption $\eta\lesssim 1/\sigma_1$.
Next, we control $(\rom{1}_{1, 2})$:
\begin{equation}
    \begin{aligned}
        (\rom{1}_{1, 2})&\leq \sigma^{\star}_{1}\norm{\mV^{\star}-\mV_t}\norm{\mS_t}\leq 2\const_1\frac{\kappa\mu (\sigma_{1}^{\star}r)^{1.5}\log\left(\frac{1}{\alpha}\right)}{\sqrt{pd}}.
    \end{aligned}
\end{equation}
Here we use Lemma~\ref{lem::fro-norm-control}. Therefore, we can bound $(\rom{1}_1)$ by
\begin{equation}
    \begin{aligned}
        (\rom{1}_1)&\leq \norm{\mS_t}\left(1+\eta\sigma_{1}^{\star}-\eta\norm{\mS_t}^2\right)+2\const_1\eta\frac{\kappa\mu (\sigma_{1}^{\star}r)^{1.5}\log\left(\frac{1}{\alpha}\right)}{\sqrt{pd}}.
    \end{aligned}
\end{equation}
This leads to 
\begin{equation}
    \begin{aligned}
        (\rom{1})&\leq \norm{\mS_t}\left(1+\eta\sigma_{1}^{\star}-\eta\norm{\mS_t}^2\right)+2\const_1\eta\frac{\kappa\mu (\sigma_{1}^{\star}r)^{1.5}\log\left(\frac{1}{\alpha}\right)}{\sqrt{pd}}+43  \const\eta\frac{\sigma^{\star 1.5}_{1}\mu r}{\sqrt{pd}}\\
        &\leq \norm{\mS_t}\left(1+\eta\sigma_{1}^{\star}-\eta\norm{\mS_t}^2\right)+3\const_1\eta\frac{\kappa\mu (\sigma_{1}^{\star}r)^{1.5}\log\left(\frac{1}{\alpha}\right)}{\sqrt{pd}}.
    \end{aligned}
\end{equation}
Next, we control $(\rom{2})$. To this end, we first notice that
\begin{equation}
    \begin{aligned}
        \proj_{\mV_{t+1}}\left(\mI+\eta\mM_t\right)\proj_{\mV_{t}}^{\perp}=\left(\proj_{\mV_{t+1}}-\proj_{\mV_{t}}\right)\proj_{\mV_{t}}^{\perp}+\eta \proj_{\mV_{t+1}}\mM_t\proj_{\mV_{t}}^{\perp}.
    \end{aligned}
\end{equation}
Hence, we can bound $(\rom{2})$ by
\begin{equation}
    \begin{aligned}
        (\rom{2})&\leq \norm{\proj_{\mV_{t+1}}-\proj_{\mV_{t}}}\norm{\mE_t}+\eta \norm{\mM_t\proj^{\perp}_{\mV_t}}\norm{\mE_t}\\
        &\stackrel{(a)}{\leq} \left(2\norm{\mV_{t+1}-\mV_t}+2\const_1\eta \sigma_1^{\star}\frac{\kappa\mu r^{1.5}\log\left(\frac{1}{\alpha}\right)}{\sqrt{pd}}\right)\norm{\mE_t}.
    \end{aligned}
\end{equation}
Here we use Lemma~\ref{lem::orthogonal-matrix} in $(a)$. 
It remains to control $\norm{\mV_{t+1}-\mV_t}$. To this end, upon noticing that $\mV_{t+1}-\mV_t=\eta\proj_{\mV_t}^{\perp} \mM_t\mV_t+\mA_t$, one has 
\begin{equation}
    \begin{aligned}
        \norm{\mV_{t+1}-\mV_t}&\leq \eta\norm{\mM_t\proj_{\mV_{t}}^{\perp}}+\norm{\mA_t}\\
        &\stackrel{(a)}{\leq} 2\const_1\eta \sigma_1^{\star}\frac{\kappa\mu r^{1.5}\log\left(\frac{1}{\alpha}\right)}{\sqrt{pd}}+300\eta^2\sigma_1^{\star 2}\\
        &\leq 3\const_1\eta \sigma_1^{\star}\frac{\kappa\mu r^{1.5}\log\left(\frac{1}{\alpha}\right)}{\sqrt{pd}}.
    \end{aligned}
\end{equation}
Here we apply Lemma~\ref{lem::helper-lemma-delta-t} and Lemma~\ref{lem::helper-lemma-A-t} in $(a)$.
Hence, we have 
\begin{equation}
    \begin{aligned}
        (\rom{2})&\leq 5 \const_1\eta \sigma^{\star}_{1}\frac{\kappa\mu r^{1.5}\log\left(\frac{1}{\alpha}\right)}{\sqrt{pd}}\norm{\mE_t}.
    \end{aligned}
\end{equation}
Putting everything together, we obtain that
\begin{equation}
    \begin{aligned}
        \norm{\mS_{t+1}}&\leq \norm{\mS_t}\left(1+\eta\sigma_{1}^{\star}-\eta\norm{\mS_t}^2\right)+3\const_1\eta\sigma_{1}^{\star}\frac{\kappa\mu r^{1.5}\log\left(\frac{1}{\alpha}\right)}{\sqrt{pd}}\left(\sqrt{\sigma_{1}^{\star}}+2\norm{\mE_t}\right)\\
        &\leq \norm{\mS_t}\left(1+\eta\sigma_{1}^{\star}-\eta\norm{\mS_t}^2\right)+6\const_1\eta\sigma_{1}^{\star 1.5}\frac{\kappa\mu r^{1.5}\log\left(\frac{1}{\alpha}\right)}{\sqrt{pd}}.
    \end{aligned}
\end{equation}
Next, we consider two cases separately. First, if $\norm{\mS_t}\leq 1.5\sqrt{\sigma_{1}^{\star}}$, then we simply have
\begin{equation}
    \begin{aligned}
        \norm{\mS_{t+1}}&\leq \norm{\mS_t}\left(1+\eta\sigma_{1}^{\star}-\eta\norm{\mS_t}^2\right)+6\const_1\eta\sigma_{1}^{\star 1.5}\frac{\kappa\mu r^{1.5}\log\left(\frac{1}{\alpha}\right)}{\sqrt{pd}}\\
        &\leq 1.5\sqrt{\sigma_{1}^{\star}}\cdot\left(1+\eta\sigma_{1}^{\star}\right)+6\const_1\eta\sigma_{1}^{\star 1.5}\frac{\kappa\mu r^{1.5}\log\left(\frac{1}{\alpha}\right)}{\sqrt{pd}}\\
        &\leq 2\sqrt{\sigma_{1}^{\star}}.
    \end{aligned}
\end{equation}
On the other hand, if $1.5\sqrt{\sigma_{1}^{\star}}\leq \norm{\mS_t}\leq 2\sqrt{\sigma_{1}^{\star}}$, then we have
\begin{equation}
    \begin{aligned}
        \norm{\mS_{t+1}}&\leq \norm{\mS_t}\left(1+\eta\sigma_{1}^{\star}-\eta\norm{\mS_t}^2\right)+6\const_1\eta\sigma_{1}^{\star 1.5}\frac{\kappa\mu r^{1.5}\log\left(\frac{1}{\alpha}\right)}{\sqrt{pd}}\\
        &\leq 2\sqrt{\sigma_{1}^{\star}}\left(1-1.25\eta\sigma_{1}^{\star}\right)+6\const_1\eta\sigma_{1}^{\star 1.5}\frac{\kappa\mu r^{1.5}\log\left(\frac{1}{\alpha}\right)}{\sqrt{pd}}\\
        &\leq 2\sqrt{\sigma_{1}^{\star}}.
    \end{aligned}
\end{equation}
This completes the proof for the maximal signal dynamic.

\paragraph{Minimal signal dynamic.} We first provide a lower-bound for $\sigma_{r}(\mS_{t+1})$ as follows
\begin{equation}
    \begin{aligned}
        \sigma_{r}(\mS_{t+1})\geq \underbrace{\sigma_{r}\!\left(\left(\mI+\eta\mM_t\right)\mS_t\right)}_{(\rom{1})}-\underbrace{\norm{\proj_{\mV_{t+1}}\left(\mI+\eta\mM_t\right)\mE_t}}_{(\rom{2})}.
    \end{aligned}
\end{equation}
For $(\rom{1})$, applying Lemma~\ref{lem::lower-bound-sigma-min}, we first obtain that
$(\rom{1})\geq \sigma_{r}\!\left(\mV_t^{\top}\left(\mI+\eta\mM_t\right)\mS_t\right)$.
Next, we decompose $\mV_t^{\top}\left(\mI+\eta\mM_t\right)\mS_t$ as follows
\begin{equation}
    \begin{aligned}
        &\mV_t^{\top}\left(\mI+\eta\mM_t\right)\mS_t\\
        &=\underbrace{\left(\mI + \eta \mV_t^{\top}\left(\mM_t+\mS_t\mS_t^{\top}\right)\mV_t\left(\mI-\eta \mV_t^{\top}\mS_t\mS_t^{\top}\mV_t\right)^{-1}\right)}_{:=\mB_t}\underbrace{\mV_t^{\top}\mS_t\left(\mI-\eta \mS_t^{\top}\mS_t\right)}_{:=\mC_t}.
    \end{aligned}
\end{equation}
According to Lemma~\ref{lem::lower-bound-sigma-min-2}, we have
\begin{equation}
    \begin{aligned}
        \sigma_{r}\!\left(\left(\mI+\eta\mM_t\right)\mS_t\right)\geq \sigma_{r}(\mB_t)\sigma_{r}(\mC_t)= \sigma_{r}(\mB_t)\left(1-\eta \sigma_{r}^2(\mS_t)\right)\sigma_{r}(\mS_t),
    \end{aligned}
\end{equation}
where in the last equality, we use the fact that $\mS_t$ and $\mS_t\mS_t^{\top}\mS_t$ share the same singular space and hence
\begin{equation}
    \sigma_{r}(\mC_t)=\sigma_{r}(\mS_t(\mI-\eta \mS_t^{\top}\mS_t))=\left(1-\eta \sigma_{r}^2(\mS_t)\right)\sigma_{r}(\mS_t).
\end{equation}
Now, it suffices to provide a lower-bound for $\sigma_{r}(\mB_t)$. To this end, we first notice that
\begin{equation}
    \begin{aligned}
        \sigma_{r}(\mB_t)&\geq 1 +\eta \sigma_{r}\!\left(\mV_t^{\top}\left(\mM_t+\mS_t\mS_t^{\top}\right)\mV_t\right)\sigma_{r}\!\left(\left(\mI-\eta \mV_t^{\top}\mS_t\mS_t^{\top}\mV_t\right)^{-1}\right)\\
        &\geq 1 +\eta \sigma_{r}\!\left(\mV_t^{\top}\left(\mM_t+\mS_t\mS_t^{\top}\right)\mV_t\right).
    \end{aligned}
\end{equation}
Here the second inequality is due to $\mI-\eta \mV_t^{\top}\mS_t\mS_t^{\top}\mV_t\preceq \mI$.
To proceed, notice that 
\begin{equation}
    \mM_t+\mS_t\mS_t^{\top}=\mX^{\star}+\left(\projection-\id\right)\left(\mDelta_t\right)-\mLambda_t.
\end{equation}
Therefore, we have 
\begin{equation}
    \begin{aligned}
        \sigma_{r}\!\left(\mV_t^{\top}\left(\mM_t+\mS_t\mS_t^{\top}\right)\mV_t\right)&\geq \sigma_{r}\!\left(\mV_t^{\top}\mX^{\star}\mV_t\right)-\norm{\left(\projection-\id\right)\left(\mDelta_t\right)}-\norm{\mLambda_t}\\
        &\geq \sigma_{r}\!\left(\mV_t^{\top}\mX^{\star}\mV_t\right)-21\const\sigma^{\star}_{1}\sqrt{\frac{\mu^2 r^2}{pd}}-5\sqrt{\sigma_1^{\star}}\norm{\mE_t}\\
        &\geq \sigma_{r}\!\left(\mV_t^{\top}\mX^{\star}\mV_t\right)-22\const\sigma^{\star}_{1}\sqrt{\frac{\mu^2 r^2}{pd}}.
    \end{aligned}
\end{equation}
For the first term in the above inequality, we have
\begin{equation}
    \begin{aligned}
        \sigma_{r}\!\left(\mV_t^{\top}\mX^{\star}\mV_t\right)&\geq \sigma_{r}\!\left(\mSigma^{\star}\mV^{\star\top}\mV_t\right)-\sigma^{\star}_{1}\norm{\mV^{\star}-\mV_t}\\
        &\geq \sigma_{r}^{\star} - 2\sigma^{\star}_{1}\norm{\mV^{\star}-\mV_t}\\
        &\geq \sigma_{r}^{\star}-2\const_1\sigma_1^{\star}\frac{\kappa\mu r^{1.5}\log\left(\frac{1}{\alpha}\right)}{\sqrt{pd}}\\
        &\geq 0.95\sigma_r^{\star}.
    \end{aligned}
\end{equation}
Therefore, we obtain
\begin{equation}
    \begin{aligned}
        \sigma_{r}\!\left(\mV_t^{\top}\left(\mM_t+\mS_t\mS_t^{\top}\right)\mV_t\right)&\geq 0.95\sigma_r^{\star}-22\const\sigma^{\star}_{1}\sqrt{\frac{\mu^2 r^2}{pd}}\geq 0.9\sigma_{r}^{\star}.
    \end{aligned}
\end{equation}
Combining the above arguments, we have 
\begin{equation}
    \begin{aligned}
        (\rom{1})\geq (1+0.9\eta\sigma_{r}^{\star})\left(1-\eta \sigma_{r}^2(\mS_t)\right)\sigma_{r}(\mS_t)\geq \left(1+0.8\eta\sigma^{\star}_{r}-\eta\sigma_{r}^2(\mS_t)\right)\sigma_{r}(\mS_t).
    \end{aligned}
\end{equation}
On the other hand, we have already derived an upper bound for $(\rom{2})$ in the maximal signal dynamic, which is
\begin{equation}
    \begin{aligned}
        (\rom{2})\leq 5 \const_1\eta \sigma^{\star}_{1}\frac{\kappa\mu r^{1.5}\log\left(\frac{1}{\alpha}\right)}{\sqrt{pd}}\norm{\mE_t}.
    \end{aligned}
\end{equation}
Putting everything together, we have
\begin{equation}
    \begin{aligned}
        \sigma_{r}(\mS_{t+1})\geq \left(1+0.8\eta\sigma^{\star}_{r}-\eta\sigma_{r}^2(\mS_t)\right)\sigma_{r}(\mS_t)-5\const_1\eta \frac{\sigma^{\star}_{1}\kappa\mu r^{1.5}\log\left(\frac{1}{\alpha}\right)}{\sqrt{pd}}\norm{\mE_t}.
    \end{aligned}
\end{equation}
\subsection{Proof of Residual Dynamic} First, we can expand $\mE_{t+1}$ as below
\begin{equation}
    \begin{aligned}
        \mE_{t+1}&=\proj^{\perp}_{\mV_{t+1}}\left(\mI+\eta\mM_t\right)\left(\mS_t+\mE_t\right)=\proj^{\perp}_{\mV_{t+1}}\left(\mI+\eta\mM_t\right)\mE_t,
    \end{aligned}
\end{equation}
where in the second equality, we use the fact that $\proj^{\perp}_{\mV_{t+1}}\left(\mI+\eta\mM_t\right)\mS_{t}=0$.
Then, by triangle inequality, we obtain 
\begin{equation}
    \begin{aligned}
        \norm{\mE_{t+1}}&\leq \left(1+\eta \norm{\mM_t\proj_{\mV_t}^{\perp}}\right)\norm{\mE_t}\leq \left(1+ 2\const_1\eta\frac{\sigma^{\star}_{1}\kappa\mu r^{1.5}\log\left(\frac{1}{\alpha}\right)}{\sqrt{pd}}\right)\norm{\mE_t}.
    \end{aligned}
\end{equation}
Here in the last inequality, we use Lemma~\ref{lem::helper-lemma-delta-t}. 

\subsection{Proof of Error Dynamic} 
The core proof idea is adapted from the proof of Proposition~4.3 appeared in \citep{li2018algorithmic}.
We first expand $\norm{\mDelta_{t+1}}_{\fro}^2$ as 
    \begin{equation}
        \begin{aligned}
            \norm{\mDelta_{t+1}}_{\fro}^2&=\norm{\mX^{\star}-\left(\mI+\eta\mM_t\right)\mU_t\mU_t^{\top}\left(\mI+\eta\mM_t\right)}^2_{\fro}\\
            &=\norm{\mDelta_{t}}_{\fro}^2-4\eta \underbrace{\inner{\mX^{\star}-\mU_{t}\mU_{t}^{\top}}{\mM_t\mU_t\mU_t^{\top}}}_{(\rom{1})}+(\rom{2}),
        \end{aligned}
    \end{equation}
    where 
    \begin{equation}
        \begin{aligned}
            (\rom{2})&=2\inner{-\mDelta_t+\eta \mM_t\mU_t\mU_t^{\top}}{\eta^2\mM_t\mU_t\mU_t^{\top}\mM_t}\\
            &\quad+\eta^2\norm{\mM_t\mU_t\mU_t^{\top}+\mU_t\mU_t^{\top}\mM_t}_{\fro}^2+\eta^4\norm{\mM_t\mU_t\mU_t^{\top}\mM_t}_{\fro}^2
        \end{aligned}
    \end{equation}
    contains all the higher-order terms.
    We then provide a lower-bound for $(\rom{1})$. To this goal, we notice that
    \begin{equation}
        \begin{aligned}
            (\rom{1})&=\norm{\mDelta_t\mU_t}_{\fro}^2-\inner{\mDelta_t}{\left(\mI-\projection\right)\left(\mDelta_t\right)\mU_{t}\mU_{t}^{\top}}\\
            &\geq \norm{\mDelta_t\mU_t}_{\fro}^2 - \norm{\mDelta_t}_{\fro}\norm{\mU_t\mU_t^{\top}}_{\fro}\norm{\left(\id-\projection\right)\left(\mDelta_t\right)}\\
            &\geq \norm{\mDelta_t\mU_t}_{\fro}^2-\norm{\mDelta_t}_{\fro}\cdot 8\sqrt{r}\sigma^{\star}_{1}\cdot \left(5  \const\sqrt{\frac{\mu^2 r^2}{pd}}\norm{\mDelta_t}+10  \const\sqrt{\frac{\sigma^{\star}_{1}\mu r}{p}}\norm{\mE_t}\right)\\
            &= \norm{\mDelta_t\mU_t}_{\fro}^2-40\const\sigma^{\star}_{1}\sqrt{\frac{\mu^2 r^3}{pd}}\norm{\mDelta_t}_{\fro}^2-80\const\sqrt{\frac{\sigma^{\star 3}_{1}\mu r^2}{p}}\norm{\mE_t}\norm{\mDelta_t}_{\fro}.
        \end{aligned}
    \end{equation}
    Here in the last inequality, we apply Lemma~\ref{lem::helper-lemma-delta-t}.
    Next, we provide a lower-bound for $\norm{\mDelta_t\mU_t}_{\fro}^2$. To this goal, we first notice that 
    \begin{equation}
        \begin{aligned}
            \norm{\mDelta_t\mU_t}_{\fro}&=\norm{\left(\mX^{\star}-\mS_{t}\mS_{t}^{\top}\right)\mS_t-\mLambda_t\mS_t+\mDelta_t\mE_t}_{\fro}\\
            &\geq \norm{\left(\mX^{\star}-\mS_{t}\mS_{t}^{\top}\right)\mS_t}_{\fro}-\norm{\mLambda_t}\norm{\mS_t}_{\fro}-\norm{\mDelta_t}_{\fro}\norm{\mE_t}\\
            &\geq \norm{\left(\mX^{\star}-\mS_{t}\mS_{t}^{\top}\right)\mS_t}_{\fro} - 10\sqrt{r}\sigma^{\star}_{1}\norm{\mE_t}-\norm{\mDelta_t}_{\fro}\norm{\mE_t}.
        \end{aligned}
        \label{eq::lower-bound-2}
    \end{equation}
    In the last inequality we use the fact that $\norm{\mLambda_t}\leq 5\sqrt{\sigma^{\star}_{1}}\norm{\mE_t}$ and $\norm{\mS_t}_{\fro}\leq \sqrt{r}\norm{\mS_t}\leq 2\sqrt{r\sigma_{1}^{\star}}$ from Lemma~\ref{lem::helper-lemma-lambda-t}. Applying Lemma~\ref{lem::norm-submatrix-2}, we can further provide a lower-bound for $\norm{\left(\mX^{\star}-\mS_{t}\mS_{t}^{\top}\right)\mS_t}_{\fro}$ as 
    \begin{equation}
        \begin{aligned}
            \norm{\left(\mX^{\star}-\mS_{t}\mS_{t}^{\top}\right)\mS_t}_{\fro}&\geq \sigma_{r}(\mS_t)\norm{\left(\mX^{\star}-\mS_{t}\mS_{t}^{\top}\right)\mV_t}_{\fro}\geq \frac{\sqrt{\sigma^{\star}_{r}}}{2}\norm{\left(\mX^{\star}-\mS_{t}\mS_{t}^{\top}\right)\mV_t}_{\fro}.
        \end{aligned}
        \label{eq::lower-bound-3}
    \end{equation}
    Next, we present the following intermediate lemma to control $\norm{\left(\mX^{\star}-\mS_{t}\mS_{t}^{\top}\right)\mV_t}_{\fro}$.
    \begin{lemma}
        \label{lem::lower-bound-1}
        Suppose that $\norm{\mV^{\star}-\mV_t} \leq 0.1$. Then, we have 
        \begin{equation}
            \norm{\left(\mX^{\star}-\mS_{t}\mS_{t}^{\top}\right)\mV_t}_{\fro}^2\geq \frac{2}{5}\norm{\mX^{\star}-\mS_{t}\mS_{t}^{\top}}_{\fro}^2.
            \label{eq::lower-bound-1}
        \end{equation}
    \end{lemma} 
    We first use this lemma to finish the proof of the loss dynamic and defer the proof to the end of this section. Applying this lemma to \Cref{eq::lower-bound-3} yields
    \begin{equation}
        \begin{aligned}
            \norm{\left(\mX^{\star}-\mS_{t}\mS_{t}^{\top}\right)\mS_t}_{\fro}&\geq \sqrt{\frac{\sigma^{\star}_{r}}{10}}\norm{\mX^{\star}-\mS_{t}\mS_{t}^{\top}}_{\fro}\\
            &\geq \sqrt{\frac{\sigma^{\star}_{r}}{10}}\norm{\mDelta_t}_{\fro}-\sqrt{\frac{\sigma^{\star}_{r}}{10}}\norm{\mLambda_t}_{\fro}\\
            &\geq \sqrt{\frac{\sigma^{\star}_{r}}{10}}\norm{\mDelta_t}_{\fro}-\sigma_{1}^{\star}\sqrt{\frac{5}{2\kappa}}\norm{\mE_t}.
        \end{aligned}
        \label{eq::lower-bound-4}
    \end{equation}
    Combining \Cref{eq::lower-bound-4} and \Cref{eq::lower-bound-2} leads to
    \begin{equation}
        \begin{aligned}
            \norm{\mDelta_t\mU_t}_{\fro}&\geq \sqrt{\frac{\sigma^{\star}_{r}}{10}}\norm{\mDelta_t}_{\fro}-\sigma_{1}^{\star}\sqrt{\frac{5}{2\kappa}}\norm{\mE_t}- 10\sqrt{r}\sigma^{\star}_{1}\norm{\mE_t}-\norm{\mDelta_t}_{\fro}\norm{\mE_t}\\
            &\geq \sqrt{\frac{\sigma^{\star}_{r}}{10}}\norm{\mDelta_t}_{\fro} - 20\sqrt{r}\sigma_{1}^{\star}\norm{\mE_t}.
        \end{aligned}
    \end{equation}
    This implies that
    \begin{equation}
        \begin{aligned}
            \norm{\mDelta_t\mU_t}_{\fro}^2&\geq \frac{\sigma_{r}^{\star}}{10}\norm{\mDelta_t}_{\fro}^2 - 13\sqrt{\frac{r}{\kappa}}\sigma_{1}^{\star 1.5}\norm{\mDelta_t}_{\fro}\norm{\mE_t}.
        \end{aligned}
    \end{equation}
    Overall, we obtain 
    \begin{equation}
        \begin{aligned}
            (\rom{1})&\geq \frac{\sigma_{r}^{\star}}{10}\norm{\mDelta_t}_{\fro}^2 \!-\! 13\sqrt{\frac{r}{\kappa}}\sigma_{1}^{\star 1.5}\norm{\mDelta_t}_{\fro}\norm{\mE_t}\!-\!40\const\sigma^{\star}_{1}\sqrt{\frac{\mu^2 r^3}{pd}}\norm{\mDelta_t}_{\fro}^2\!-\!80C\!\sqrt{\frac{\sigma^{\star 3}_{1}\mu r^2}{p}}\!\norm{\mE_t}\norm{\mDelta_t}_{\fro}\\
            &\geq \frac{\sigma_{r}^{\star}}{15}\norm{\mDelta_t}_{\fro}^2\!-\!81 \const\sqrt{\frac{\sigma^{\star 3}_{1}\mu r^2}{p}}\norm{\mE_t}\norm{\mDelta_t}_{\fro}.
        \end{aligned}
    \end{equation}
    Next, we control $(\rom{2})$. To this end, we first notice that
    \begin{equation}
        \begin{aligned}
            -2\inner{\mDelta_t}{\mM_t\mU_t\mU_t^{\top}\mM_t}&\leq 2\norm{\mDelta_t}_{\fro}\norm{\mM_t}^2\norm{\mU_t\mU_t^{\top}}_{\fro}\\
            &\leq 16\sqrt{r}\sigma^{\star}_{1}\norm{\mDelta_t}_{\fro}\left(2\norm{\mDelta_t}+10\const\sqrt{\frac{\sigma_1^{\star}\mu r}{p}}\norm{\mE_t}\right)^2\\
            &\leq 128\sqrt{r}\sigma^{\star}_{1}\norm{\mDelta_t}_{\fro}\left(\norm{\mDelta_t}^2_{\fro}+25\const^2\frac{\sigma_1^{\star}\mu r}{p}\norm{\mE_t}^2\right).
        \end{aligned}
    \end{equation}
    Similarly, we have
    \begin{equation}
        \begin{aligned}
            \inner{\mM_t\mU_t\mU_t^{\top}}{\mM_t\mU_t\mU_t^{\top}\mM_t}&\leq \norm{\mM_t}^3\norm{\mU_t\mU_t^{\top}}_{\fro}^2\\
            &\leq C_1 r\sigma^{\star 2}_{1}\left(\norm{\mDelta_t}^3_{\fro}+\frac{(\sigma_1^{\star}\mu r)^{1.5}}{p^{1.5}}\norm{\mE_t}^3\right),
        \end{aligned}
    \end{equation}
    \begin{equation}
        \begin{aligned}
            \norm{\mM_t\mU_t\mU_t^{\top}+\mU_t\mU_t^{\top}\mM_t}_{\fro}^2&\leq 4\norm{\mM_t}^2\norm{\mU_t\mU_t^{\top}}_{\fro}^2\\
            &\leq C_2 r\sigma^{\star 2}_{1}\left(\norm{\mDelta_t}^2_{\fro}+\frac{\sigma_1^{\star}\mu r}{p}\norm{\mE_t}^2\right).
        \end{aligned}
    \end{equation}
    \begin{equation}
        \begin{aligned}
            \norm{\mM_t\mU_t\mU_t^{\top}\mM_t}_{\fro}^2&\leq C_3 r\sigma^{\star 2}_{1}\left(\norm{\mDelta_t}^4_{\fro}+\frac{\sigma_1^{\star 2}\mu^2 r^2}{p^2}\norm{\mE_t}^4\right).
        \end{aligned}
    \end{equation}
    These inequalities lead to
    \begin{equation}
        (\rom{2})\leq C_4\eta^2r\sigma^{\star 2}_{1}\left(\norm{\mDelta_t}^2_{\fro}+\const^2\frac{\sigma_1^{\star}\mu r}{p}\norm{\mE_t}^2\right).
    \end{equation}
    Overall, we have 
    \begin{equation}
        \begin{aligned}
            \norm{\mDelta_{t+1}}_{\fro}^2&\leq \left(1-\frac{1}{5}\eta\sigma_{r}^{\star}\right)\norm{\mDelta_t}_{\fro}^2+324 \const\eta\sqrt{\frac{\sigma^{\star 3}_{1}\mu r^2}{p}}\norm{\mE_t}\norm{\mDelta_t}_{\fro}+C_5\eta^2r\frac{\sigma_1^{\star 3}\mu r}{p}\norm{\mE_t}^2\\
            &\leq\left(\left(1-\frac{1}{10}\eta\sigma_r^{\star}\right)\norm{\mDelta_t}_{\fro}+C_6\eta\sqrt{\frac{\sigma^{\star 3}_{1}\mu r^2}{p}}\norm{\mE_t} \right)^2,
        \end{aligned}
    \end{equation}
    which implies that
    \begin{equation}
        \norm{\mDelta_{t+1}}_{\fro}\leq \left(1-\frac{1}{10}\eta\sigma_{r}^{\star}\right)\norm{\mDelta_t}_{\fro}+C_6\eta\sqrt{\frac{\sigma^{\star 3}_{1}\mu r^2}{p}}\norm{\mE_t}.
    \end{equation}
    Lastly, we provide the proof of Lemma~\ref{lem::lower-bound-1}.

    \begin{proof}\linkofproof{Lemma~\ref{lem::lower-bound-1}}
        First, we define $\mP = \mV^{\star\top}\mV_t$ and note that $\mS_{t}\mS_{t}^{\top} = \mV_t\mSigma_t\mV_t^{\top}$. This allows us to write
    \begin{equation}
        \begin{aligned}
            \norm{\left(\mX^{\star}-\mS_{t}\mS_{t}^{\top}\right)\mV_t}_{\fro}^2&=\norm{\mSigma_t}_{\fro}^2+\norm{\mSigma^{\star}\mP}_{\fro}^2-2\inner{\mSigma_t}{\mP^{\top}\mSigma^{\star}\mP},\\
            \norm{\mX^{\star}-\mS_{t}\mS_{t}^{\top}}_{\fro}^2&=\norm{\mSigma_t}_{\fro}^2+\norm{\mSigma^{\star}}_{\fro}^2-2\inner{\mSigma_t}{\mP^{\top}\mSigma^{\star}\mP}.
        \end{aligned}
    \end{equation}
    Substituting the above equivalent forms into \Cref{eq::lower-bound-1}, we need to show that
    \begin{equation}
        3\norm{\mSigma_t}_{\fro}^2+5\norm{\mSigma^{\star}\mP}_{\fro}^2\geq 2\norm{\mSigma^{\star}}_{\fro}^2+6\inner{\mSigma_t}{\mP^{\top}\mSigma^{\star}\mP}.
    \end{equation}
    To this end, we first apply the Cauchy-Schwartz inequality, which gives us $2\inner{\mSigma_t}{\mP^{\top}\mSigma^{\star}\mP}\leq \norm{\mSigma_t}_{\fro}^2+\norm{\mP^{\top}\mSigma^{\star}\mP}_{\fro}^2$. Therefore, it suffices to show that
    \begin{equation}
        5\norm{\mSigma^{\star}\mP}_{\fro}^2-2\norm{\mSigma^{\star}}_{\fro}^2-3\norm{\mP^{\top}\mSigma^{\star}\mP}_{\fro}^2\geq 0.
    \end{equation}
    This follows from
    \begin{equation}
        \begin{aligned}
            5\norm{\mSigma^{\star}\mP}_{\fro}^2-2\norm{\mSigma^{\star}}_{\fro}^2-3\norm{\mP^{\top}\mSigma^{\star}\mP}_{\fro}^2=\tr\left(\left(\mSigma^{\star}\left(\mI-\mP\mP^{\top}\right)\mSigma^{\star}\right)\cdot\left(3\mP\mP^{\top}-2\mI\right)\right)\geq 0.
        \end{aligned}
    \end{equation}
    Here we use the facts that $\mSigma^{\star}\left(\mI-\mP\mP^{\top}\right)\mSigma^{\star}\succeq 0$ since $\norm{\mP}\leq \norm{\mV^{\star}}\norm{\mV_t}\leq 1$, and $3\mP\mP^{\top}-2\mI\succeq 0$ since $\sigma_{r}(\mP)\geq 1-\norm{\mV^{\star}-\mV_t}\geq 0.9$. This completes the proof. 
\end{proof}
    
\section{Proofs for Main Theorems}\label{sec_main_proofs}
In this section, we use the one-step dynamics in Proposition~\ref{prop_onestep} to prove our main theorems under the conditions that $\norm{\mV_t}_{2, \infty}\leq 2\sqrt{\frac{\mu r}{d}}$ and $\norm{\mV^{\star}-\mV_{t}}_{\fro}\leq  \const_1\frac{\kappa\mu r^{1.5}\log\left(\frac{1}{\alpha}\right)}{\sqrt{pd}}$ for all $0\leq t\leq T$. These two conditions will be established later in Appendix~\ref{sec_incoherence}.

\subsection{Proof of Theorem~\ref{thm::main}}
The proof is divided into three distinct steps.
\paragraph{Step 1.} In the first step, we show that $\norm{\mS_t}\leq 2\sqrt{\sigma_{1}^{\star}}$ and $\norm{\mE_t}\leq 2\alpha$ hold for all $0\leq t\leq T$. 

We prove this by induction. First, in the base case where $t=0$, these two conditions are naturally met because $\norm{\mS_0}\leq \norm{\mU_0}\leq \alpha\leq 2\sqrt{\sigma_{1}^{\star}}$ and $\norm{\mE_0}\leq 2\norm{\mE_0}\leq 2\alpha$. Next, for the induction step, we assume that $\norm{\mS_t} \leq 2\sqrt{\sigma_{1}^{\star}}$ and $\norm{\mE_t} \leq 2\alpha$ hold for all $0 \leq s \leq t$, with $t \leq T-1$. Utilizing Proposition~\ref{prop_onestep}, we can directly derive that $\norm{\mS_{t+1}} \leq 2\sqrt{\sigma_{1}^{\star}}$. Regarding $\norm{\mE_{t+1}}$, we have
\begin{equation}
    \begin{aligned}
        \norm{\mE_{t+1}}&\!\leq\! \biggl(1+ 2\const_1\eta\frac{\sigma^{\star}_{1}\kappa\mu r^{1.5}\log\left(\frac{1}{\alpha}\right)}{\sqrt{pd}}\biggr)^{t+1}\alpha\!\stackrel{(a)}{\leq}\! \biggl(1+ 4\const_1\eta\frac{\sigma^{\star}_{1}\kappa\mu r^{1.5}\log\left(\frac{1}{\alpha}\right)}{\sqrt{pd}}\cdot (t+1)\biggr)\alpha\leq 2\alpha.
    \end{aligned}
\end{equation}
Here in $(a)$, we apply Lemma~\ref{lem::exponential-linear}. This is valid since $4\const_1\eta\frac{\sigma^{\star}_{1}\kappa\mu r^{1.5}\log\left(\frac{1}{\alpha}\right)}{\sqrt{pd}}\cdot (t+1)\leq 1$ for any $t\leq T-1\lesssim \frac{1}{\eta\sigma^{\star}_{r}}\log\left(\frac{1}{\alpha}\right)$ provided that the sampling rate satisfies $p\gtrsim \frac{\kappa^4\mu^2r^3\log^4\left(\frac{1}{\alpha}\right)}{d}$. This completes the induction step.

\paragraph{Step 2.} This step demonstrates that the minimal signal $\sigma_{r}(\mS_t)$ grows linearly to $\frac{\sqrt{\sigma_{r}^{\star}}}{2}$.

Given that we have already established $\norm{\mS_t} \leq 2\sqrt{\sigma_{1}^{\star}}$ and $\norm{\mE_t} \leq 2\alpha$ for all $0 \leq t \leq T$, we can simplify the minimal signal dynamic in Proposition~\ref{prop_onestep} as
\begin{equation}
    \begin{aligned}
        \sigma_{r}(\mS_{t+1})&\geq \left(1+0.8\eta\sigma^{\star}_{r}-\eta\sigma_{r}^2(\mS_t)\right)\sigma_{r}(\mS_t)-12\const_1\eta \frac{\sigma^{\star}_{1}\kappa\mu r^{1.5}\log\left(\frac{1}{\alpha}\right)}{\sqrt{pd}}\alpha\\
        &\geq \left(1+0.4\eta\sigma^{\star}_{r}\right)\sigma_{r}(\mS_t)-12\const_1\eta \frac{\sigma^{\star}_{1}\kappa\mu r^{1.5}\log\left(\frac{1}{\alpha}\right)}{\sqrt{pd}}\alpha.
    \end{aligned}
\end{equation} 
This holds for any $t$ that satisfies $\sigma_{r}(\mS_t)\leq \frac{3\sqrt{\sigma_{r}^{\star}}}{4}$. By applying Lemma~\ref{lem::appendix-series-ineq}, we obtain
\begin{equation}
    \sigma_{r}(\mS_{t})\geq \left(1+0.4\eta\sigma_{r}^{\star}\right)^t\left(\sigma_{r}(\mS_{t})-30\const_1\frac{\kappa^2\mu r^{1.5}\log\left(\frac{1}{\alpha}\right)}{\sqrt{pd}}\alpha\right).
\end{equation}
At initialization, it is observed that 
\begin{equation}
    \sigma_{r}(\mS_{0})-30\const_1\frac{\kappa^2\mu r^{1.5}\log\left(\frac{1}{\alpha}\right)}{\sqrt{pd}}\alpha\geq c_0\alpha - 30\const_1\frac{\kappa^2\mu r^{1.5}\log\left(\frac{1}{\alpha}\right)}{\sqrt{pd}}\alpha\geq \frac{c_0}{2}\alpha,
\end{equation}
provided that $p\gtrsim \frac{\kappa^4\mu^2r^3\log^2\left(\frac{1}{\alpha}\right)}{d}$. Consequently, within $T_1 \lesssim \frac{1}{\eta\sigma_{r}^{\star}}\log\left(\frac{\sigma_r^{\star}}{\alpha}\right)$ iterations, $\sigma_{r}(\mS_t)$ reaches $\frac{\sqrt{\sigma_{r}^{\star}}}{2}$. It is also easy to show that $\sigma_{r}(\mS_t) \geq \frac{\sqrt{\sigma_{r}^{\star}}}{2}$ holds true for all $t \geq T_1$.

\paragraph{Step 3.} This step is dedicated to demonstrating that the error $\norm{\mX^{\star}-\mU_{t}\mU_{t}^{\top}}_{\fro}$ converges linearly to $O(\alpha)$ once $\sigma_{r}(\mS_t) \geq \frac{\sqrt{\sigma_{r}^{\star}}}{2}$.

Based on our arguments in \textbf{Step 2}, where we established that $\sigma_{r}(\mS_t) \geq \frac{\sqrt{\sigma_{r}^{\star}}}{2}$ for $t \geq T_1$, and leveraging Proposition~\ref{prop_onestep} along with Lemma~\ref{lem::appendix-series-ineq}, we can derive that
\begin{equation}
    \norm{\mX^{\star}-\mU_{t}\mU_{t}^{\top}}_{\fro}\leq \left(1-\frac{1}{5}\eta\sigma_{r}^{\star}\right)^{t-T_1}\norm{\mX^{\star}-\mU_{T_1}\mU_{T_1}^{\top}}_{\fro}+3240 \const\sqrt{\frac{\sigma^{\star}_{1}\kappa^2\mu r^2}{p}}\alpha.
\end{equation}
Note that $\norm{\mX^{\star}-\mU_{T_1}\mU_{T_1}^{\top}}_{\fro}\leq \norm{\mX^{\star}}_{\fro}+\norm{\mU_{T_1}\mU_{T_1}^{\top}}_{\fro}\leq 9\sqrt{r}\sigma_1^{\star}$ according to Lemma~\ref{lem::helper-lemma-U-t}.
Hence, within an additional $T_2 = \frac{1}{\eta\sigma_r^{\star}}\log\left(\frac{r\sigma_1^{\star}}{\alpha}\right)$ iterations, the error converges to $O\Big(\sqrt{\frac{\sigma^{\star}_{1}\kappa^2\mu r^2}{p}}\alpha\Big)$, thus concluding the proof of Theorem~\ref{thm::main}.

\subsection{Proof of Theorem~\ref{thm::main-exact-param}}
To prove this result, we first apply Theorem~\ref{thm::main} to output a solution $\mU_{t_0}$ and its leave-one-out versions $\mU_{t_0}^{(l)}$ that meet the initialization conditions in Theorem~\ref{thm::local-linear-convergence-exact-param}. Then, we apply Theorem~\ref{thm::local-linear-convergence-exact-param} to obtain the desired result. 

\paragraph{Establishing Condition~(\ref{eq_U0}).} 
By choosing the initialization scale $\alpha = c\cdot\frac{\sigma_r^{\star}}{\kappa^{1.5} d}$ for sufficiently small $c>0$ and assuming the sampling rate of $p\gtrsim \frac{\kappa^2\mu^4r^9\log^4(d)}{d}$, Theorem~\ref{thm::main} guarantees that, with probability at least $1-\frac{1}{d^2}$, the iterations of GD with step-size $\eta\asymp \frac{\mu r}{\sqrt{pd}\sigma^{\star}_{1}}$ satisfy
    \begin{equation}
        \norm{\mU_{t_0}\mU_{t_0}^{\top}-\mX^{\star}}_{\fro}\leq 0.9\Gamma_4\sqrt{\frac{\sigma_1^\star\kappa^2\mu r^2}{p}}\frac{\sigma_r^{\star}}{\kappa^{1.5} d}, \quad \text{for some\quad $t_0\lesssim\frac{1}{\eta\sigma_r^{\star}}\log\left(\frac{\kappa^{1.5} d}{\sigma_r^\star}\right)$.}
    \end{equation}
    On the other hand, Lemma~\ref{lem::dist-fro-tu2016} in the appendix implies that
\begin{equation}\nonumber
        \dist\left(\mU_{t_0}, \mU^{\star}\right)\leq \frac{1.1}{\sigma_r^{\star}}\norm{\mU_{t_0}\mU_{t_0}^{\top}-\mX^{\star}}_{\fro}\leq \const_4\sqrt{\frac{\sigma_r^{\star}\mu^3 r^3\log(d)}{pd^2}},
    \end{equation}
    which establishes Condition~(\ref{eq_U0}). 
    
    \paragraph{Establishing Condition~(\ref{eq_U0l}).} The proof of Condition~(\ref{eq_U0l}) follows a similar logic, recognizing that the leave-one-out sequences exhibit a stronger concentration than the original iterations. Consequently, they fulfill Condition~(\ref{eq_U0l}) within at most $t_0\lesssim\frac{1}{\eta\sigma_r^{\star}}\log\left(\frac{\kappa^{1.5} d}{\sigma_r^\star}\right)$ iterations. Further details of this argument are omitted for brevity.

    This shows that the initial conditions of Theorem~\ref{thm::local-linear-convergence-exact-param} are satisfied after $t_0$ iterations. From this iteration onward, Theorem~\ref{thm::local-linear-convergence-exact-param} shows that the iterations of GD enter a local linear convergence regime, which readily establishes the final result of Theorem~\ref{thm::main-exact-param}.

\section{Proofs for Incoherence Dynamic}\label{sec_incoherence}
In this section, we present our proofs for establishing the incoherence of $\mV_t$. To simplify the presentation, we omit the ``$\sim$'' from our notations. Therefore, $\mV_t^{(l)}, \mZ_t^{(l)}, \mM_t^{(l)}, \ldots$ in this section refer to $\widetilde\mV_t^{(l)}, \widetilde\mZ_t^{(l)}, \widetilde\mM_t^{(l)}, \ldots$ defined in Section~\ref{sec::incoherence-dynamic}. 
    \subsection{Proof of Lemma~\ref{lem::fro-norm-control}}
    \label{sec::proof-frobenius-norm-control}
    We first state a finer variant of this lemma here.
    \begin{proposition}[Controlling $\norm{\mV_t-\mV^\star}_{\fro}$]
        \label{prop::frobenius-norm-control-appendix}
        Suppose that the stepsize satisfies $\eta\asymp\frac{  \mu r}{\sqrt{pd}\sigma^{\star}_{1}}$. Moreover, suppose that $\norm{\mS_t}\leq 2\sqrt{\sigma^{\star}_{1}}$, $\norm{\mE_t}\leq \sqrt{\frac{\sigma_1^{\star}}{d}}$, $\norm{\mV_t}_{2, \infty}\leq 2\sqrt{\frac{\mu r}{d}}$ and $\norm{\mV^{\star}-\mV_{t}}_{\fro}\leq  \const_1\frac{\kappa\mu r^{1.5}\log\left(\frac{1}{\alpha}\right)}{\sqrt{pd}}$ for all $t\leq T\lesssim\frac{1}{\eta\sigma_{r}^{\star}}\log\left(\frac{1}{\alpha}\right)$. Then, conditioned on $\event$, we have
        \begin{equation}
            \begin{aligned}
                \norm{\mV_t-\mV^\star}_{\fro}&\leq  \const_1\frac{\kappa\mu r^{1.5}\log\left(\frac{1}{\alpha}\right)}{\sqrt{pd}}\quad \forall t\leq T\lesssim\frac{1}{\eta\sigma_{r}^{\star}}\log\left(\frac{1}{\alpha}\right).
            \end{aligned}
        \end{equation}
    \end{proposition}

    \begin{proof}
    First notice that $\mV_{t+1}$ can be rewritten as
    \begin{equation}
        \label{eqn::v-update}
            \begin{aligned}
                \mV_{t+1}&=\mZ_{t+1}\bigl(\mZ_{t+1}^{\top}\mZ_{t+1}\bigr)^{-1/2}\\
            &=\left(\mI +\eta \mM_t\right)\mV_t\bigl(\mV_t^{\top}\left(\mI +\eta \mM_t\right)^2\mV_t\bigr)^{-1/2}\\
            &=\left(\mI +\eta \mM_t\right)\mV_t\left(\mI +\mY_t\right)^{-1/2}
            \end{aligned}
    \end{equation}
    where we denote $\mY_t=\mV_t^{\top}\left(2\eta \mM_t+\eta^2\mM_t^2\right)\mV_t$. Next, we apply Taylor expansion for the matrix-valued function $f(\mX)=(\mI+\mX)^{-1/2}$, which states that for any $\mX$ satisfying $\norm{\mX}<1$,
    \begin{equation}
        f(\mX)=(\mI+\mX)^{-1/2}=\mI - \frac{1}{2}\mX + \mR(\mX) \text{ where } \mR(\mX)=\sum_{k=2}^{\infty}\frac{(-1)^k(2k)!}{4^k(k!)^2}\mX^k.
    \end{equation}
    Then, upon setting $\mX=\mY_t$ in the above equation and plugging it into \Cref{eqn::v-update} and rearranging the subterms, we have
    \begin{equation}
        \begin{aligned}
            \mV_{t+1}&=\left(\mI +\eta \mM_t\right)\mV_t\left(\mI - \frac{1}{2}\mV_t^{\top}\left(2\eta \mM_t+\eta^2\mM_t^2\right)\mV_t+\mR(\mY_t)\right)=\left(\mI +\eta\proj_{\mV_t}^{\perp} \mM_t\right)\mV_t+\mA_t,
        \end{aligned}
        \label{eqn::v-update-2}
    \end{equation}
    where
    \begin{equation}
        \mA_t=-\eta^2\mM_t\mV_t\mV_t^{\top}\mM_t\mV_{t}-0.5\eta^2\mV_t\mV_t^{\top}\mM_t^2\mV_t-0.5\eta^3\mM_t\mV_t\mV^{\star\top}\mM_t^2\mV_t+\left(\mI +\eta \mM_t\right)\mV_t\mR(\mY_t)
    \end{equation}
    contains all the higher-order terms. Next, according to triangle inequality, we can provide an upper bound for $\norm{\mV^{\star}- \mV_{t+1}}_{\fro}$ as follows
    \begin{equation}
        \begin{aligned}
            \norm{\mV^{\star}- \mV_{t+1}}_{\fro}\leq \underbrace{\norm{\mV^{\star}- \left(\mI +\eta\proj_{\mV_t}^{\perp} \mM_t\right)\mV_t}_{\fro}}_{:=(\rom{1})}+\norm{\mA_t}_{\fro}.
        \end{aligned}
    \end{equation} 
    We first control the leading term $(\rom{1})$. To this goal, we apply triangle inequality and obtain
    \begin{equation}
        \begin{aligned}
            (\rom{1})&\leq \norm{\mV^{\star}- \left(\mI +\eta\proj_{\mV_t}^{\perp} \left(\mX^{\star}-\mU_t\mU_t^{\top}\right)\right)\mV_t}_{\fro} + \eta\norm{\left(\id-\projection\right)\left(\mX^{\star}-\mU_t\mU_t^{\top}\right)\mV_t}_{\fro}\\
            &\leq \underbrace{\norm{\mV^{\star}- \left(\mI +\eta\proj_{\mV_t}^{\perp}\mDelta_t\right)\mV_t}_{\fro}}_{:=(\rom{1}_1)} + \eta\sqrt{r}\underbrace{\norm{\left(\id-\projection\right)\left(\mDelta_t\right)}}_{:=(\rom{1}_2)}.
        \end{aligned}
    \end{equation}
    To control $(\rom{1}_1)$, we further decompose it as
    \begin{equation}
        \begin{aligned}
            (\rom{1}_1)&\leq \norm{\mV^{\star}- \left(\mI +\eta\proj_{\mV_t}^{\perp} \mV^{\star}\mSigma^{\star}\mV^{\star\top}\right)\mV_t}_{\fro} + \eta\norm{\proj_{\mV_t}^{\perp} \mU_t\mU_t^{\top}\mV_t}_{\fro}\\
            &\stackrel{(a)}{\leq} \norm{\mV^{\star}-\mV_t- \eta\proj_{\mV_t}^{\perp}\left(\mV^{\star}-\mV_t\right)\mSigma^{\star}\mV^{\star\top}\mV_t}_{\fro} + \eta\sqrt{r}\norm{\mE_t\mU_t^{\top}\mV_t}\\
            &\stackrel{(b)}{\leq} \underbrace{\norm{\mV^{\star}-\mV_t- \eta\proj_{\mV_t}^{\perp}\left(\mV^{\star}-\mV_t\right)\mSigma^{\star}\mV^{\star\top}\mV_t}_{\fro}}_{:=(\rom{1}_{1, 1})} + 2\eta\sqrt{r\sigma^{\star}_{1}}\norm{\mE_t}.
        \end{aligned}
    \end{equation}
    Here in $(a)$, we use the fact that $\proj_{\mV_t}^{\perp}\mV_t=0$ and the definition $\mE_t=\proj_{\mV_t}^{\perp}\mU_t$. In $(b)$, we use the assumption that $\norm{\mS_t}\leq 2\sqrt{\sigma^{\star}_{1}}$.
    Next, according to the orthogonality of $\mV_t$ and $\mV_t^{\perp}$, we can upper-bound $(\rom{1}_{1, 1})$ as follows
    \begin{equation}
        \begin{aligned}
            (\rom{1}_{1, 1})^2&=\norm{\proj_{\mV_t}(\mV^{\star}-\mV_t)}_{\fro}^2+\norm{\proj_{\mV_t}^{\perp}\left(\mV^{\star}-\mV_t\right)\left(\mI - \eta \mSigma^{\star}\mV^{\star\top}\mV_t\right)}_{\fro}^2\\
            &\leq \norm{\proj_{\mV_t}(\mV^{\star}-\mV_t)}_{\fro}^2+\norm{\proj_{\mV_t}^{\perp}\left(\mV^{\star}-\mV_t\right)}_{\fro}^2\norm{\mI - \eta \mSigma^{\star}\mV^{\star\top}\mV_t}^2\\
            &\stackrel{(a)}{\leq}\norm{\proj_{\mV_t}(\mV^{\star}-\mV_t)}_{\fro}^2+\norm{\proj_{\mV_t}^{\perp}\left(\mV^{\star}-\mV_t\right)}_{\fro}^2\\
            &=\norm{\mV^{\star}-\mV_t}_{\fro}^2.
        \end{aligned}
    \end{equation}
    Here $(a)$ is due to the fact that $\norm{\mI - \eta \mSigma^{\star}\mV^{\star\top}\mV_t}\leq \norm{\mI - \eta \mSigma^{\star}}+\eta \sigma^{\star}_{1}\norm{\mV^{\star}-\mV_t}\leq 1-\eta(\sigma^{\star}_{r}-\sigma^{\star}_{1}\norm{\mV^{\star}-\mV_t})\leq 1$ since we assume $\norm{\mV^{\star}-\mV_t}\leq \frac{1}{2\kappa}$. Therefore, we derive that 
    \begin{equation}
        (\rom{1}_1)\leq \norm{\mV^{\star}-\mV_t}_{\fro}+2\eta \sqrt{r\sigma^{\star}_{1}}\norm{\mE_t}.
    \end{equation}
    On the other hand, Lemma~\ref{lem::helper-lemma-delta-t} tells us that, conditioned on $\event$, we have $(\rom{1}_2)\leq 21\const\sigma_{1}^{\star}\sqrt{\frac{\mu^2 r^2}{pd}}$. Therefore, we can conclude that
    \begin{equation}
        \begin{aligned}
            (\rom{1})&\leq \norm{\mV^{\star}-\mV_t}_{\fro}+2\eta \sqrt{r\sigma^{\star}_{1}}\norm{\mE_t}+\eta\sqrt{r}\cdot 21\const\sigma_{1}^{\star}\sqrt{\frac{\mu^2 r^2}{pd}}\leq \norm{\mV^{\star}-\mV_t}_{\fro}+22  \const\eta \frac{\sigma^{\star}_{1}\mu r^{1.5}}{\sqrt{pd}}.
        \end{aligned}
    \end{equation}
    Next, according to Lemma~\ref{lem::helper-lemma-A-t}, we can control $\norm{\mA_t}_{\fro}$ as 
    \begin{equation}
        \norm{\mA_t}_{\fro}\leq \sqrt{r}\norm{\mA_t}\leq 300\sqrt{r}\eta^2\sigma_1^{\star 2}.
    \end{equation}
    Putting everything together, we have 
    \begin{equation}
        \begin{aligned}
            \norm{\mV^{\star}-\mV_{t+1}}_{\fro}&\leq  \norm{\mV^{\star}-\mV_t}_{\fro}+22  \const\eta \frac{\sigma^{\star}_{1}\mu r^{1.5}}{\sqrt{pd}}+300\sqrt{r}\eta^2\sigma^{\star 2}_{1}\\
            &\leq \norm{\mV^{\star}-\mV_t}_{\fro}+23  \const\eta \frac{\sigma^{\star}_{1}\mu r^{1.5}}{\sqrt{pd}},
        \end{aligned}
    \end{equation}
    provided that $\eta\asymp\frac{  \mu r}{\sqrt{pd}\sigma^{\star}_{1}}$. This completes the proof.        
    \end{proof}
    
    \subsection{Proof of Proposition~\ref{prop::dynamic-of-v}}
    We restate the proposition here for clarity.
    \begin{proposition}[Dynamic of $\Big\|\big(\mV^{\star}- \mV_{t}^{(l)}\big)_{l,\cdot}\Big\|$]\label{prop_VstarVl}
        Under the same conditions as Proposition~\ref{prop::frobenius-norm-control-appendix}, for all $t\leq T\lesssim\frac{1}{\eta\sigma_{r}^{\star}}\log\left(\frac{1}{\alpha}\right)$, we have
        \begin{equation}
            \norm{\left(\mV^{\star}- \mV_{t+1}^{(l)}\right)_{l,\cdot}}\leq \left(1-0.5\eta\sigma_{r}^{\star}\right)\norm{\left(\mV^{\star}-\mV_t^{(l)}\right)_{l,\cdot}}+\const_2\eta\sigma_{1}^{\star}\frac{\kappa\mu^{1.5} r^{2}\log\left(\frac{1}{\alpha}\right)}{\sqrt{pd^2}}.
        \end{equation}
    \end{proposition}
    \begin{proof}
    Similar to \Cref{eqn::v-update-2}, we can express $\mV_{t+1}^{(l)}$ as 
    \begin{equation}
        \mV_{t+1}^{(l)}=\left(\mI +\eta\proj_{\mV_t^{(l)}}^{\perp} \mM_t^{(l)}\right)\mV_t^{(l)}+\mA_t^{(l)},
    \end{equation}
    where $\mM_t^{(l)}=\cR_{\Omega^{(l)}}\left(\mX^{\star}-\mV_t^{(l)}\mSigma_t\mV_t^{(l)\top}\right)$ and $\mA_t^{(l)}$ is defined as 
    \begin{equation}
        \begin{aligned}
            \mA_t^{(l)}&=-\eta^2\mM_t^{(l)}\mV_t^{(l)}\mV_t^{(l)\top}\mM_t^{(l)}\mV_{t}^{(l)}-0.5\eta^2\mV_t^{(l)}\mV_t^{(l)\top}\mM_t^{(l)2}\mV_t^{(l)}\\
            &\quad\,-0.5\eta^3\mM_t^{(l)}\mV_t^{(l)}\mV^{(l)\top}\mM_t^{(l)2}\mV_t^{(l)}+\left(\mI +\eta \mM_t\right)\mV_t^{(l)}\mR\left(\mY_t^{(l)}\right)
        \end{aligned}
    \end{equation}
    containing all the higher-order terms.
    Applying triangle inequality yields
    \begin{equation}
        \begin{aligned}
            \norm{\left(\mV^{\star}- \mV_{t+1}^{(l)}\right)_{l,\cdot}}&\leq \norm{\left(\mV^{\star}- \left(\mI +\eta\proj_{\mV_t^{(l)}}^{\perp} \mM_t^{(l)}\right)\mV_t^{(l)}\right)_{l,\cdot}}+\norm{\left(\mA_t^{(l)}\right)_{l,\cdot}}\\
            &=\underbrace{\norm{\left(\mV^{\star}- \left(\mI +\eta\proj_{\mV_t^{(l)}}^{\perp} \mXi_t^{(l)}\right)\mV_t^{(l)}\right)_{l,\cdot}}}_{:=(\rom{1})}+\norm{\left(\mA_t^{(l)}\right)_{l,\cdot}}.
        \end{aligned}
    \end{equation}
    Here in the last equality, we use the fact that $\proj_{\mV_t^{(l)}}^{\perp} \mM_t^{(l)}=\proj_{\mV_t^{(l)}}^{\perp}\mXi_t^{(l)}$ where $\mXi_t^{(l)}=\mM_t^{(l)}-\mV^{(l)}_t\mSigma^{\star}\mV^{\star\top}+\mV^{(l)}_t\mSigma_t\mV_t^{(l)\top}$.
    Upon noticing that $\proj_{\mV_t^{(l)}}^{\perp}=\mI-\proj_{\mV_t^{(l)}}$, we further decompose $(\rom{1})$ as follows
    \begin{equation}
        \begin{aligned}
            (\rom{1})\leq \underbrace{\norm{\left(\mV^{\star}- \left(\mI +\eta \mXi_t^{(l)}\right)\mV_t^{(l)}\right)_{l,\cdot}}}_{:=(\rom{1}_1)} + \eta\underbrace{\norm{\left(\proj_{\mV_t^{(l)}}\mXi_t^{(l)}\mV_t^{(l)}\right)_{l,\cdot}}}_{:=(\rom{1}_2)}.
        \end{aligned}
    \end{equation}
    To control $(\rom{1}_1)$, notice that the $l$-th row of $\mM_t^{(l)}$ is equal to the $l$-th row of $\mX^{\star}-\mV_t^{(l)}\mSigma_t\mV_t^{(l)\top}$ due to our choice of $\cR_{\Omega^{(l)}}$. Therefore, we have
    \begin{equation}
        \begin{aligned}
            (\rom{1}_1)&= \norm{\left(\mV^{\star}-\left(\mI+\eta\left(\mX^{\star}-\mV_t^{(l)}\mSigma_t\mV_t^{(l)\top}-\mV^{(l)}_t\mSigma^{\star}\mV^{\star\top}-\mV^{(l)}_t\mSigma_t\mV_t^{(l)\top}\right)\right)\mV_t^{(l)}\right)_{l, \cdot}}\\
            &= \norm{\left(\mV^{\star}-\left(\mI+\eta\left(\left(\mV^{\star}-\mV_t^{(l)}\right)\mSigma^{\star}\mV^{\star\top}\right)\right)\mV_t^{(l)}\right)_{l, \cdot}}\\
            &=\norm{\left(\left(\mV^{\star}-\mV_t^{(l)}\right)\left(\mI-\eta\mSigma^{\star}\mV^{\star\top}\mV_t^{(l)}\right)\right)_{l,\cdot}}\\
            &\leq \norm{\left(\mV^{\star}-\mV_t^{(l)}\right)_{l,\cdot}}\norm{\mI-\eta\mSigma^{\star}\mV^{\star\top}\mV_t^{(l)}}\\
            &\stackrel{(a)}{\leq} \left(1-0.5\eta\sigma_{r}^{\star}\right)\norm{\left(\mV^{\star}-\mV_t^{(l)}\right)_{l,\cdot}}.
        \end{aligned}
    \end{equation}
    Here in $(a)$, we use the fact that $\norm{\mI - \eta \mSigma^{\star}\mV^{\star\top}\mV_t^{(l)}}\leq \norm{\mI - \eta \mSigma^{\star}}+\eta \sigma^{\star}_{1}\norm{\mV^{\star}-\mV_t^{(l)}}\leq 1-\eta(\sigma^{\star}_{r}-\sigma^{\star}_{1}\norm{\mV^{\star}-\mV_t^{(l)}})\leq 1-0.5\eta\sigma_{r}^{\star}$ since we have $\norm{\mV^{\star}-\mV_t^{(l)}}\leq \frac{1}{2\kappa}$ according to the following proposition.
    \begin{lemma}[Frobenius norm control]
        \label{prop::frobenius-norm-control-LOO}
        Under the same conditions as Proposition~\ref{prop::frobenius-norm-control-appendix}, for any $1\leq t\leq T=\frac{100}{\eta\sigma_{r}^{\star}}\log\left(\frac{1}{\alpha}\right)$, for all $1\leq l\leq d$, we have
        \begin{equation}
            \begin{aligned}
                \norm{\mV^{\star}-\mV_{t}^{(l)}}_{\fro}&\leq  \const_1\frac{\kappa\mu r^{1.5}\log\left(\frac{1}{\alpha}\right)}{\sqrt{pd}}.
            \end{aligned}
        \end{equation}
    \end{lemma} 
    The proof of the above lemma is the same as that of Proposition~\ref{prop::frobenius-norm-control-appendix} and hence omitted here.
    Next, for $(\rom{1}_2)$, we first have
    \begin{equation}
        \begin{aligned}
            (\rom{1}_2)&=\norm{\left(\mV_t^{(l)}\mV_t^{(l)\top}\mXi_t^{(l)}\mV_t^{(l)}\right)_{l,\cdot}}\leq \norm{\left(\mV_t^{(l)}\right)_{l,\cdot}}\norm{\mV_t^{(l)\top}\mXi_t^{(l)}\mV_t^{(l)}}\leq \norm{\left(\mV_t^{(l)}\right)_{l,\cdot}}\norm{\mXi_t^{(l)}}.
        \end{aligned}
        \label{eqn::v-star-v-t}
    \end{equation} 
    For the first part, we have
    \begin{equation}
        \norm{\left(\mV_t^{(l)}\right)_{l,\cdot}}\leq \norm{\mV^{\star}_{l,\cdot}}+\norm{\left(\mV^{\star}-\mV_t^{(l)}\right)_{l,\cdot}}\leq \sqrt{\frac{\mu r}{d}}+\norm{\left(\mV^{\star}-\mV_t^{(l)}\right)_{l,\cdot}}\leq 2\sqrt{\frac{\mu r}{d}}.
        \label{eqn::v-t}
    \end{equation}
    Here we use the assumption that $\norm{\left(\mV^{\star}-\mV_t^{(l)}\right)_{l,\cdot}}\leq \sqrt{\frac{\mu r}{d}}$.
    For the second part, we have  
    \begin{equation}
        \begin{aligned}
            \norm{\mXi_t^{(l)}}&=\norm{\left(\mV^{\star}-\mV_t^{(l)}\right)\mSigma^{\star}\mV^{\star\top}-\left(\id-\cR_{\Omega^{(l)}}\right)\left(\mX^{\star}-\mV_t^{(l)}\mSigma_t\mV_t^{(l)\top}\right)}\\
            &\leq \norm{\left(\mV^{\star}-\mV_t^{(l)}\right)\mSigma^{\star}\mV^{\star\top}}+\norm{\left(\id-\cR_{\Omega^{(l)}}\right)\left(\mX^{\star}-\mV_t^{(l)}\mSigma_t\mV_t^{(l)\top}\right)}\\
            &\stackrel{(a)}{\leq} \sigma^{\star}_{1}\norm{\mV^{\star}-\mV_t^{(l)}}+21  \const\frac{\sigma^{\star}_{1}\mu r}{\sqrt{pd}}\\
            &\stackrel{(b)}{\leq} \sigma^{\star}_{1}\cdot \const_1\frac{\kappa\mu r^{1.5}\log\left(\frac{1}{\alpha}\right)}{\sqrt{pd}}+21  \const\frac{\sigma^{\star}_{1}\mu r}{\sqrt{pd}}\\
            &\leq 2\const_1\sigma^{\star}_{1}\frac{\kappa\mu r^{1.5}\log\left(\frac{1}{\alpha}\right)}{\sqrt{pd}}.
            \label{eqn::xi-t}
        \end{aligned}
    \end{equation}
    Here in $(a)$, we apply Lemma~\ref{lem::helper-lemma-delta-t}. In $(b)$, we apply Lemma~\ref{prop::frobenius-norm-control-LOO}.
    Invoking \Cref{eqn::v-t,eqn::xi-t} in \Cref{eqn::v-star-v-t}, we obtain that
    \begin{equation}
        (\rom{1}_2)\leq 2\sqrt{\frac{\mu r}{d}}\cdot2\const_1\sigma^{\star}_{1}\frac{\kappa\mu r^{1.5}\log\left(\frac{1}{\alpha}\right)}{\sqrt{pd}}=4\const_1\sigma_{1}^{\star}\frac{\kappa\mu^{1.5}r^2\log\left(\frac{1}{\alpha}\right)}{\sqrt{pd^2}}.
    \end{equation}
    Next, we control the higher-order term $\norm{\left(\mA_t^{(l)}\right)_{l,\cdot}}$.
    To this end, we first notice that the $l$-th row of $\mA_t^{(l)}$ is equal to the $l$-th row of
    \begin{equation}
        \begin{aligned}
            &-\eta^2\mDelta_t^{(l)}\proj_{\mV_t^{(l)}}\mV_t^{(l)\top}\mM_t^{(l)}\mV_{t}^{(l)}-0.5\eta^2\proj_{\mV_t^{(l)}}\mM_t^{(l)2}\mV_t^{(l)}-0.5\eta^3\mDelta_t^{(l)}\proj_{\mV_t^{(l)}}\mM_t^{(l)2}\mV_t^{(l)}\\
            &\qquad+\left(\mI +\eta \mDelta_t^{(l)}\right)\mV_t^{(l)}\mR\left(\mY_t^{(l)}\right).
        \end{aligned}
    \end{equation}
    Therefore, we can upper-bound its operator norm by 
    \begin{equation}
        \begin{aligned}
            \norm{\left(\mA_t^{(l)}\right)_{l,\cdot}}&\leq \norm{\left(\mDelta_t^{(l)}\right)_{l,\cdot}}\left(\eta^2\norm{\mM_t^{(l)}}+0.5\eta^3\norm{\mM_t^{(l)}}^2+\eta \norm{\mR\left(\mY_t^{(l)}\right)}\right)\\
            &\qquad +\norm{\left(\mV_t^{(l)}\right)_{l,\cdot}}\left(0.5\eta^2\norm{\mM_t^{(l)}}^2+\norm{\mR\left(\mY_t^{(l)}\right)}\right).
        \end{aligned}
    \end{equation}
    Next, we notice that $\norm{\left(\mV_t^{(l)}\right)_{l,\cdot}}\leq 2\sqrt{\frac{\mu r}{d}}$ and 
    \begin{equation}
        \begin{aligned}
            \norm{\left(\mDelta_t^{(l)}\right)_{l,\cdot}}&\leq\norm{\left(\mX^{\star}-\mV_t^{(l)}\mSigma_t\mV_t^{(l)\top}\right)_{l,\cdot}}\\
            &\leq \sigma^{\star}_{1}\sqrt{\frac{\mu r}{d}}+\norm{\left(\mV_t^{(l)}\right)_{l,\cdot}}\norm{\mSigma_t}\\
            &\leq 5\sigma^{\star}_{1}\sqrt{\frac{\mu r}{d}}.
        \end{aligned}
    \end{equation}
    On the other hand, according to Lemma~\ref{lem::helper-lemma-A-t}, we have $\norm{\mM_t^{(l)}}\leq 6\sqrt{\sigma^{\star}_{1}}$ and $\norm{\mR\left(\mY_t^{(l)}\right)}\leq 200\eta^2\sigma_1^{\star 2}$. Therefore, we derive that
    \begin{equation}
        \norm{\left(\mA_t^{(l)}\right)_{l,\cdot}}\leq 470\eta^2\sigma^{\star 2}_{1}\sqrt{\frac{\mu r}{d}}.
    \end{equation}
    Overall, we obtain 
    \begin{equation}
        \begin{aligned}
            \norm{\left(\mV^{\star}- \mV_{t+1}^{(l)}\right)_{l,\cdot}}&\leq \left(1-0.5\eta\sigma_{r}^{\star}\right)\norm{\left(\mV^{\star}-\mV_t^{(l)}\right)_{l,\cdot}}+4\const_1\eta\sigma_{1}^{\star}\frac{\kappa\mu^{1.5}r^2\log\left(\frac{1}{\alpha}\right)}{\sqrt{pd^2}}+ 470\eta^2\sigma^{\star 2}_{1}\sqrt{\frac{\mu r}{d}}\\
            &\leq \left(1-0.5\eta\sigma_{r}^{\star}\right)\norm{\left(\mV^{\star}-\mV_t^{(l)}\right)_{l,\cdot}}+8\const_1\eta\sigma_{1}^{\star}\frac{\kappa\mu^{1.5}r^2\log\left(\frac{1}{\alpha}\right)}{\sqrt{pd^2}},
        \end{aligned}
    \end{equation}
    provided that $\eta\asymp\frac{\mu r}{\sqrt{pd}\sigma^{\star}_{1}}$. This completes the proof of Proposition~\ref{prop_VstarVl}.
    \end{proof}

    \subsection{Proof of Proposition~\ref{prop::incoherence-dynamic}}
    We restate the proposition here for clarity.
    \begin{proposition}[One-step dynamic of $\bigl\|\mV_t-\mV_t^{(l)}\bigr\|_{\fro}$]
        Suppose that the sampling rate satisfies $p\gtrsim \frac{\kappa^6\mu^4r^9\log^4(\frac{1}{\alpha})\log^2(d)}{d}$. Suppose that $\norm{\mV^{\star}-\mV_t}\leq \frac{1}{2\kappa}$ and $\big\|\mV_t-\mV_t^{(l)}\big\|_{\fro}\leq \sqrt{\frac{\mu r}{4d}}$. With probability at least $1-\frac{1}{d^3}$, for any $1\leq t\leq T\lesssim\frac{1}{\eta\sigma_{r}^{\star}}\log\left(\frac{1}{\alpha}\right)$, we have
        \begin{equation}
            \norm{\mV_{t+1}-\mV_{t+1}^{(l)}}_{\fro}\leq \norm{\mV_t-\mV_t^{(l)}}_{\fro}+\const\eta \sigma_1^{\star}\sqrt{\frac{\kappa\mu^3r^{5.5}\log\left(\frac{1}{\alpha}\right)\log\left(d\right)}{\sqrt{pd}\cdot d}}.
        \end{equation}
    \end{proposition}
    \begin{proof}
    Note that 
    \begin{equation}
        \mV_{t+1}=\left(\mI +\eta\proj_{\mV_t}^{\perp} \mM_t\right)\mV_t+\mA_t\quad\text{and} \quad\mV_{t+1}^{(l)}=\left(\mI +\eta\proj_{\mV_t^{(l)}}^{\perp} \mM_t^{(l)}\right)\mV_t^{(l)}+\mA_t^{(l)}.
    \end{equation}
    Hence, we can expand $\norm{\mV_{t+1}-\mV_{t+1}^{(l)}}_{\fro}^2$ as 
    \begin{equation}
        \begin{aligned}
            \norm{\mV_{t+1}-\mV_{t+1}^{(l)}}_{\fro}^2&=\norm{\mV_{t}-\mV_{t}^{(l)}}_{\fro}^2+2\eta\underbrace{\inner{\mV_{t}-\mV_{t}^{(l)}}{\proj_{\mV_t}^{\perp}\mM_t\mV_t-\proj_{\mV_t^{(l)}}^{\perp}\mM_t^{(l)}\mV_t^{(l)}}}_{:=(\rom{1})}+(\rom{2}),
        \end{aligned}
    \end{equation}
    where
    \begin{equation}
        \begin{aligned}
            (\rom{2})&=\eta^2\norm{\proj_{\mV_t}^{\perp}\mM_t\mV_t-\proj_{\mV_t^{(l)}}^{\perp}\mM_t^{(l)}\mV_t^{(l)}}_{\fro}^2+\norm{\mA_t-\mA_t^{(l)}}_{\fro}^2\\
            &\quad + 2\inner{\left(\mI +\eta\proj_{\mV_t}^{\perp} \mM_t\right)\mV_t-\left(\mI +\eta\proj_{\mV_t^{(l)}}^{\perp} \mM_t^{(l)}\right)\mV_t^{(l)}}{\mA_t-\mA_t^{(l)}}
        \end{aligned}
    \end{equation}
    contains all the higher-order terms.

    We first control $(\rom{1})$. Notice that $\proj_{\mV_t}^{\perp}\mM_t=\proj_{\mV_t}^{\perp}\mXi_t$ and $\proj_{\mV_t^{(l)}}^{\perp}\mM_t^{(l)}=\proj_{\mV_t^{(l)}}^{\perp}\mXi_t^{(l)}$. Here we define $\mXi_t=\mM_t-\mV_t\mSigma^{\star}\mV^{\star\top}+\mV_t\mSigma_t\mV_t^{\top}$ and $\mXi_t^{(l)}=\mM_t^{(l)}-\mV^{(l)}_t\mSigma^{\star}\mV^{\star\top}+\mV^{(l)}_t\mSigma_t\mV_t^{(l)\top}$, respectively. Therefore, we can decompose $(\rom{1})$ as follows
    \begin{equation}
        \begin{aligned}
            (\rom{1})&=\inner{\mV_{t}-\mV_{t}^{(l)}}{\proj_{\mV_t}^{\perp}\mXi_t\mV_t-\proj_{\mV_t^{(l)}}^{\perp}\mXi_t^{(l)}\mV_t^{(l)}}\\
            &= \underbrace{\inner{\mV_{t}-\mV_{t}^{(l)}}{\left(\proj_{\mV_t}^{\perp}-\proj_{\mV_t^{(l)}}^{\perp}\right)\mXi_t^{(l)}\mV_t^{(l)}}}_{:=(\rom{1}_1)}+\underbrace{\inner{\mV_{t}-\mV_{t}^{(l)}}{\proj_{\mV_t}^{\perp}\mXi_t\left(\mV_t-\mV_t^{(l)}\right)}}_{:=(\rom{1}_2)}\\
            &\qquad+\underbrace{\inner{\mV_{t}-\mV_{t}^{(l)}}{\proj_{\mV_t}^{\perp}\left(\mXi_t-\mXi_t^{(l)}\right)\mV_t^{(l)}}}_{:=(\rom{1}_3)}.
        \end{aligned}
    \end{equation}
    Next, we provide upper bounds for these terms separately. For $(\rom{1}_1)$, applying Cauchy-Schwarz inequality leads to
    \begin{equation}
        \begin{aligned}
            (\rom{1}_1)&\leq \norm{\mV_{t}-\mV_{t}^{(l)}}_{\fro}\norm{\left(\proj_{\mV_t}^{\perp}-\proj_{\mV_t^{(l)}}^{\perp}\right)\mXi_t^{(l)}\mV_t^{(l)}}_{\fro}\\
            &\leq \norm{\mV_{t}-\mV_{t}^{(l)}}_{\fro}\norm{\proj_{\mV_t}^{\perp}-\proj_{\mV_t^{(l)}}^{\perp}}_{\fro}\norm{\mXi_t^{(l)}}\\
            &\stackrel{(a)}{\leq} 2\norm{\mV_{t}-\mV_{t}^{(l)}}_{\fro}^2\norm{\mXi_t^{(l)}}\\
            &\stackrel{(b)}{\leq} 4\norm{\mV_{t}-\mV_{t}^{(l)}}_{\fro}^2 \cdot \const_1\sigma^{\star}_{1}\frac{\kappa\mu r^{1.5}\log\left(\frac{1}{\alpha}\right)}{\sqrt{pd}}.
        \end{aligned}
    \end{equation}
    Here in $(a)$, we apply Lemma~\ref{lem::orthogonal-matrix}. In $(b)$, we use the result from \Cref{eqn::xi-t}.
    Similarly, $(\rom{1}_2)$, we have 
    \begin{equation}
        \begin{aligned}
            (\rom{1}_2)&\leq \norm{\mV_{t}-\mV_{t}^{(l)}}_{\fro}^2\norm{\mXi_t}\leq 2\norm{\mV_{t}-\mV_{t}^{(l)}}_{\fro}^2 \cdot \const_1\sigma^{\star}_{1}\frac{\kappa\mu r^{1.5}\log\left(\frac{1}{\alpha}\right)}{\sqrt{pd}}.
        \end{aligned}
    \end{equation}
    Next, for $(\rom{1}_3)$, we further decompose it as 
    \begin{equation}
        \begin{aligned}
            (\rom{1}_3)&=\inner{\mV_{t}-\mV_{t}^{(l)}}{\proj_{\mV_t}^{\perp}\left(\mM_t-\mM_t^{(l)}+\mV^{(l)}_t\mSigma^{\star}\mV^{\star\top}-\mV^{(l)}_t\mSigma_t\mV_t^{(l)\top}\right)\mV_t^{(l)}}\\
            &=\underbrace{\inner{\mV_{t}-\mV_{t}^{(l)}}{\proj_{\mV_t}^{\perp}\left(\projection-\cR_{\Omega^{(l)}}\right)\left(\mX^{\star}-\mV_t^{(l)}\mSigma_t\mV_t^{(l)\top}\right)\mV_t^{(l)}}}_{:=(\rom{1}_{3, 1})}\\
            &\qquad + \underbrace{\inner{\mV_{t}-\mV_{t}^{(l)}}{\proj_{\mV_t}^{\perp}\left(\id-\cR_{\Omega}\right)\left(\mV_t\mSigma_t\left(\mV_t-\mV_t^{(l)}\right)^{\top}\right)\mV_t^{(l)}}}_{:=(\rom{1}_{3, 2})}\\
            &\qquad + \underbrace{\inner{\mV_{t}-\mV_{t}^{(l)}}{\proj_{\mV_t}^{\perp}\left(\id-\cR_{\Omega}\right)\left(\left(\mV_t-\mV_t^{(l)}\right)\mSigma_t\mV_t^{(l)\top}\right)\mV_t^{(l)}}}_{:=(\rom{1}_{3, 3})}\\
            &\qquad +\underbrace{\inner{\mV_{t}-\mV_{t}^{(l)}}{\proj_{\mV_t}^{\perp}\mV_t^{(l)}\mSigma^{\star}\mV^{\star\top}\mV_t^{(l)}}}_{:=(\rom{1}_{3, 4})}.
        \end{aligned}
    \end{equation}
    We control these terms separately. First, we present the following key proposition, the proof of which is deferred to the end of this section.
    \begin{proposition}
        \label{prop::incoherence-dynamic-1}
        For all $0\leq t\leq T$, we have 
        \begin{equation}
            \norm{\left(\projection-\cR_{\Omega^{(l)}}\right)\left(\mX^{\star}-\mV_t^{(l)}\mSigma_t\mV_t^{(l)\top}\right)\mV_t^{(l)}}_{\fro}\leq \const_3\sigma_1^{\star}\sqrt{\frac{\kappa\mu^3r^{5.5}\log\left(\frac{1}{\alpha}\right)\log\left(d\right)}{\sqrt{pd}\cdot d}}.
        \end{equation}
    \end{proposition}
    
    For $(\rom{1}_{3, 1})$, we have
    \begin{equation}
        \begin{aligned}
            (\rom{1}_{3, 1})&\leq \norm{\mV_{t}-\mV_{t}^{(l)}}_{\fro}\norm{\left(\projection-\cR_{\Omega^{(l)}}\right)\left(\mX^{\star}-\mV_t^{(l)}\mSigma_t\mV_t^{(l)\top}\right)\mV_t^{(l)}}_{\fro}\\
            &\leq \const_3\sigma_1^{\star}\sqrt{\frac{\kappa\mu^3r^{5.5}\log\left(\frac{1}{\alpha}\right)\log\left(d\right)}{\sqrt{pd}\cdot d}}\norm{\mV_{t}-\mV_{t}^{(l)}}_{\fro}.
        \end{aligned}
    \end{equation}
    Here in the last inequality, we use Proposition~\ref{prop::incoherence-dynamic-1}. Next, we apply Lemma~\ref{lem::concentration-chen} to control $(\rom{1}_{3, 2})$. Specifically, upon setting $\mA=\mV_{t}$, $\mB=\proj_{\mV_t}^{\perp}\left(\mV_{t}-\mV_{t}^{(l)}\right)$, $\mC=\left(\mV_{t}-\mV_t^{(l)}\right)\mSigma_{t}$, and $\mD=\mV_{t}^{(l)}$ in Lemma~\ref{lem::concentration-chen}, with probability at least $1-d^{-3}$, one has 
    \begin{equation}
        \begin{aligned}
            (\rom{1}_{3, 2})&=\inner{\left(\id-\projection\right)\left(\mA\mC^{\top}\right)}{\mB\mD^{\top}}\\
            &\leq  \const\sqrt{\frac{d}{p}}\norm{\mA}_{2, \infty}\norm{\mB}_{\fro}\cdot \norm{\mC}_{\fro}\norm{\mD}_{2, \infty}\\
            &= \const\sqrt{\frac{d}{p}}\norm{\mV_{t}}_{2, \infty}\norm{\mV_{t}^{(l)}}_{2, \infty}\norm{\proj_{\mV_t}^{\perp}\left(\mV_{t}-\mV_{t}^{(l)}\right)}_{\fro}\norm{\left(\mV_{t}-\mV_t^{(l)}\right)\mSigma_{t}}_{\fro}\\
            &\leq 16 \const\sigma^{\star}_{1}\sqrt{\frac{\mu^2r^2}{pd}}\norm{\mV_{t}-\mV_{t}^{(l)}}_{\fro}^2.
        \end{aligned}
    \end{equation}
    Similarly, we apply Lemma~\ref{lem::concentration-chen} to control $(\rom{1}_{3, 3})$. Upon setting $\mA=\left(\mV_{t}-\mV_t^{(l)}\right)\mSigma_{t}$, $\mB=\proj_{\mV_t}^{\perp}\left(\mV_{t}-\mV_{t}^{(l)}\right)$, $\mC=\mV_{t}$, and $\mD=\mV_{t}^{(l)}$ in Lemma~\ref{lem::concentration-chen}, with probability at least $1-d^{-3}$, one has 
    \begin{equation}
        \begin{aligned}
            (\rom{1}_{3, 3})&\leq \const\sqrt{\frac{d}{p}}\norm{\mV_{t}}_{\fro}\norm{\mV_{t}^{(l)}}_{2, \infty}\norm{\proj_{\mV_t}^{\perp}\left(\mV_{t}-\mV_{t}^{(l)}\right)}_{\fro}\norm{\left(\mV_{t}-\mV_t^{(l)}\right)\mSigma_{t}}_{2, \infty}\\
            &\leq 8\const \sigma_1^{\star}\sqrt{\frac{\mu r^2}{p}}\norm{\mV_{t}-\mV_{t}^{(l)}}_{\fro}\norm{\mV_{t}-\mV_{t}^{(l)}}_{2, \infty}.
        \end{aligned}
    \end{equation}
    To control $\norm{\mV_{t}-\mV_{t}^{(l)}}_{2, \infty}$, we apply Lemma~\ref{lem::2-inf-norm-decomposition}:
    \begin{equation}
        \begin{aligned}
            \norm{\mV_{t}-\mV_{t}^{(l)}}_{2, \infty}&\leq \norm{\mV_{t}-\mV_{t}^{(l)}}_{\fro}\norm{\mL_{\mV_{t}-\mV_{t}^{(l)}}}_{2, \infty}\\
            &\leq \norm{\mV_{t}-\mV_{t}^{(l)}}_{\fro}\sqrt{\norm{\mL_{\mV_{t}}}_{2, \infty}^2+\norm{\mL_{\mV_{t}^{(l)}}}_{2, \infty}^2}\\
            &\leq 3\sqrt{\frac{\mu r}{d}}\norm{\mV_{t}-\mV_{t}^{(l)}}_{\fro}.
        \end{aligned}
    \end{equation}
    Hence, we have 
    \begin{equation}
        (\rom{1}_{3, 3})\leq 24\const \sigma_1^{\star}\sqrt{\frac{\mu^2 r^3}{pd}}\norm{\mV_{t}-\mV_{t}^{(l)}}_{\fro}^2.
    \end{equation}
    Lastly, for $(\rom{1}_{3, 4})$, we notice that
    \begin{equation}
        \begin{aligned}
            (\rom{1}_{3, 4})&=-\inner{\mV_{t}-\mV_{t}^{(l)}}{\proj_{\mV_t}^{\perp}\left(\mV_{t}-\mV_{t}^{(l)}\right)\mSigma^{\star}\mV^{\star\top}\mV_t^{(l)}}\\
            &\leq -\inner{\mV_{t}-\mV_{t}^{(l)}}{\proj_{\mV_t}^{\perp}\left(\mV_{t}-\mV_{t}^{(l)}\right)\mSigma^{\star}}+\sigma^{\star}_{1}\norm{\mV^{\star}-\mV^{(l)}_t}_{\fro}\norm{\mV_t-\mV_t^{(l)}}_{\fro}^2\\
            &\stackrel{(a)}{\leq} \sigma^{\star}_{1}\norm{\mV^{\star}-\mV^{(l)}_t}_{\fro}\norm{\mV_t-\mV_t^{(l)}}_{\fro}^2\\
            &\stackrel{(b)}{\leq} \const_1\sigma_{1}^{\star}\frac{\kappa\mu r^{1.5}\log\left(\frac{1}{\alpha}\right)}{\sqrt{pd}}\norm{\mV_t-\mV_t^{(l)}}_{\fro}^2.
        \end{aligned}
    \end{equation}
    Here in $(a)$, we use the fact that
    \begin{equation}
        \inner{\mV_{t}-\mV_{t}^{(l)}}{\proj_{\mV_t}^{\perp}\left(\mV_{t}-\mV_{t}^{(l)}\right)\mSigma^{\star}}=\norm{\proj_{\mV_t}^{\perp}\left(\mV_{t}-\mV_{t}^{(l)}\right)\mSigma^{\star 1/2}}_{\fro}^2\geq 0.
    \end{equation} 
    In $(b)$, we apply Proposition~\ref{prop::frobenius-norm-control-appendix}. Therefore, we obtain
    \begin{equation}
        \begin{aligned}
            (\rom{1}_3)\leq C_1\sigma_{1}^{\star}\frac{\kappa\mu r^{1.5}\log\left(\frac{1}{\alpha}\right)}{\sqrt{pd}}\norm{\mV_t-\mV_t^{(l)}}_{\fro}^2 + \const_3\sigma_1^{\star}\sqrt{\frac{\kappa\mu^3r^{5.5}\log\left(\frac{1}{\alpha}\right)\log\left(d\right)}{\sqrt{pd}\cdot d}}\norm{\mV_{t}-\mV_{t}^{(l)}}_{\fro},
        \end{aligned}
    \end{equation}
    which implies that
    \begin{equation}
        (\rom{1})\leq C_2\sigma_{1}^{\star}\frac{\kappa\mu r^{1.5}\log\left(\frac{1}{\alpha}\right)}{\sqrt{pd}}\norm{\mV_t-\mV_t^{(l)}}_{\fro}^2 + \const_3\sigma_1^{\star}\sqrt{\frac{\kappa\mu^3r^{5.5}\log\left(\frac{1}{\alpha}\right)\log\left(d\right)}{\sqrt{pd}\cdot d}}\norm{\mV_{t}-\mV_{t}^{(l)}}_{\fro}.
    \end{equation}
    Next, we move on to controlling $(\rom{2})$. First, by the basic inequality $2ab\leq a^2+b^2$, we bound $(\rom{2})$ as
    \begin{equation}
        (\rom{2})\leq 2\eta^2\norm{\proj_{\mV_t}^{\perp}\mM_t\mV_t-\proj_{\mV_t^{(l)}}^{\perp}\mM_t^{(l)}\mV_t^{(l)}}_{\fro}^2+2\norm{\mA_t-\mA_t^{(l)}}_{\fro}^2+ 2\inner{\mV_t-\mV_t^{(l)}}{\mA_t-\mA_t^{(l)}}.
    \end{equation}
    Next, we provide the control over $\norm{\proj_{\mV_t}^{\perp}\mM_t\mV_t-\proj_{\mV_t^{(l)}}^{\perp}\mM_t^{(l)}\mV_t^{(l)}}_{\fro}$. The remaining two terms can be controlled in a similar fashion, so we omit their analysis. We first apply triangle inequality to obtain 
    \begin{equation}
        \begin{aligned}
            \norm{\proj_{\mV_t}^{\perp}\mM_t\mV_t-\proj_{\mV_t^{(l)}}^{\perp}\mM_t^{(l)}\mV_t^{(l)}}_{\fro}&\leq 2\norm{\mM_t}\norm{\mV_t-\mV_t^{(l)}}_{\fro}+\norm{\left(\mM_t-\mM_t^{(l)}\right)\mV_t^{(l)}}_{\fro}\\
            &\leq 12\sigma_1^{\star}\norm{\mV_t-\mV_t^{(l)}}_{\fro}+\norm{\left(\mM_t-\mM_t^{(l)}\right)\mV_t^{(l)}}_{\fro}.
        \end{aligned}
    \end{equation}
    For the second term, we further decompose it as
    \begin{equation}
        \begin{aligned}
            \norm{\left(\mM_t-\mM_t^{(l)}\right)\mV_t^{(l)}}_{\fro}&\leq \norm{\left(\projection-\cR_{\Omega^{(l)}}\right)\left(\mX^{\star}-\mV_t^{(l)}\mSigma_t\mV_t^{(l)\top}\right)\mV_t^{(l)}}_{\fro}\\
            &\quad +\norm{\left(\id-\cR_{\Omega}\right)\left(\mV_t\mSigma_t\left(\mV_t-\mV_t^{(l)}\right)^{\top}\right)\mV_t}_{\fro}\\
            &\quad +\norm{\left(\id-\cR_{\Omega}\right)\left(\left(\mV_t-\mV_t^{(l)}\right)\mSigma_t\mV_t^{\top}\right)\mV_t}_{\fro}\\
            &\quad +4\sigma_1^{\star}\norm{\mV_t-\mV_t^{(l)}}_{\fro}.
        \end{aligned}
    \end{equation}
    The first term can be controlled by Proposition~\ref{prop::incoherence-dynamic-1}. For the second term, we have 
    \begin{equation}
        \begin{aligned}
            \norm{\left(\id-\cR_{\Omega}\right)\left(\mV_t\mSigma_t\left(\mV_t-\mV_t^{(l)}\right)^{\top}\right)\mV_t}_{\fro}&=\sup_{\norm{\mZ}_{\fro}\leq 1}\inner{\left(\id-\cR_{\Omega}\right)\left(\mV_t\mSigma_t\left(\mV_t-\mV_t^{(l)}\right)^{\top}\right)}{\mZ\mV_t}\\
            &\stackrel{(a)}{\leq} \const\sqrt{\frac{d}{p}}\norm{\mV_t}_{2, \infty}^2\norm{\left(\mV_t-\mV_t^{(l)}\right)\mSigma_t}_{\fro}\\
            &\leq 8\const\sigma^{\star}_{1}\sqrt{\frac{\mu^2r^2}{pd}}\norm{\mV_t-\mV_t^{(l)}}_{\fro}.
        \end{aligned}
    \end{equation}
    Here we apply Corollary~\ref{lem::concentration-chen} in $(a)$. Similarly, we can also control the third term as follows
    \begin{equation}
        \begin{aligned}
            \norm{\left(\id-\cR_{\Omega}\right)\left(\left(\mV_t-\mV_t^{(l)}\right)\mSigma_t\mV_t^{\top}\right)\mV_t}_{\fro}&=\sup_{\norm{\mZ}_{\fro}\leq 1}\inner{\left(\id-\cR_{\Omega}\right)\left(\left(\mV_t-\mV_t^{(l)}\right)\mSigma_t\mV_t^{\top}\right)}{\mZ\mV_t}\\
            &\leq 8\const\sqrt{\frac{\mu r^2}{p}}\norm{\mV_t-\mV_t^{(l)}}_{2, \infty}\\
            &\leq 24\const \sigma_1^{\star}\sqrt{\frac{\mu^2 r^3}{pd}}\norm{\mV_{t}-\mV_{t}^{(l)}}_{\fro}.
        \end{aligned}
    \end{equation}
    Combining the above inequalities leads to
    \begin{equation}
        \norm{\proj_{\mV_t}^{\perp}\mM_t\mV_t-\proj_{\mV_t^{(l)}}^{\perp}\mM_t^{(l)}\mV_t^{(l)}}_{\fro}\leq C\sigma_1^{\star}\norm{\mV_t-\mV_t^{(l)}}_{\fro}.
    \end{equation}
    Similarly, we can derive that $\norm{\mA_t-\mA_t^{(l)}}_{\fro}\leq C\sigma_1^{\star}\norm{\mV_t-\mV_t^{(l)}}_{\fro}+C\sigma_1^{\star}\sqrt{\frac{\kappa\mu^3r^{5.5}\log\left(\frac{1}{\alpha}\right)\log\left(d\right)}{\sqrt{pd}\cdot d}}$. Therefore, we have 
    \begin{equation}
        (\rom{2})\leq C_3\eta^2 \left(\sigma_1^{\star}\norm{\mV_t-\mV_t^{(l)}}_{\fro}^2+\sigma_1^{\star 2}\frac{\kappa\mu^3r^{5.5}\log\left(\frac{1}{\alpha}\right)\log\left(d\right)}{\sqrt{pd}\cdot d}\right).
    \end{equation}
    Putting everything together, we obtain
    \begin{equation}
        \begin{aligned}
            \norm{\mV_{t+1}\!-\!\mV_{t+1}^{(l)}}_{\fro}&\leq \left(1\!+\!\const_1\eta\sigma_{1}^{\star}\frac{\kappa\mu r^{1.5}\log\left(\frac{1}{\alpha}\right)}{\sqrt{pd}}\right)\norm{\mV_t\!-\!\mV_t^{(l)}}_{\fro}+\const_3\eta \sigma_1^{\star}\sqrt{\frac{\kappa\mu^3r^{5.5}\log\left(\frac{1}{\alpha}\right)\log\left(d\right)}{\sqrt{pd}\cdot d}}\\
            &\leq \norm{\mV_t-\mV_t^{(l)}}_{\fro}+2\const_3\eta \sigma_1^{\star}\sqrt{\frac{\kappa\mu^3r^{5.5}\log\left(\frac{1}{\alpha}\right)\log\left(d\right)}{\sqrt{pd}\cdot d}}.
        \end{aligned}
    \end{equation}
\end{proof}  

\subsection{Proof of Proposition~\ref{prop::incoherence-dynamic-1}}
To prove this proposition, we propose a novel argument based on {\it adaptive $\epsilon$-nets}. Upon fixing an $\epsilon$-net $\cN_{\epsilon}$ of $\cB^{r\times r}_{\mathrm{op}}(4\sigma_{1}^{\star})$ with $\epsilon=\frac{c}{d}$, we first construct a series of adaptive $\epsilon$-nets $\left\{\cV^{(l)}_{\epsilon, t}\right\}_{t=0}^{T}$ in the following recursive manner.
    \begin{equation}
        \begin{aligned}
            \cV^{(l)}_{\epsilon, t+1}&=\left\{\mY^{(l)}_{\epsilon, t+1}:\norm{\mY^{(l)}_{\epsilon, t+1}}_{2, \infty}\leq 2\sqrt{\frac{\mu r}{d}}, \text{ and }\mY^{(l)}_{\epsilon, t+1}=\mZ^{(l)}_{\epsilon, t+1}\left(\mZ^{(l)\top}_{\epsilon, t+1}\mZ^{(l)}_{\epsilon, t+1}\right)^{-1/2}\text{ where}\right.\\
            &\quad\quad\left.\mZ^{(l)}_{\epsilon, t+1}=\left(\mI+\eta\cR_{\Omega^{(l)}}\left(\mX^{\star}-\mV_{\epsilon, t}^{(l)}\mSigma_{\epsilon, t}\mV_{\epsilon, t}^{(l)\top}\right)\right)\mV_{\epsilon, t}^{(l)}, \forall \mV_{\epsilon, t}^{(l)}\in \cV^{(l)}_{\epsilon, t}, \mSigma_{\epsilon, t}\in \cN_{\epsilon}\right\}.
        \end{aligned}
    \end{equation}
    Moreover, we set $\cV^{(l)}_{\epsilon, 0}=\{\mV^{\star}\}$. 
    Then, we denote the best approximation of $\mV_t^{(l)}$ in $\cV^{(l)}_{\epsilon, t}$ as $\mV_{\epsilon, t}^{(l)}=\argmin_{\mV\in \cV^{(l)}_{\epsilon, t}}\norm{\mV_t^{(l)}-\mV}_{\fro}$. Hence, we have the following decomposition
    \begin{equation}
        \begin{aligned}
            &\norm{\left(\projection-\cR_{\Omega^{(l)}}\right)\left(\mX^{\star}-\mV_t^{(l)}\mSigma_t\mV_t^{(l)\top}\right)\mV_t^{(l)}}_{\fro}\\
            &\leq \underbrace{\norm{\left(\projection-\cR_{\Omega^{(l)}}\right)\left(\mX^{\star}-\mV_{\epsilon, t}^{(l)}\mSigma_{\epsilon, t}\mV_{\epsilon, t}^{(l)\top}\right)\mV_{\epsilon, t}^{(l)}}_{\fro}}_{(\rom{1})}\\
            &\quad +\underbrace{\norm{\left(\projection-\cR_{\Omega^{(l)}}\right)\left(\mX^{\star}-\mV_t^{(l)}\mSigma_t\mV_t^{(l)\top}\right)\left(\mV_t^{(l)}-\mV_{\epsilon, t}^{(l)}\right)}_{\fro}}_{(\rom{2})}\\
            &\quad+ \underbrace{\norm{\left(\projection-\cR_{\Omega^{(l)}}\right)\left(\mV_{t}^{(l)}\mSigma_{t}\left(\mV_{t}^{(l)}-\mV_{\epsilon, t}^{(l)}\right)^{\top}\right)\mV_{\epsilon, t}^{(l)}}_{\fro}}_{(\rom{3})}\\
            &\quad +\underbrace{\norm{\left(\projection-\cR_{\Omega^{(l)}}\right)\left(\mV_{t}^{(l)}\left(\mSigma_{t}-\mSigma_{\epsilon, t}\right)\mV_{\epsilon, t}^{(l)\top}\right)\mV_{\epsilon, t}^{(l)}}_{\fro}}_{(\rom{4})}\\
            &\quad +\underbrace{\norm{\left(\projection-\cR_{\Omega^{(l)}}\right)\left(\left(\mV_{t}^{(l)}-\mV_{\epsilon, t}^{(l)}\right)\mSigma_{\epsilon, t}\mV_{\epsilon, t}^{(l)\top}\right)\mV_{\epsilon, t}^{(l)}}_{\fro}}_{(\rom{5})}.
        \end{aligned}
    \end{equation} 
    To proceed, we require the following key lemmas.
    \begin{lemma}
        \label{lem::concentration-fro-norm-difference}
        For any $1\leq i\leq j\leq d$, let $p_{i,j}$ be a Bernoulli random variable taking the value $p_{i,j}=1$ if and only if $(i,j)\in \Omega$. Let $r_{i,j} = \frac{1}{2p}(p_{i,j}+p_{j,i})$. The following statements hold: 
        \begin{itemize}
            \item {\bf Indepndent case:} Suppose the sampling rate $p\gtrsim \frac{\mu r}{d}\log\left(\frac{4r}{\delta}\right)$. Suppose that $\mX\in \cS_{d\times d}$ and $\mV\in \cO_{d\times r}$ with $\norm{\mV}_{2,\infty}\leq 2\sqrt{\frac{\mu r}{d}}$ are independent of $r_{l, 1}, \ldots, r_{l, d}$. Then, with probability at least $1-\delta$, we have
            \begin{equation}
                \norm{\left(\projection-\cR_{\Omega^{(l)}}\right)\left(\mX\right)\mV}_{\fro}^2\leq \frac{32\mu r\log(4r/\delta)}{p}\norm{\mX}_{\max}^2.
            \end{equation}
            \item {\bf General case:} Suppose the sampling rate $p\gtrsim \frac{\log(d)}{d}$. For arbitrary $\mX\in \bR^{d\times d}$ and $\mV\in \bR^{d\times r}$, with probability at least $1-\frac{1}{d^3}$, we have 
            \begin{equation}
                \norm{\left(\projection-\cR_{\Omega^{(l)}}\right)\left(\mX\right)\mV}_{\fro}^2\leq \frac{2d}{p}\norm{\mX}_{\max}^{2}\norm{\mV}_{\fro}^2.
            \end{equation}
        \end{itemize}
    \end{lemma}
    \begin{lemma}
        \label{prop::concentration}
        Suppose the sampling rate satisfies $p\gtrsim \frac{\mu r}{d}\log\left(\frac{4r}{\delta}\right)$. For any $0\leq t\leq T=\frac{100}{\eta\sigma_{r}^{\star}}\log\left(\frac{1}{\alpha}\right), 1\leq l\leq d$, with probability at least $1-\frac{1}{d^3}$, we have
        \begin{equation}
            \norm{\left(\projection-\cR_{\Omega^{(l)}}\right)\left(\mX^{\star}-\mV_{\epsilon, t}^{(l)}\mSigma_{\epsilon, t}\mV_{\epsilon, t}^{(l)\top}\right)\mV_{\epsilon, t}^{(l)}}_{\fro}^2\leq \const\frac{\kappa\mu^3r^{5.5}\log\left(\frac{1}{\alpha}\right)\log\left(\frac{r\sigma_{1}^{\star}}{\epsilon}\right)}{\sqrt{pd^3}}.
        \end{equation}
    \end{lemma}
    \begin{lemma}
        \label{prop::incoherence-dynamic-2}
        Under the same conditions as Lemma~\ref{prop::concentration}, we have 
        \begin{equation}
            \norm{\mV_{t+1}^{(l)}-\mV_{\epsilon, t+1}^{(l)}}_{\fro}\leq \left(1+ \const\eta \sigma_1^{\star}\frac{\kappa\mu r^{1.5}\log\left(\frac{1}{\alpha}\right)}{\sqrt{pd}}\right)\norm{\mV_{t}^{(l)}-\mV_{\epsilon, t}^{(l)}}_{\fro}+ C_3\eta \frac{\sigma^{\star}_{1}\mu r}{\sqrt{pd}}\epsilon.
        \end{equation}
    \end{lemma}
    The proof of these lemmas is deferred to the end of this section. Now we use Lemma~\ref{prop::incoherence-dynamic-2} to control the dynamic of $\norm{\mV_{t}^{(l)}-\mV_{\epsilon, t}^{(l)}}_{\fro}$. Note that $\norm{\mV_{0}^{(l)}-\mV_{\epsilon, 0}^{(l)}}_{\fro}=0$. Hence, applying Lemma~\ref{lem::appendix-series-ineq}, we have 
    \begin{equation}
        \begin{aligned}
            \norm{\mV_{t}^{(l)}-\mV_{\epsilon, t}^{(l)}}_{\fro}&\leq C\left(\left(1+ \const\eta \sigma_1^{\star}\frac{\kappa\mu r^{1.5}\log\left(\frac{1}{\alpha}\right)}{\sqrt{pd}}\right)^t-1\right)\frac{\eta \frac{\sigma^{\star}_{1}\mu r}{\sqrt{pd}}\epsilon}{\eta \sigma_1^{\star}\frac{\kappa\mu r^{1.5}\log\left(\frac{1}{\alpha}\right)}{\sqrt{pd}}}\\
            &\leq C\eta \frac{\sigma^{\star}_{1}\mu r}{\sqrt{pd}}\epsilon\cdot t.
        \end{aligned}
    \end{equation}
    Now we are ready to control $(\rom{1})$ to $(\rom{5})$ separately. First, Lemma~\ref{prop::concentration} directly implies that 
    \begin{equation}
        (\rom{1})\leq \const\frac{\sqrt{\kappa\mu^3r^{5.5}\log\left(\frac{1}{\alpha}\right)\log\left(\frac{r\sigma_{1}^{\star}}{\epsilon}\right)}}{\sqrt[4]{pd^3}}.
    \end{equation}
    On the other hand, applying Lemma~\ref{lem::concentration-fro-norm-difference} to $(\rom{2})$ leads to
    \begin{equation}
        \begin{aligned}
            (\rom{2})&\leq \sqrt{\frac{2d}{p}}\norm{\mX^{\star}-\mV_t^{(l)}\mSigma_t\mV_t^{(l)\top}}_{\max}\norm{\mV_t^{(l)}-\mV_{\epsilon, t}^{(l)}}_{\fro}\leq \const\sqrt{\frac{\mu^2r^2}{pd}}\norm{\mV_t^{(l)}-\mV_{\epsilon, t}^{(l)}}_{\fro}.
        \end{aligned}
    \end{equation}
    Similarly, we have
    \begin{equation}
        \begin{aligned}
            (\rom{3})&\leq \sqrt{\frac{2d}{p}}\norm{\mV_{t}^{(l)}\mSigma_{t}\left(\mV_{t}^{(l)}-\mV_{\epsilon, t}^{(l)}\right)^{\top}}_{\max}\norm{\mV_{\epsilon, t}^{(l)}}_{\fro}\leq \const\sqrt{\frac{\mu r^2}{p}}\norm{\mV_t^{(l)}-\mV_{\epsilon, t}^{(l)}}_{\fro},\\
            (\rom{4})&\leq \sqrt{\frac{2d}{p}}\norm{\mV_{t}^{(l)}\left(\mSigma_{t}-\mSigma_{\epsilon, t}\right)\mV_{\epsilon, t}^{(l)\top}}_{\max}\norm{\mV_{\epsilon, t}^{(l)}}_{\fro}\leq \const\sqrt{\frac{\mu^2r^3}{pd}}\epsilon,\\
            (\rom{5})&\leq \sqrt{\frac{2d}{p}}\norm{\left(\mV_{t}^{(l)}-\mV_{\epsilon, t}^{(l)}\right)\mSigma_{\epsilon, t}\mV_{\epsilon, t}^{(l)\top}}_{\max}\norm{\mV_{\epsilon, t}^{(l)}}_{\fro}\leq \const\sqrt{\frac{\mu r^2}{p}}\norm{\mV_t^{(l)}-\mV_{\epsilon, t}^{(l)}}_{\fro}.
        \end{aligned}
    \end{equation}
    Overall, we derive that
    \begin{equation}
        \begin{aligned}
            &\norm{\left(\projection-\cR_{\Omega^{(l)}}\right)\left(\mX^{\star}-\mV_t^{(l)}\mSigma_t\mV_t^{(l)\top}\right)\mV_t^{(l)}}_{\fro}\\
            &\leq \const\frac{\sqrt{\kappa\mu^3r^{5.5}\log\left(\frac{1}{\alpha}\right)\log\left(\frac{r\sigma_{1}^{\star}}{\epsilon}\right)}}{\sqrt[4]{pd^3}} + \const\sqrt{\frac{\mu r^2}{p}}\norm{\mV_t^{(l)}-\mV_{\epsilon, t}^{(l)}}_{\fro}\\
            &\leq \const\frac{\sqrt{\kappa\mu^3r^{5.5}\log\left(\frac{1}{\alpha}\right)\log\left(\frac{r\sigma_{1}^{\star}}{\epsilon}\right)}}{\sqrt[4]{pd^3}} + \const\sqrt{\frac{\mu r^2}{p}}\cdot \eta \frac{\sigma^{\star}_{1}\mu r}{\sqrt{pd}}\epsilon\cdot t\\
            &\leq 2\const\sigma_1^{\star}\frac{\sqrt{\kappa\mu^3r^{5.5}\log\left(\frac{1}{\alpha}\right)\log\left(d\right)}}{\sqrt[4]{pd^3}}.
        \end{aligned}
    \end{equation}
    The last inequality follows from the fact that $\epsilon=\frac{c}{d}$. This completes the proof of Proposition~\ref{prop::incoherence-dynamic-1}.$\hfill\blacksquare$

    Next, we proceed to present the proofs of Lemma~\ref{lem::concentration-fro-norm-difference}, Lemma~\ref{prop::concentration}, and Lemma~\ref{prop::incoherence-dynamic-2}.
\begin{proof}\linkofproof{Lemma~\ref{lem::concentration-fro-norm-difference}}
        We first expand $\norm{\left(\projection-\cR_{\Omega^{(l)}}\right)\left(\mX\right)\mV}_{\fro}^2$ as follows
        \begin{equation}
            \begin{aligned}
                \norm{\left(\projection-\cR_{\Omega^{(l)}}\right)\left(\mX\right)\mV}_{\fro}^2=\sum_{j=1}^{r}\left(\sum_{k=1}^{d}\left(r_{l, k}-1\right)\emX_{l, k}\emV_{k, j}\right)^2+\sum_{i\neq l}\left(r_{i, l}-1\right)^2\emX_{i, l}^{2}\norm{\mV_{l,\cdot}}^2.
            \end{aligned}
        \end{equation}
        Next, we prove these two cases separately.
    
        \paragraph{Independent case.} First, we control $\sum_{k=1}^{d}\left(r_{l, k}-1\right)\emX_{l, k}\emV_{k, j}$ for all $1\leq j\leq r$ via Bernstein's inequality (Lemma~\ref{lem::bernstein}). To this goal, for a fixed $1\leq j\leq r$, upon defining $Z_{k}=\left(r_{l, k}-1\right)\emX_{l, k}\emV_{k, j}$ with $\bE[Z_k]=0$, we have
        \begin{equation}
            \begin{aligned}
                &M:=\max_{k}|Z_k|\leq \frac{1}{p}\norm{\mX}_{\max}\norm{\mV}_{2, \infty},\\
                &\nu^2:=\sum_{k=1}^{d}\Var\left[Z_k^2\right]\leq \frac{1}{p}\sum_{k=1}^{d}\emX_{l, k}^{2}\emV_{k, j}^2\leq \frac{1}{p}\norm{\mX}_{\max}^2\sum_{k=1}^{d}\emV_{k, j}^2=\frac{1}{p}\norm{\mX}_{\max}^2.
            \end{aligned}
        \end{equation}
        Here in the last equality we use the fact that $\sum_{k=1}^{d}\emV_{k, j}^2=1$ since $\mV\in \cO_{d\times r}$.
        Therefore, due to Bernstein's inequality, with probability at least $1-\frac{\delta}{2r}$, one has 
        \begin{equation}
            \begin{aligned}
                \left|\sum_{k=1}^{d}Z_k\right|&\leq 2\nu\sqrt{\log\left(\frac{4r}{\delta}\right)}+\frac{4}{3}M\log\left(\frac{4r}{\delta}\right)\\
                &\leq \frac{2\norm{\mX}_{\max}}{\sqrt{p}}\sqrt{\log\left(\frac{4r}{\delta}\right)}+\frac{4\norm{\mX}_{\max}\norm{\mV}_{2, \infty}}{3p}\log\left(\frac{4r}{\delta}\right)\\
                &\leq \frac{4\norm{\mX}_{\max}}{\sqrt{p}}\sqrt{\log\left(\frac{4r}{\delta}\right)},
            \end{aligned}
        \end{equation}
        where we use the assumption $p\gtrsim \frac{\mu r}{d}\log\left(\frac{4r}{\delta}\right)$.
        Hence, via a union bound, we know that with probability at least $1-\frac{\delta}{2}$, we have 
        \begin{equation}
            \sum_{j=1}^{r}\left(\sum_{k=1}^{d}\left(r_{l, k}-1\right)\emX_{l, k}\emV_{k, j}\right)^2\leq \frac{16r\norm{\mX}_{\max}^2}{p}\log\left(\frac{4r}{\delta}\right).
        \end{equation}
        Next, we can control $\sum_{i\neq l}\left(r_{i, l}-1\right)^2\emX_{i, l}^{2}\norm{\mV_{l,\cdot}}^2$ as 
        \begin{equation}
            \begin{aligned}
                \sum_{i\neq l}\left(r_{i, l}-1\right)^2\emX_{i, l}^{2}\norm{\mV_{l,\cdot}}^2&\leq \sum_{i=1}^{d}\left(r_{i, l}-1\right)^2\norm{\mX}_{\max}^{2}\norm{\mV}_{2, \infty}^2.
            \end{aligned}
        \end{equation}
        Then, we apply Bernstein's inequality to control $\sum_{k=1}^{d}\left(r_{l, k}-1\right)^2$. To this end, notice that 
        \begin{equation}
            \begin{aligned}
                &\bE\left[\left(r_{l, k}-1\right)^2\right]=\frac{1-p}{2p}\leq \frac{1}{2p},\\
                &M:=\max_{k}\left(r_{l, k}-1\right)^2\leq \left(\frac{1}{p}-1\right)^2\leq \frac{1}{p^2},\\
                &\nu^2:=\sum_{k=1}^{d}\Var\left[\left(r_{l, k}-1\right)^2\right]\leq \frac{d}{p^3}.
            \end{aligned}
        \end{equation}
        Therefore, with probability at least $1-\frac{\delta}{2}$, we have 
        \begin{equation}
            \begin{aligned}
                \sum_{k=1}^{d}\left(r_{l, k}-1\right)^2&\leq \frac{d}{2p}+\frac{4}{3}\frac{1}{p^2}\log\left(\frac{2}{\delta}\right)+2\sqrt{\frac{d}{p^3}}\sqrt{\log\left(\frac{2}{\delta}\right)}\leq \frac{d}{p}
            \end{aligned}
        \end{equation}
        since we set $p\gtrsim \frac{\mu r}{d}\log\left(\frac{4r}{\delta}\right)$. This implies that with probability at least $1-\frac{\delta}{2}$, we have
        \begin{equation}
            \sum_{i\neq l}\left(r_{i, l}-1\right)^2\emX_{i, l}^{2}\norm{\mV_{l,\cdot}}^2\leq \frac{d}{p}\norm{\mX}_{\max}^{2}\norm{\mV}_{2, \infty}^2.
        \end{equation}
        Finally, taking a union bound, we conclude that with probability at least $1-\delta$, we have
        \begin{equation}
            \begin{aligned}
                \norm{\left(\projection-\cR_{\Omega^{(l)}}\right)\left(\mX\right)\mV}_{\fro}^2&\leq \frac{16r\norm{\mX}_{\max}^2\log(4r/\delta)}{p}+\frac{d}{p}\norm{\mX}_{\max}^{2}\norm{\mV}_{2, \infty}^2\\
                &\leq \frac{32\mu r\log(4r/\delta)}{p}\norm{\mX}_{\max}^2.
            \end{aligned}
            \label{eq::116}
        \end{equation}
    
        \paragraph{General case.} First, we apply Cauchy-Schwarz inequality to obtain
        \begin{equation}
            \begin{aligned}
                \sum_{j=1}^{r}\left(\sum_{k=1}^{d}\left(r_{l, k}-1\right)\emX_{l, k}\emV_{k, j}\right)^2&\leq \sum_{j=1}^{r}\left(\sum_{k=1}^{d}\left(r_{l, k}-1\right)^2\right)\left(\sum_{k=1}^{d}\emX_{l, k}^2\emV_{k, j}^2\right)\\
                &\leq \left(\sum_{k=1}^{d}\left(r_{l, k}-1\right)^2\right)\cdot \sum_{j=1}^{r}\norm{\mX}^2_{\max}\sum_{k=1}^{d}\emV_{k, j}^2\\
                &=\left(\sum_{k=1}^{d}\left(r_{l, k}-1\right)^2\right)\cdot \norm{\mX}^2_{\max}\norm{\mV}_{\fro}^2.
            \end{aligned}
        \end{equation}
        Next, we can control $\sum_{i\neq l}\left(r_{i, l}-1\right)^2\emX_{i, l}^{2}\norm{\mV_{l,\cdot}}^2$ as follows 
        \begin{equation}
            \begin{aligned}
                \sum_{i\neq l}\left(r_{i, l}-1\right)^2\emX_{i, l}^{2}\norm{\mV_{l,\cdot}}^2&\leq \sum_{i=1}^{d}\left(r_{i, l}-1\right)^2\norm{\mX}_{\max}^{2}\norm{\mV}_{\fro}^2\\
                &=\left(\sum_{k=1}^{d}\left(r_{l, k}-1\right)^2\right)\norm{\mX}_{\max}^{2}\norm{\mV}_{\fro}^2.
            \end{aligned}
        \end{equation}
        Here we use the fact that $r_{i, j}=r_{j, i}, \forall i, j\in [d]$. Next, we apply Bernstein's inequality to control $\sum_{k=1}^{d}\left(r_{l, k}-1\right)^2$. To this end, notice that 
        \begin{equation}
            \begin{aligned}
                &\bE\left[\left(r_{l, k}-1\right)^2\right]=\frac{1-p}{2p}\leq \frac{1}{2p},\\
                &M:=\max_{k}\left(r_{l, k}-1\right)^2\leq \left(\frac{1}{p}-1\right)^2\leq \frac{1}{p^2},\\
                &\nu^2:=\sum_{k=1}^{d}\Var\left[\left(r_{l, k}-1\right)^2\right]\leq \frac{d}{p^3}.
            \end{aligned}
        \end{equation}
        Therefore, with probability at least $1-\delta$, we have 
        \begin{equation}
            \begin{aligned}
                \sum_{k=1}^{d}\left(r_{l, k}-1\right)^2&\leq \frac{d}{2p}+\frac{4}{3}\frac{1}{p^2}\log\left(\frac{2}{\delta}\right)+2\sqrt{\frac{d}{p^3}}\sqrt{\log\left(\frac{2}{\delta}\right)}.
            \end{aligned}
        \end{equation}
        Specifically, upon setting $\delta=\frac{1}{d^3}$, we obtain that with probability at least $1-\frac{1}{d^3}$, one has
        \begin{equation}
            \sum_{k=1}^{d}\left(r_{l, k}-1\right)^2\leq \frac{d}{p},
        \end{equation}
        since we set $p\gtrsim \frac{\log(d)}{d}$. Overall, we have
        \begin{equation}
            \norm{\left(\projection-\cR_{\Omega^{(l)}}\right)\left(\mX\right)\mV}_{\fro}^2\leq 2\left(\sum_{k=1}^{d}\left(r_{l, k}-1\right)^2\right)\norm{\mX}_{\max}^{2}\norm{\mV}_{\fro}^2\leq \frac{2d}{p}\norm{\mX}_{\max}^{2}\norm{\mV}_{\fro}^2,
        \end{equation}
        with probability at least $1-\frac{1}{d^3}$. This completes the proof. 
    \end{proof}\begin{proof}\linkofproof{Lemma~\ref{prop::concentration}}
        First, for fixed $\mSigma_{\epsilon, t}\in \cN_{\epsilon}$ and $\mV_{\epsilon, t}^{(l)}\in \cV_{\epsilon, t}^{(l)}$, Lemma~\ref{lem::concentration-fro-norm-difference} implies that with probability at least $1-\delta$, we have
        \begin{equation}
            \norm{\left(\projection-\cR_{\Omega^{(l)}}\right)\left(\mX^{\star}-\mV_{\epsilon, t}^{(l)}\mSigma_{\epsilon, t}\mV_{\epsilon, t}^{(l)\top}\right)\mV_{\epsilon, t}^{(l)}}_{\fro}^2\leq \frac{32\mu r\log(4r/\delta)}{p}\norm{\mX^{\star}-\mV_{\epsilon, t}^{(l)}\mSigma_{\epsilon, t}\mV_{\epsilon, t}^{(l)\top}}_{\max}^2.
        \end{equation}
        Note that 
        \begin{equation}
            \begin{aligned}
                \norm{\mX^{\star}-\mV_{\epsilon, t}^{(l)}\mSigma_{\epsilon, t}\mV_{\epsilon, t}^{(l)\top}}_{\max}&\leq \norm{\mX^{\star}}_{\max}+\norm{\mV_{\epsilon, t}^{(l)}\mSigma_{\epsilon, t}\mV_{\epsilon, t}^{(l)\top}}_{\max}\\
                &\stackrel{(a)}{\leq}\norm{\mSigma^{\star}}\norm{\mV^{\star}}^2_{2, \infty}+\norm{\mSigma_{\epsilon, t}}\norm{\mV_{\epsilon, t}^{(l)}}_{2, \infty}^2\\
                &\leq 9\sigma_{1}^{\star}\frac{\mu r}{d}.
            \end{aligned}
        \end{equation}
        Therefore, with probability at least $1-\delta$, we have
        \begin{equation}
            \norm{\left(\projection-\cR_{\Omega^{(l)}}\right)\left(\mX^{\star}-\mV_{\epsilon, t}^{(l)}\mSigma_{\epsilon, t}\mV_{\epsilon, t}^{(l)\top}\right)\mV_{\epsilon, t}^{(l)}}_{\fro}^2\leq \frac{288\mu^3 r^3\log(4r/\delta)}{pd^2}.
        \end{equation}
        Lastly, we apply the union bound to finalize the desired result. To this end, note that $\cB^{r\times r}_{\mathrm{op}}(4\sigma_{1}^{\star})\subset \cB^{r\times r}_{\fro}(4\sqrt{r}\sigma_{1}^{\star})$. Hence, according to Lemma~\ref{lem::covering-number}, we know that $|\cN_{\epsilon}|\leq \left(\frac{12\sqrt{r}\sigma^{\star}_{1}}{\epsilon}\right)^{r^2}$which implies that $\left|\cV_{\epsilon, t}^{(l)}\right|\leq \left(\frac{6\sqrt{r}\sigma^{\star}_{1}}{\epsilon}\right)^{r^2t}$. Therefore, the total cardinality of $\bigcup_{l=1}^{d}\bigcup_{t=0}^{T}\cV_{\epsilon, t}^{(l)}$ is upper-bounded by 
        \begin{equation}
            \left|\bigcup_{l=1}^{d}\bigcup_{t=0}^{T}\cV_{\epsilon, t}^{(l)}\right|\leq d\cdot \sum_{t=0}^{T}\left(\frac{12\sqrt{r}\sigma^{\star}_{1}}{\epsilon}\right)^{r^2t}\leq 2d\left(\frac{12\sqrt{r}\sigma^{\star}_{1}}{\epsilon}\right)^{r^2T}.
        \end{equation}
        Hence, once we set $\delta=\frac{1}{2d^4}\left(\frac{6\sqrt{r}\sigma^{\star}_{1}}{\epsilon}\right)^{-r^2T}$, we obtain that with probability at least $1-d^{-3}$, for any $0\leq t\leq T, 1\leq l\leq d$ and $\mV_{\epsilon, t}^{(l)}\in \cV_{\epsilon, t}^{(l)}$,
        \begin{equation}
            \begin{aligned}
                \norm{\left(\projection-\cR_{\Omega^{(l)}}\right)\left(\mX^{\star}-\mV_{\epsilon, t}^{(l)}\mSigma_{\epsilon, t}\mV_{\epsilon, t}^{(l)\top}\right)\mV_{\epsilon, t}^{(l)}}_{\fro}^2&\leq \const\frac{\mu^3 r^3}{pd^2}\left(\log(d)+r^2T\left(\frac{r\sigma_{1}^{\star}}{\epsilon}\right)\right)\\
                &\leq \const\frac{\kappa\mu^3r^{5.5}\log\left(\frac{1}{\alpha}\right)\log\left(\frac{r\sigma_{1}^{\star}}{\epsilon}\right)}{\sqrt{pd^3}}.
            \end{aligned}
        \end{equation}
        Here we use the fact that $T\leq \frac{100}{\eta\sigma^{\star}_{r}}\log(1/\alpha)$ and $\eta=\Theta\left(\frac{\mu}{\sigma^{\star}_{1}}\sqrt{\frac{r}{pd}}\right)$. This completes the proof.
\end{proof}
\begin{proof}\linkofproof{Lemma~\ref{prop::incoherence-dynamic-2}}
    The derivation of $\norm{\mV_{t+1}^{(l)}-\mV_{\epsilon, t+1}^{(l)}}_{\fro}$ is nearly the same as that of $\norm{\mV_{t+1}-\mV_{t+1}^{(l)}}_{\fro}$. First, note that
\begin{equation}
\mV_{\epsilon, t+1}^{(l)}=\left(\mI+\eta \proj^{\perp}_{\mV_{\epsilon, t}^{(l)}}\mM_{\epsilon, t}^{(l)}\right)\mV_{\epsilon, t}^{(l)}+\mA^{(l)}_{\epsilon, t},
\end{equation}
where $\mM_{\epsilon, t}^{(l)}$ and $\mA_{\epsilon, t}^{(l)}$ are defined similar to $\mM_{t}^{(l)}$ and $\mA_{t}^{(l)}$.
Hence, we can expand $\norm{\mV_{t+1}^{(l)}-\mV_{\epsilon, t+1}^{(l)}}_{\fro}^2$ as
\begin{equation}
\begin{aligned}
    \norm{\mV_{t+1}^{(l)}\!-\!\mV_{\epsilon, t+1}^{(l)}}_{\fro}^2&=\norm{\mV_{t}^{(l)}\!-\!\mV_{\epsilon, t}^{(l)}}_{\fro}^2+2\eta\underbrace{\inner{\mV_{t}^{(l)}\!-\!\mV_{\epsilon, t}^{(l)}}{\proj_{\mV_t^{(l)}}^{\perp}\mM_t^{(l)}\mV_t^{(l)}\!-\!\proj_{\mV_{\epsilon, t}^{(l)}}^{\perp}\mM_{\epsilon, t}^{(l)}\mV_{\epsilon, t}^{(l)}}}_{:=(\rom{1})}\!+(\rom{2}).
\end{aligned}
\end{equation}
where
    \begin{equation}
        \begin{aligned}
            (\rom{2})&=\eta^2\norm{\proj_{\mV_t^{(l)}}^{\perp}\mM_t^{(l)}\mV_t^{(l)}-\proj_{\mV_{\epsilon, t}^{(l)}}^{\perp}\mM_{\epsilon, t}^{(l)}\mV_{\epsilon, t}^{(l)}}_{\fro}^2+\norm{\mA_t^{(l)}-\mA_{\epsilon, t}^{(l)}}_{\fro}^2\\
            &\quad + 2\inner{\left(\mI +\eta\proj_{\mV_t^{(l)}}^{\perp} \mM_t^{(l)}\right)\mV_t^{(l)}-\left(\mI +\eta\proj_{\mV_{\epsilon, t}^{(l)}}^{\perp} \mM_{\epsilon, t}^{(l)}\right)\mV_{\epsilon, t}^{(l)}}{\mA_t^{(l)}-\mA_{\epsilon, t}^{(l)}}
        \end{aligned}
    \end{equation}
    contains all the higher-order terms.
Next, we further decompose $(\rom{1})$ as follows
\begin{equation}
\begin{aligned}
    (\rom{1})&=\inner{\mV_{t}^{(l)}-\mV_{\epsilon, t}^{(l)}}{\proj_{\mV_t^{(l)}}^{\perp}\mXi_t^{(l)}\mV_t^{(l)}-\proj_{\mV_{\epsilon, t}^{(l)}}^{\perp}\mXi_{\epsilon, t}^{(l)}\mV_{\epsilon, t}^{(l)}}\\
    &= \underbrace{\inner{\mV_{t}^{(l)}-\mV_{\epsilon, t}^{(l)}}{\left(\proj_{\mV_t^{(l)}}^{\perp}-\proj_{\mV_{\epsilon, t}^{(l)}}^{\perp}\right)\mXi_{\epsilon, t}^{(l)}\mV_{\epsilon, t}^{(l)}}}_{:=(\rom{1}_1)}+\underbrace{\inner{\mV_{t}^{(l)}-\mV_{\epsilon, t}^{(l)}}{\proj_{\mV_t^{(l)}}^{\perp}\mXi_t^{(l)}\left(\mV_{t}^{(l)}-\mV_{\epsilon, t}^{(l)}\right)}}_{:=(\rom{1}_2)}\\
    &\qquad+\underbrace{\inner{\mV_{t}^{(l)}-\mV_{\epsilon, t}^{(l)}}{\proj_{\mV_t^{(l)}}^{\perp}\left(\mXi_t^{(l)}-\mXi_{\epsilon, t}^{(l)}\right)\mV_{\epsilon, t}^{(l)}}}_{:=(\rom{1}_3)}.
\end{aligned}
\end{equation}
Here we define $\mXi_t^{(l)}=\mM_t^{(l)}-\mV^{(l)}_t\mSigma^{\star}\mV^{\star\top}+\mV^{(l)}_t\mSigma_t\mV_t^{(l)\top}$ and $\mXi_{\epsilon, t}^{(l)}=\mM_{\epsilon, t}^{(l)}-\mV^{(l)}_{\epsilon, t}\mSigma^{\star}\mV^{\star\top}+\mV^{(l)}_{\epsilon, t}\mSigma_{\epsilon, t}\mV_{\epsilon, t}^{(l)\top}$. We control $(\rom{1}_1)$ as follows 
\begin{equation}\label{eq_app_I1}
(\rom{1}_1)\leq \norm{\mV_{t}^{(l)}-\mV_{\epsilon, t}^{(l)}}_{\fro}\norm{\proj_{\mV_t^{(l)}}^{\perp}-\proj_{\mV_{\epsilon, t}^{(l)}}^{\perp}}_{\fro}\norm{\mXi_{\epsilon, t}^{(l)}}\leq 2\norm{\mV_{t}^{(l)}-\mV_{\epsilon, t}^{(l)}}_{\fro}^2\norm{\mXi_{\epsilon, t}^{(l)}}.
\end{equation}
Here we apply Lemma~\ref{lem::orthogonal-matrix} in the second inequality. For $\norm{\mXi_{\epsilon, t}^{(l)}}$, following the same analysis as in \Cref{eqn::xi-t}, we have
\begin{equation}
\begin{aligned}
    \norm{\mXi_{\epsilon, t}^{(l)}}&\leq \sigma^{\star}_{1}\norm{\mV^{\star}-\mV_{\epsilon, t}^{(l)}}+10  \const\frac{\sigma^{\star}_{1}\mu r}{\sqrt{pd}}\\
    &\leq \sigma^{\star}_{1}\left(\norm{\mV^{\star}-\mV_{t}^{(l)}}+\norm{\mV_t^{(l)}-\mV_{\epsilon, t}^{(l)}}\right)+10  \const\frac{\sigma^{\star}_{1}\mu r}{\sqrt{pd}}\\
    &\leq \sigma^{\star}_{1}\norm{\mV^{\star}-\mV_{t}^{(l)}}+20  \const\frac{\sigma^{\star}_{1}\mu r}{\sqrt{pd}}.
\end{aligned}
\end{equation}
Here in the last inequality, we use the fact that $\norm{\mV_t^{(l)}-\mV_{\epsilon, t}^{(l)}}\leq \epsilon\leq 10  \const\frac{\mu r}{\sqrt{pd}}$. Plugging this inequality into~\Cref{eq_app_I1}, we obtain
\begin{equation}
(\rom{1}_1)\leq 2\sigma^{\star}_{1}\norm{\mV_{t}^{(l)}-\mV_{\epsilon, t}^{(l)}}_{\fro}^2\left(\sigma^{\star}_{1}\norm{\mV^{\star}-\mV_t^{(l)}}+20  \const\frac{\sigma^{\star}_{1}\mu r}{\sqrt{pd}}\right).
\end{equation}
Similar to our derivation for $(\rom{1}_1)$, it follows that $(\rom{1}_2)\leq\sigma^{\star}_{1}\norm{\mV^{\star}-\mV_t^{(l)}}_{\fro}^2\Big(\sigma^{\star}_{1}\norm{\mV^{\star}-\mV_t^{(l)}}+10  \const\frac{\sigma^{\star}_{1}\mu r}{\sqrt{pd}}\Big)$.
Lastly, to bound $(\rom{1}_3)$, we further decompose it as
\begin{equation}
\begin{aligned}
    (\rom{1}_3)&=\inner{\mV_{t}^{(l)}-\mV_{\epsilon, t}^{(l)}}{\proj_{\mV_t^{(l)}}^{\perp}\left(\mM_t^{(l)}-\mM_{\epsilon, t}^{(l)}+\mV^{(l)}_{\epsilon, t}\mSigma^{\star}\mV^{\star\top}-\mV^{(l)}_{\epsilon, t}\mSigma_t^{(l)}\mV_{\epsilon, t}^{(l)\top}\right)\mV_{\epsilon, t}^{(l)}}\\
    &=\underbrace{\inner{\mV_{t}^{(l)}-\mV_{\epsilon, t}^{(l)}}{\proj_{\mV_t^{(l)}}^{\perp}\left(\id-\cR_{\Omega^{(l)}}\right)\left(\mV_{\epsilon, t}^{(l)}\mSigma_{\epsilon, t}^{(l)}\left(\mV_{\epsilon, t}^{(l)}-\mV_{t}^{(l)}\right)^{\top}\right)\mV_{\epsilon, t}^{(l)}}}_{:=(\rom{1}_{3, 1})}\\
    &\qquad + \underbrace{\inner{\mV_{t}^{(l)}-\mV_{\epsilon, t}^{(l)}}{\proj_{\mV_t^{(l)}}^{\perp}\left(\id-\cR_{\Omega^{(l)}}\right)\left(\left(\mV_{\epsilon, t}^{(l)}-\mV_{t}^{(l)}\right)\mSigma_{\epsilon, t}^{(l)}\mV_{t}^{(l)\top}\right)\mV_{\epsilon, t}^{(l)}}}_{:=(\rom{1}_{3, 2})}\\
    &\qquad + \underbrace{\inner{\mV_{t}^{(l)}-\mV_{\epsilon, t}^{(l)}}{\proj_{\mV_t^{(l)}}^{\perp}\left(\id-\cR_{\Omega^{(l)}}\right)\left(\mV_t^{(l)}\left(\mSigma_{\epsilon, t}-\mSigma_{t}\right)\mV_t^{(l)\top}\right)\mV_{\epsilon, t}^{(l)}}}_{:=(\rom{1}_{3, 3})}\\
    &\qquad +\underbrace{\inner{\mV_{t}^{(l)}-\mV_{\epsilon, t}^{(l)}}{\proj_{\mV_t^{(l)}}^{\perp}\mV_{\epsilon, t}^{(l)}\mSigma^{\star}\mV^{\star\top}\mV_{\epsilon, t}^{(l)}}}_{:=(\rom{1}_{3, 4})}.
\end{aligned}
\end{equation}
Next, we control these terms separately. First, we apply Lemma~\ref{lem::concentration-chen} to control $(\rom{1}_{3, 1})$. Specifically, upon setting $\mA=\mV_{\epsilon, t}^{(l)}, \mB=\proj_{\mV_t^{(l)}}^{\perp}\left(\mV_{t}^{(l)}-\mV_{\epsilon, t}^{(l)}\right), \mC=\left(\mV_{\epsilon, t}^{(l)}-\mV_t^{(l)}\right)\mSigma_{\epsilon, t}, \mD=\mV_{\epsilon, t}^{(l)}$ in Lemma~\ref{lem::concentration-chen}, with probability at least $1-d^{-3}$, we have 
\begin{equation}
\begin{aligned}
    (\rom{1}_{3, 1})&=\inner{\left(\id-\projection\right)\left(\mA\mC^{\top}\right)}{\mB\mD^{\top}}\\
    &\leq  \const\sqrt{\frac{d}{p}}\norm{\mA}_{2, \infty}\norm{\mB}_{\fro}\cdot \norm{\mC}_{\fro}\norm{\mD}_{2, \infty}\\
    &= \const\sqrt{\frac{d}{p}}\norm{\mV_{\epsilon, t}^{(l)}}_{2, \infty}^2\norm{\proj_{\mV_t^{(l)}}^{\perp}\left(\mV_{t}^{(l)}-\mV_{\epsilon, t}^{(l)}\right)}_{\fro}\norm{\left(\mV_{\epsilon, t}^{(l)}-\mV_t^{(l)}\right)\mSigma_{\epsilon, t}}_{\fro}\\
    &\leq 16 \const\sigma^{\star}_{1}\sqrt{\frac{\mu^2r^2}{pd}}\norm{\mV_{t}^{(l)}-\mV_{\epsilon, t}^{(l)}}_{\fro}^2.
\end{aligned}
\end{equation}
Here in the last inequality, we use the fact that $\norm{\mSigma_{\epsilon, t}}\leq 4\sigma^{\star}_{1}$ and $\norm{\mV_{\epsilon, t}^{(l)}}_{2,\infty}\leq 2\sqrt{\frac{\mu r}{d}}$.
In a manner akin to our derivation for $(\rom{1}_{3, 1})$, we can also show that 
\begin{equation}
\begin{aligned}
    (\rom{1}_{3, 2})&\leq 16 \const\sigma^{\star}_{1}\sqrt{\frac{\mu^2r^2}{pd}}\norm{\mV_{t}^{(l)}-\mV_{\epsilon, t}^{(l)}}_{\fro}^2\\
    (\rom{1}_{3, 3})&\leq 4 \const\sigma^{\star}_{1}\sqrt{\frac{\mu^2r^2}{pd}}\epsilon\norm{\mV_{\epsilon, t}^{(l)}-\mV_t^{(l)}}_{\fro}.
\end{aligned}
\end{equation}
Lastly, for $(\rom{1}_{3, 4})$, we observe that
\begin{equation}
\begin{aligned}
    (\rom{1}_{3, 4})&=-\inner{\mV_{\epsilon, t}^{(l)}}{\proj_{\mV_t^{(l)}}^{\perp}\mV_{\epsilon, t}^{(l)}\mSigma^{\star}\mV^{\star\top}\mV_{\epsilon, t}^{(l)}}\\
    &\leq -\inner{\mV_{\epsilon, t}^{(l)}}{\proj_{\mV_t^{(l)}}^{\perp}\mV_{\epsilon, t}^{(l)}\mSigma^{\star}}+\sigma^{\star}_{1}\norm{\mV^{\star}-\mV^{(l)}_{\epsilon, t}}_{\fro}\norm{\mV_t^{(l)}-\mV_{\epsilon, t}^{(l)}}_{\fro}^2\\
        &\stackrel{(a)}{\leq} \sigma^{\star}_{1}\norm{\mV^{\star}-\mV^{(l)}_{\epsilon, t}}_{\fro}\norm{\mV_t^{(l)}-\mV_{\epsilon, t}^{(l)}}_{\fro}^2\\
        &\leq \sigma^{\star}_{1}\left(\norm{\mV^{\star}-\mV^{(l)}_{t}}_{\fro}+\norm{\mV^{(l)}_t-\mV^{(l)}_{\epsilon, t}}_{\fro}\right)\norm{\mV_t^{(l)}-\mV_{\epsilon, t}^{(l)}}_{\fro}^2\\
        &\stackrel{(b)}{\leq} \sigma_{1}^{\star}\left(\const_1\frac{\kappa\mu r^{1.5}\log\left(\frac{1}{\alpha}\right)}{\sqrt{pd}}+\norm{\mV^{(l)}_t-\mV^{(l)}_{\epsilon, t}}_{\fro}\right)\norm{\mV_t^{(l)}-\mV_{\epsilon, t}^{(l)}}_{\fro}^2.
    \end{aligned}
\end{equation}
Here in $(a)$, we use the fact that $\inner{\mV_{\epsilon, t}^{(l)}}{\proj_{\mV_t^{(l)}}^{\perp}\mV_{\epsilon, t}^{(l)}\mSigma^{\star}}=\norm{\proj_{\mV_t^{(l)}}^{\perp}\mV_{\epsilon, t}^{(l)}\mSigma^{\star 1/2}}_{\fro}^2\geq 0$.
In $(b)$, we apply Proposition~\ref{prop::frobenius-norm-control-appendix}.

Putting everything together, we obtain that
\begin{equation}
\begin{aligned}
(\rom{1}_3)\leq \norm{\mV_{t}^{(l)}-\mV_{\epsilon, t}^{(l)}}_{\fro}\left(32 \const\sigma^{\star}_{1}\sqrt{\frac{\mu^2r^2}{pd}}\norm{\mV_{t}^{(l)}-\mV_{\epsilon, t}^{(l)}}_{\fro}+4 \const\sigma^{\star}_{1}\sqrt{\frac{\mu^2r^2}{pd}}\epsilon\right),
\end{aligned}
\end{equation}
which in turn leads to
\begin{equation}
\begin{aligned}
    (\rom{1})&\leq \const\sigma^{\star}_{1}\frac{\kappa\mu r^{1.5}\log\left(\frac{1}{\alpha}\right)}{\sqrt{pd}}\norm{\mV_{t}^{(l)}-\mV_{\epsilon, t}^{(l)}}_{\fro}^2+\norm{\mV_{t}^{(l)}-\mV_{\epsilon, t}^{(l)}}_{\fro}\cdot 4 \const\sigma^{\star}_{1}\sqrt{\frac{\mu^2r^2}{pd}}\epsilon.
\end{aligned}
\end{equation}
Therefore, our final bound is established as 
\begin{equation}
\begin{aligned}
    \norm{\mV_{t+1}^{(l)}-\mV_{\epsilon, t+1}^{(l)}}_{\fro}\leq \left(1+ \const\eta \sigma_1^{\star}\frac{\kappa\mu r^{1.5}\log\left(\frac{1}{\alpha}\right)}{\sqrt{pd}}\right)\norm{\mV_{t}^{(l)}-\mV_{\epsilon, t}^{(l)}}_{\fro}+ C_3\eta \frac{\sigma^{\star}_{1}\mu r}{\sqrt{pd}}\epsilon.
\end{aligned}
\end{equation}
This completes the proof.
\end{proof}
    
\section{Proofs for Different Initialization Schemes}\label{sec_init_proof}
\paragraph*{Gaussian initialization.} Suppose the search rank $0.5d\leq r'\leq d$. Let $\mZ=\mG/\norm{\mG}$ where $\mG\in \bR^{d\times r'}$ is a standard Gaussian matrix. Before proceeding, we first introduce the following lemma, which characterizes the concentration of the largest and smallest singular values of a standard Gaussian matrix. 

\begin{lemma}[Adapted from Theorem~6.1 in \citep{wainwright2019high}]
    \label{lem::concentration-singular-value}
    Suppose that $\mG\in \bR^{d_1\times d_2}$ is a standard Gaussian matrix where $d_1\geq d_2$. Then, for any $\delta>0$, we have 
    \begin{equation}
        \begin{aligned}
            \bP\left(\norm{\mG}\geq (2+\delta)\sqrt{d_1}\right)&\leq \exp\left\{-\frac{d_1\delta^2}{2}\right\},\\
            \bP\left(\sigma_{\min}(\mG)\geq (1-\delta)\sqrt{d_1}-\sqrt{d_2}\right)&\leq \exp\left\{-\frac{d_1\delta^2}{2}\right\}.
        \end{aligned}
    \end{equation}
\end{lemma}
Notice that $\sigma_{r}(\proj_{\mV^{\star}}\mG)=\sigma_{\min}(\mV^{\star\top}\mG)$ where $\mV^{\star\top}\mG\in \bR^{r\times r'}$ is another standard Gaussian matrix. Hence, upon setting $\delta=\frac{1}{2}$ and noting that $r'\geq \frac{d}{2}$, Lemma~\ref{lem::concentration-singular-value} implies that with probability at least $1-\exp\{d/16\}$, we have 
\begin{equation}
    \sigma_{r}(\proj_{\mV^{\star}}\mG)\geq \frac{1}{2}\sqrt{r'}-\sqrt{r}\geq \frac{1}{4}\sqrt{d}.
\end{equation}
On the other hand, upon setting $\delta=\frac{1}{2}$, Lemma~\ref{lem::concentration-singular-value} implies that with probability at least $1-\exp\{d/8\}$, we have
\begin{equation}
    {\mG}\leq \frac{5}{2}\sqrt{d}.
\end{equation}
Via a union bound, we know that with probability at least $1-2\exp\{d/16\}$, we have
\begin{equation}
    \sigma_r(\proj_{\mV^{\star}}\mZ)=\frac{\sigma_r(\proj_{\mV^{\star}}\mG)}{\norm{\mG}}\geq \frac{\sqrt{d}/4}{5\sqrt{d}/2}=\frac{1}{10}.
\end{equation}
Hence, with probability at least $1-2\exp\{d/16\}$, Condition~\ref{condition::small-initialization} is satisfied with $c_0=0.1$.

\paragraph*{Orthogonal initialization.} Suppose the search rank satisfies $r'=d$. Upon choosing $\mZ=\mO$ for some $\mO\in \cO_{d\times d}$, we have  
\begin{equation}
    \sigma_{r}(\proj_{\mV^{\star}}\mO)=\sigma_{r}(\mV^{\star}\mV^{\star\top})=1.
\end{equation} 
Hence, Condition~\ref{condition::small-initialization} is satisfied with $c_0=1$.

\paragraph*{Spectral initialization.}
Let $\mV\mSigma \mV^\top$ be the eigendecomposition of the best rank-$r'$ approximation of $\projection(\mX^{\star})$ measured in Frobenius norm. Suppose $r\leq r'\leq d$ and $\mZ=\mU/\norm{\mU}$, where $\mU=\mV\mSigma^{1/2}$. Corollary~\ref{lem::uniform-concentration-operator-norm} tells us that, with probability at least $1-\frac{1}{d^3}$, we have
\begin{equation}
    \norm{\projection(\mX^{\star})-\mX^{\star}}\leq \const \sigma_1^{\star}\sqrt{\frac{\mu^2r^2}{pd}}.
\end{equation}
Conditioned on this event, we have
\begin{equation}
    \begin{aligned}
        \norm{\mU}^2=\norm{\projection(\mX^{\star})}\leq \norm{\mX^{\star}}+ \norm{\projection(\mX^{\star})-\mX^{\star}}\leq 2\sigma_1^{\star}.
    \end{aligned}
\end{equation}
On the other hand, by Weyl's inequality, we have 
\begin{equation}
    \begin{aligned}
        \sigma^2_{r}(\proj_{\mV^{\star}}\mU)&=\sigma_{r}(\mV^{\star\top}\mU\mU^{\top}\mV^{\star})\\
        &\geq \sigma_{r}(\mV^{\star\top}\projection(\mX^{\star})\mV^{\star})-\sigma_{r+1}(\projection(\mX^{\star}))\\
        &\geq \sigma_{r}(\mV^{\star\top}\mX^{\star}\mV^{\star})-\norm{\projection(\mX^{\star})-\mX^{\star}}-\sigma_{r+1}(\projection(\mX^{\star}))\\
        &\geq \sigma_r^{\star}-2\norm{\projection(\mX^{\star})-\mX^{\star}}\\
        &\geq 0.5\sigma_r^{\star}.
    \end{aligned}
\end{equation}
Combining the above two inequalities, we conclude that, with probability at least $1-\frac{1}{d^3}$, Condition~\ref{condition::small-initialization} is satisfied with $c_0=\frac{1}{2\kappa}$.

\section{Concentration Inequalities for Matrix Completion}
\label{sec::concentration-inequalities-matrix-completion}
Recall that the sampling matrix $\mOmega\in \bR^{d\times d}$ is defined as 
\begin{equation}
    \Omega_{i, j}=\begin{cases}
        1 & \text{if $(i, j)\in \Omega$},\\
        0 & \text{otherwise}.
     \end{cases}
\end{equation}
The following lemma characterizes the concentration behavior of $\mOmega$.
\begin{lemma}[Adapted from \protect{\cite[Lemma~8]{vu2018simple}}]
    \label{lem::Omega-concentration}
     Suppose the sampling rate satisfies $p\gtrsim \frac{\log(d)}{d}$. There is a universal constant $\const > 0$ such that, with probability at least $1-\frac{1}{d^3}$, we have 
    \begin{equation}
        \norm{\frac{\mOmega+\mOmega^{\top}}{2p}-\mJ}\leq \const\sqrt{\frac{d}{p}}.
    \end{equation}
    Here $\mJ$ is the all-one matrix.
    \label{lem::Omega}
\end{lemma}
The original result appeared in \cite[Lemma~8]{vu2018simple} only holds for symmetric Bernoulli model, i.e., $\mOmega=\mOmega^{\top}$. However, we can easily extend it to the asymmetric case via the dilation trick \citep{tropp2015introduction}. Hence, we omit the details here. Next, we have the following extension to the leave-one-out sequences.
\begin{corollary}
    Suppose the sampling rate satisfies $p\gtrsim \frac{\log(d)}{d}$. Then, with probability at least $1-\frac{1}{d^3}$, for all $1\leq l\leq d$, we have 
    \begin{equation}
        \norm{\frac{\mOmega^{(l)}+\mOmega^{(l)\top}}{2p}-\mJ}\leq \const\sqrt{\frac{d}{p}}.
    \end{equation}
    \label{cor::Omega}
\end{corollary}
\begin{proof}\linkofproof{Corollary~\ref{cor::Omega}}
    Note that for any $1\leq l\leq d$, the matrix $\frac{\mOmega^{(l)}+\mOmega^{(l)\top}}{2p}-\mJ$ can be derived from $\frac{\mOmega+\mOmega^{\top}}{2p}-\mJ$ by zeroing out the $l$-th row and column. Hence, the proof follows by invoking Lemma~\ref{lem::norm-submatrix} in Lemma~\ref{lem::Omega}.
\end{proof}

\begin{lemma}[\protect{\citep[Lemma~A.1]{chen2020nonconvex} and \citep[Lemma~8]{chen2019model}}]
    For all $\mA$, $\mB$, $\mC$, and $\mD\in \bR^{d\times r}$, we have
    \begin{equation}
        \begin{aligned}
            &\left|\inner{\left(\id-\projection\right)\left(\mA\mC^{\top}\right)}{\mB\mD^{\top}}\right|\leq  \norm{\frac{\mOmega+\mOmega^{\top}}{2p}-\mJ}\cdot\norm{\mA}_{2, \infty}\norm{\mB}_{\fro}\cdot \norm{\mC}_{\fro}\norm{\mD}_{2, \infty}.
        \end{aligned}
    \end{equation}
    Moreover, for all $1\leq l\leq d$, we have 
    \begin{equation}
        \begin{aligned}
            &\left|\inner{\left(\id-\cR_{\Omega^{(l)}}\right)\left(\mA\mC^{\top}\right)}{\mB\mD^{\top}}\right|\leq  \norm{\frac{\mOmega^{(l)}+\mOmega^{(l)\top}}{2p}-\mJ}\cdot\norm{\mA}_{2, \infty}\norm{\mB}_{\fro}\cdot \norm{\mC}_{\fro}\norm{\mD}_{2, \infty}.
        \end{aligned}
    \end{equation}
\end{lemma}
As a special case, we have 
\begin{lemma}
    For any matrix $\mX\in \bR^{d\times d}$ with the form $\mX=\mU\mV^{\top}$, we have 
    \begin{equation}
        \norm{\left(\id-\cR_{\Omega^{(l)}}\right)\left(\mX\right)}\leq\norm{\left(\id-\projection\right)\left(\mX\right)}\leq \norm{\frac{\mOmega+\mOmega^{\top}}{2p}-\mJ}\norm{\mU}_{2, \infty}\norm{\mV}_{2, \infty}, \quad \forall 1\leq l\leq d.
    \end{equation}
\end{lemma}
Note that the above two results are deterministic. Combining them with Lemma~\ref{lem::Omega-concentration} leads to the following high-probability results.
\begin{corollary}[\protect{\citep[Lemma 4.3 and Lemma A.1]{chen2020nonconvex}}]
    \label{lem::concentration-chen}
    Suppose that the sampling rate satisfies $p\gtrsim \frac{\log(d)}{d}$. There exists a universal constant $ C>0$ such that, for any $\mA, \mB, \mC, \mD\in \bR^{d\times r}$, with probability at least $1-\frac{1}{d^3}$, we have
    \begin{equation}
        \begin{aligned}
            &\left|\inner{\left(\id-\projection\right)\left(\mA\mC^{\top}\right)}{\mB\mD^{\top}}\right|\leq  \const\sqrt{\frac{d}{p}}\cdot\norm{\mA}_{2, \infty}\norm{\mB}_{\fro}\cdot \norm{\mC}_{\fro}\norm{\mD}_{2, \infty}.
        \end{aligned}
    \end{equation}
    Moreover, with the same probability, for any $1\leq l\leq d$, we have 
    \begin{equation}
        \begin{aligned}
            &\left|\inner{\left(\id-\cR_{\Omega^{(l)}}\right)\left(\mA\mC^{\top}\right)}{\mB\mD^{\top}}\right|\leq  \const\sqrt{\frac{d}{p}}\cdot\norm{\mA}_{2, \infty}\norm{\mB}_{\fro}\cdot \norm{\mC}_{\fro}\norm{\mD}_{2, \infty}.
        \end{aligned}
    \end{equation}
\end{corollary}

\begin{corollary}[\protect{\citep[Lemma~9]{chen2019model}}]
    \label{lem::uniform-concentration-operator-norm}
    Consider an arbitrary matrix $\mX\in \bR^{d\times d}$ decomposed as $\mX=\mU\mV^{\top}$. There exists a universal constant $C$ such that, with probability at least $1-\frac{1}{d^3}$, we have
    \begin{equation}
        \norm{\left(\id-\projection^{(l)}\right)(\mX)}\leq\norm{\left(\id-\projection\right)(\mX)}\leq  \const\sqrt{\frac{d}{p}}\norm{\mU}_{2, \infty}\norm{\mV}_{2, \infty}.
    \end{equation}
\end{corollary}
Lastly, we provide a finer result for a matrix of the form $\mX-\mY$.
\begin{lemma}
    \label{lem::2-norm-diff}
    For two arbitrary symmetric matrices $\mX, \mY\in \bR^{d\times d}$, with probability at least $1-\frac{1}{d^3}$, we have 
    \begin{equation}
        \norm{\left(\id-\projection^{(l)}\right)(\mX-\mY)}\leq\norm{\left(\id-\projection\right)\left(\mX-\mY\right)}\leq  \const\sqrt{\frac{d}{p}}\norm{\mX-\mY}\left(\norm{\mU_{\mX}}_{2, \infty}^2+\norm{\mU_{\mY}}_{2, \infty}^2\right).
    \end{equation}
\end{lemma}
\begin{proof}
    We denote $\mZ=\mX-\mY$ with an SVD $\mZ=\mL_{\mZ}\mSigma_{\mZ}\mL_{\mZ}^{\top}$ (recall that $\mZ$ is symmetric). Hence, with probability at least $1-\frac{1}{d^3}$, we have
    \begin{equation}
        \begin{aligned}
            \norm{\left(\id-\projection\right)\left(\mZ\right)}&\stackrel{(a)}{\leq}  \const\sqrt{\frac{d}{p}}\norm{\mL_{\mZ}\mSigma_{\mZ}}_{2, \infty}\norm{\mL_{\mZ}}_{2, \infty}\\
            &\stackrel{(b)}{\leq}  \const\sqrt{\frac{d}{p}}\norm{\mSigma_\mZ}\norm{\mL_{\mZ}}^2_{2, \infty}\\
            &=  \const\sqrt{\frac{d}{p}}\norm{\mZ}\norm{\mL_{\mZ}}^2_{2, \infty}\\
            &\stackrel{(c)}{\leq}  \const\sqrt{\frac{d}{p}}\norm{\mZ}\left(\norm{\mL_{\mX}}_{2, \infty}^2+\norm{\mL_{\mY}}_{2, \infty}^2\right).
        \end{aligned}
    \end{equation}
    Here in $(a)$, we apply Lemma~\ref{lem::uniform-concentration-operator-norm} upon setting $\mU=\mL_{\mZ}\mSigma_{\mZ}$ and $\mV=\mL_{\mZ}$. In $(b)$, we apply Lemma~\ref{lem::2-inf-ineq}.
    Lastly, in $(c)$, we apply Lemma~\ref{lem::2-inf-norm-decomposition}.
    This completes the proof.
\end{proof}

\section{Auxiliary Lemmas}\label{sec_aux_lemmas}
\subsection{Concentration Inequalities}
\begin{lemma}[Bernstein's inequality]
    \label{lem::bernstein}
    Let $X_1,\cdots, X_n$ be independent zero-mean random variables. Suppose that $|X_{i}|\leq M$ almost surely, for all $i$ and set $\nu^2=\sum_{i=1}^{n}\Var\left[X_i^2\right]$. Then, for all positive $t$,
    \begin{equation}
        \label{eq::bernstein}
        \bP\left(\left|\sum_{i=1}^{n}X_i\right|\geq t\right)\leq 2\exp\left(-\frac{t^2}{2\nu^2+\frac{2}{3}Mt}\right).
    \end{equation}
    Or equivalently, with probability at least $1-\delta$, one has 
    \begin{equation}
        \left|\sum_{i=1}^{n}X_i\right|\leq 2\nu\sqrt{\log\left(2/\delta\right)}+\frac{4}{3}M\log\left(2/\delta\right).
    \end{equation}
\end{lemma}

\subsection{Matrix Norm Inequalities}
\begin{lemma}[\protect{\citep[Proposition 6.5]{cape2019two}}]
    \label{lem::2-inf-ineq}
    For $\mA\in \bR^{d_1\times d_2}$, and $\mB\in\bR^{d_2\times d_3}$, we have 
    \begin{equation}
        \norm{\mA\mB}_{2, \infty}\leq \norm{\mA}_{2, \infty}\norm{\mB} \quad \text{and} \quad  \norm{\mA\mB}_{2, \infty}\leq \norm{\mA}_{\infty}\norm{\mB}_{2, \infty}.
    \end{equation}
\end{lemma}
\begin{lemma}[Adapted from \protect{\cite[Proposition A.3]{sun2016guaranteed}}]
    \label{lem::norm-submatrix}
    For any matrix $\mA\in \bR^{d_1\times d_2}$, denote $\mA_{-i, \cdot}$ ($\mA_{\cdot, -i}$) as the matrix obtained by replacing the $i$-th row (column) of $\mA$ by zeros, respectively. Then, we have
    \begin{equation}
        \norm{\mA_{-i, \cdot}}\leq \norm{\mA} \quad \text{and}\quad \norm{\mA_{\cdot, -j}}\leq \norm{\mA}, \forall i \in [d_1], j\in [d_2].
    \end{equation}
\end{lemma}
\begin{lemma}[\protect{\cite[Proposition A.4]{sun2016guaranteed}}]
    \label{lem::norm-submatrix-2}
    For any two matrices $\mA\in \bR^{d_1\times d_2}, \mB\in \bR^{d_2\times d_3}$, we have 
    \begin{equation}
        \norm{\mA\mB}\leq \norm{\mA}\norm{\mB} \quad \text{and} \quad \norm{\mA\mB}_{\fro}\leq \norm{\mA}\norm{\mB}_{\fro}.
    \end{equation}
    Furthermore, if $d_1\geq d_2$, then 
    \begin{equation}
        \sigma_{\min}(\mA)\norm{\mB}_{\fro}\leq \norm{\mA\mB}_{\fro} \quad \text{and} \quad\sigma_{\min}(\mA)\norm{\mB}\leq \norm{\mA\mB}.
    \end{equation}
\end{lemma}

\begin{lemma}
    \label{lem::max-norm-decomposition}
    For arbitrary matrices $\mU\in \bR^{d_1\times d_2}$, $\mSigma\in \bR^{d_2\times d_3}$ and $\mV\in \bR^{d_4\times d_3}$, we have 
    \begin{equation}
        \norm{\mU\mSigma\mV^{\top}}_{\max}\leq \norm{\mSigma}\norm{\mU}_{2, \infty}\norm{\mV}_{2, \infty}.
    \end{equation}
\end{lemma}
\begin{proof}
    By Cauchy-Schwartz inequality, we have
    \begin{equation}
        \begin{aligned}
            \norm{\mU\mSigma\mV^{\top}}_{\max}&=\max_{i, j}\left|\sum_{k}[\mU\mSigma]_{i, k}\emV_{j, k}\right|\\
            &\leq \max_{i, j}\norm{(\mU\mSigma)_{i,\cdot}}\norm{\mV_{j, \cdot}}\\
            &=\norm{\mU\mSigma}_{2,\infty}\norm{\mV}_{2,\infty}\\
            &\leq \norm{\mSigma}    \norm{\mU}_{2,\infty}\norm{\mV}_{2,\infty}.
        \end{aligned}
    \end{equation}
    Here we apply Lemma~\ref{lem::2-inf-ineq} to derive the last inequality. This completes the proof.
\end{proof}

\begin{lemma}
    \label{lem::2inf-max}
    For any matrix $\mA\in \bR^{d_1\times d_2}$, we have 
    \begin{equation}
        \norm{\mA}^2_{2, \infty}=\norm{\mA\mA^{\top}}_{\max}.
    \end{equation}
\end{lemma}
\begin{proof}
    We write $\va_1, \cdots,\va_{d_1}$ as the row vectors of $\mA$. Then, we have
    \begin{equation}
        \norm{\mA}^2_{2, \infty}=\max_{1\leq i\leq d_1}\norm{\va_i}^2.
    \end{equation}
    On the other hand, we can write $\norm{\mA\mA^{\top}}_{\max}$ as 
    \begin{equation}
        \norm{\mA\mA^{\top}}_{\max}=\max_{1\leq i, j\leq d_1}\left|\inner{\va_i}{\va_j}\right|.
    \end{equation}
    First, we have $\norm{\mA\mA^{\top}}_{\max}\leq \norm{\mA}^2_{2, \infty}$ since 
    \begin{equation}
        \norm{\mA\mA^{\top}}_{\max}=\max_{1\leq i, j\leq d_1}\left|\inner{\va_i}{\va_j}\right|\leq \max_{1\leq i, j\leq d_1}\norm{\va_i}\norm{\va_j}=\max_{1\leq i\leq d_1}\norm{\va_i}^2=\norm{\mA}^2_{2, \infty}.
    \end{equation}
    Second, we have $\norm{\mA\mA^{\top}}_{\max}\geq \norm{\mA}^2_{2, \infty}$ upon noting that 
    \begin{equation}
        \norm{\mA\mA^{\top}}_{\max}=\max_{1\leq i, j\leq d_1}\left|\inner{\va_i}{\va_j}\right|\geq \max_{i=j}\left|\inner{\va_i}{\va_j}\right|=\max_{1\leq i\leq d_1}\norm{\va_i}^2=\norm{\mA}^2_{2, \infty}.
    \end{equation}
    Therefore, we derive that $\norm{\mA\mA^{\top}}_{\max}=\norm{\mA}^2_{2, \infty}$, which completes the proof.
\end{proof}

\begin{lemma}
    \label{lem::max-norm-PSD-matrix}
    For two PSD matrices $\mA, \mB\in \bR^{d\times d}$ with $\mA\preceq \mB$, we have 
    \begin{equation}
        \norm{\mA}_{\max}\leq \norm{\mB}_{\max}.
    \end{equation}
\end{lemma}
\begin{proof}
    The proof follows by the fact that for any PSD matrix $\mA$, we have $\norm{\mA}_{\max}=\max_{i}\{\emA_{i, i}\}$. According to this fact, we immediately have 
    \begin{equation}
        \norm{\mA}_{\max}=\max_{i}\{\emA_{i, i}\}\leq \max_{i}\{\emB_{i, i}\}=\norm{\mB}_{\max}
    \end{equation}
    since $\mA\preceq \mB$. Now we turn to prove this fact. Note that we can write any PSD matrix $\mA$ as $\mA=\mP\mP^{\top}$. Then, according to Lemma~\ref{lem::2inf-max}, we have 
    \begin{equation}
        \norm{\mA}_{\max}=\norm{\mP}^2_{2, \infty}=\max_{i}\norm{\vp_i}^2=\max_{i}\{\emA_{i, i}\}.
    \end{equation}
    Here we write $\{\vp_i\}$ as the row vectors of $\mP$. This completes the proof.
\end{proof}

\begin{lemma}
    \label{lem::2-inf-norm-decomposition}
    For $\mZ=\mX-\mY$, we have 
    \begin{equation}
        \norm{\mL_{\mZ}}^2_{2, \infty}\leq \norm{\mL_{\mX}}_{2, \infty}^2+\norm{\mL_{\mY}}_{2, \infty}^2.
    \end{equation}
\end{lemma}
\begin{proof}
    We bound $\norm{\mL_{\mZ}}^2_{2, \infty}$ as follows
    \begin{equation}
        \begin{aligned}
            \norm{\mL_{\mZ}}^2_{2, \infty}&\stackrel{(a)}{=}\norm{\mL_{\mZ}\mL_{\mZ}^{\top}}_{\max}\\
            &\stackrel{(b)}{\leq} \norm{\mL_{\mX}\mL_{\mX}^{\top}+\mL_{\mY}\mL_{\mY}^{\top}}_{\max}\\
            &\leq \norm{\mL_{\mX}\mL_{\mX}^{\top}}_{\max}+\norm{\mL_{\mY}\mL_{\mY}^{\top}}_{\max}\\
            &\stackrel{(c)}{=}\norm{\mL_{\mX}}_{2, \infty}^2+\norm{\mL_{\mY}}_{2, \infty}^2.
        \end{aligned}
    \end{equation}
    Here $(a)$ and $(c)$ follow from Lemma~\ref{lem::2inf-max}. In $(b)$, we apply Lemma~\ref{lem::max-norm-PSD-matrix} since $\mL_{\mZ}\mL_{\mZ}^{\top}\preceq \mL_{\mX}\mL_{\mX}^{\top}+\mL_{\mY}\mL_{\mY}^{\top}$. This is due to the fact that $\col(\mZ)\subseteq \col(\mX)\oplus\col(\mY)$ which leads to $\mL_{\mZ}\mL_{\mZ}^{\top}=\proj_{\mZ}\preceq \proj_{\mX}+\proj_{\mY}$. This completes the proof.
\end{proof}

\begin{lemma}[\protect{\cite[Lemma 5.4]{tu2016low}}]
    \label{lem::dist-fro-tu2016}
    For any $\mX, \mY\in \bR^{d\times r}$ with $\sigma_{r}(\mX)>0$, we have 
    \begin{equation}
        \dist^2(\mX, \mY)\leq \frac{1}{2\left(\sqrt{2}-1\right)\sigma_{r}^2(\mX)}\norm{\mX\mX^{\top}-\mY\mY^{\top}}_{\fro}^2.
    \end{equation}
\end{lemma}
\begin{lemma}
    \label{lem::lower-bound-2-inf-norm}
    Consider a matrix $\mU\in \bR^{d_1\times d_2}$ and a diagonal matrix $\mSigma\in \bR^{d_2\times d_2}$. We have
    \begin{equation}
        \norm{\mU\mSigma}_{2, \infty}\geq \sigma_{r}(\mSigma)\norm{\mU}_{2, \infty}.
    \end{equation}
\end{lemma}
\begin{proof}
    We first write $\mSigma=\diag\{\sigma_1, \cdots, \sigma_{d_2}\}$. Next, note that 
    \begin{equation}
        \norm{(\mU\mSigma)_{l,\cdot}}^2=\sum_{j=1}^{d_2}\emU_{l, j}^2\sigma_j^2\geq \sigma_{r}^2(\mSigma)\sum_{j=1}^{d_2}\emU_{l, j}^2=\sigma_{r}^2(\mSigma)\norm{\mU_{l, \cdot}}^2.
    \end{equation}
    Hence, taking the maximum over index $l$ on both sides, we immediately obtain 
    \begin{equation}
        \norm{\mU\mSigma}_{2, \infty}\geq \sigma_{r}(\mSigma)\norm{\mU}_{2, \infty}.
    \end{equation}
\end{proof}
\subsection{Other Useful Inequalities}
\begin{lemma}[Adapted from \protect{\cite[Corollary 4.2.13]{vershynin2018high}}]
    \label{lem::covering-number}
    The covering number $\cN_{\epsilon}$ of $\cB_{\fro}^{r\times r}(R)$ satisfies the following inequality for any $0<\epsilon\leq 1$:
    \begin{equation}
        \cN_{\epsilon}\leq \left(\frac{3R}{\epsilon}\right)^{r^2}.
    \end{equation}
\end{lemma}
\begin{lemma}
    \label{lem::orthogonal-matrix}
    For two orthogonal matrices $\mV_1, \mV_2\in \cO_{d\times r}$, we have 
    \begin{equation}
        \begin{aligned}
            \norm{\mV_1\mV_1^{\top}-\mV_2\mV_2^{\top}}\leq 2\norm{\mV_1-\mV_2}\quad \text{and}\quad
            \norm{\mV_1\mV_1^{\top}-\mV_2\mV_2^{\top}}_{\fro}&\leq 2\norm{\mV_1-\mV_2}_{\fro}.
        \end{aligned}
    \end{equation}
\end{lemma}
\begin{proof}
    Note that $\mV_1\mV_1^{\top}-\mV_2\mV_2^{\top}=\mV_1\left(\mV_1-\mV_2\right)^{\top}+\left(\mV_1-\mV_2\right)\mV_2^{\top}$. Hence, we have 
    \begin{equation}
        \norm{\mV_1\mV_1^{\top}-\mV_2\mV_2^{\top}}\leq \norm{\mV_1-\mV_2}\left(\norm{\mV_1}+\norm{\mV_2}\right)\leq 2\norm{\mV_1-\mV_2}.
    \end{equation}
    Similarly, for the Frobenius norm, we also have 
    \begin{equation}
        \norm{\mV_1\mV_1^{\top}-\mV_2\mV_2^{\top}}_{\fro}\leq \norm{\mV_1-\mV_2}_{\fro}\left(\norm{\mV_1}+\norm{\mV_2}\right)\leq 2\norm{\mV_1-\mV_2}_{\fro}.
    \end{equation}
    This completes the proof.
\end{proof}
\begin{lemma}
    \label{lem::lower-bound-sigma-min}
    For arbitrary matrix $\mX\in \bR^{d_1\times d_2}$ with $\rank(\mX)=r$ and $\mO\in \cO_{d_1\times r}$, we have 
    \begin{equation}
        \sigma_{r}(\mX)\geq \sigma_{r}\!\left(\mO^{\top}\mX\right).
    \end{equation}
    Moreover, if $\mO=\mU_{\mX}$, we have $\sigma_i(\mX)=\sigma_i\left(\mO^{\top}\mX\right)$ for all $1\leq i\leq r$.
\end{lemma}
\begin{proof}
    We first prove the special case. Suppose $\mO=\mU_{\mX}$, then we have $\mO^{\top}\mX=\mSigma_{\mX}\mV_{\mX}^{\top}$. Note that this is the SVD of $\mO^{\top}\mX$ with the singular value matrix $\mSigma_{\mX}$. Hence, $\mO^{\top}\mX$ has the same singluar values as $\mX$. For the general case, it follows by 
    \begin{equation}
        \begin{aligned}
            \sigma_{r}\!\left(\mO^{\top}\mX\right)&=\inf_{\mY\in \cM_{\leq r-1}}\norm{\mO^{\top}\mX-\mY}_{\fro}\\
            &\leq \norm{\mO^{\top}\mX - \mO^{\top}\mU_{\mX}\mSigma_{\mX, -1}\mV_{\mX}^{\top}}_{\fro}\\
            &\leq \norm{\mO^{\top}\sigma_{r}(\mX)\vu_{\mX, -1}\vv_{\mX, -1}^{\top}}_{\fro}\\
            &\leq \sigma_{r}(\mX).
        \end{aligned}
    \end{equation}
    \begin{sloppypar}
        Here we define $\cM_{\leq r-1}$ as the set of all matrices of rank at most $r-1$.
    We also denote $\mSigma_{\mX, -1}=\diag\{\sigma^{\star}_{1}(\mX), \cdots, \sigma_{r-1}(\mX), 0\}$. Finally, $\vu_{\mX, -1}, \vv_{\mX, -1}$ refer to the last columns of $\mU_{\mX}, \mV_{\mX}$, respectively.
    \end{sloppypar}
\end{proof}

\begin{lemma}
    \label{lem::lower-bound-sigma-min-2}
    Consider two matrices $\mA\in \bR^{r\times r}, \mB\in \bR^{r\times d}$ where $\mA$ is invertible and $\mB\neq \vzero$. We have 
    \begin{equation}
        \sigma_{r}(\mA\mB)\geq \sigma_{r}(\mA)\sigma_{r}(\mB).
    \end{equation}
\end{lemma}
\begin{proof}
    It directly follows by 
    \begin{equation}
        \begin{aligned}
            \sigma_{r}(\mA\mB)&=\min_{\norm{\vx}=1, \vx\in \mathrm{range}(\mB^{\top})}\norm{\mA\mB\vx}\\
            &\geq \sigma_{r}(\mA)\min_{\norm{\vx}=1, \vx\in \mathrm{range}(\mB^{\top})}\norm{\mB\vx}\\
            &=\sigma_{r}(\mA)\sigma_{r}(\mB).
        \end{aligned}
    \end{equation}
\end{proof}

\begin{lemma}
    \label{lem::exponential-linear}
    For any $x\in [0, \frac{1}{2(r-1)})$ and $r>1$, we have
    \begin{equation}
        \left(1+x\right)^{r}\leq 1+2rx.
    \end{equation}
\end{lemma}
\begin{lemma}
    \label{lem::appendix-series-ineq}
    For the series $\{x_t\}_{t=0}^{\infty}$, the following two statements hold:
    \begin{itemize}
        \item Suppose that $x_{t+1}\leq Ax_t+B, \forall t\geq 0$ where $A>0, A\neq 1$ and $x_0+\frac{B}{A-1}\geq 0$. Then, we have 
        \begin{equation}
            x_t\leq A^t\left(x_0+\frac{B}{A-1}\right)-\frac{B}{A-1}.
        \end{equation}
        \item Suppose that $x_{t+1}\geq Ax_t-B, \forall t\geq 0$ where $A>0$ and $x_0-\frac{B}{A-1}\geq 0$. Then, we have
        \begin{equation}
            x_t\geq A^t\left(x_0-\frac{B}{A-1}\right)+\frac{B}{A-1}.
        \end{equation}
    \end{itemize}
\end{lemma}
\begin{proof}
    We first prove the first statement. Note that we can rewrite $x_{t+1}\leq Ax_t+B$ as $y_{t+1}\leq Ay_t$ where $y_t=x_t+\frac{B}{A-1}$. Note that $y_0\geq 0, A>0$ by our assumption. Hence, we derive $y_t\leq A^ty_0=A^t\left(x_0+\frac{B}{A-1}\right)$, which implies that $x_t=y_t-\frac{B}{A-1}\leq A^t\left(x_0+\frac{B}{A-1}\right)-\frac{B}{A-1}$. 
    
    For the second statement, we first rewrite the condition as $y_{t+1}\geq Ay_t$ where $y_t=x_t-\frac{B}{A-1}$. Then, we have $y_t\geq A^ty_0$, which implies that $x_t=y_t+\frac{B}{A-1}\geq A^ty_0+\frac{B}{A-1}=A^t\left(x_0-\frac{B}{A-1}\right)+\frac{B}{A-1}$.
\end{proof}

\end{document}